\documentclass[11pt]{article}

\usepackage[margin=1in]{geometry}

\usepackage{graphicx}
\usepackage{booktabs}
\usepackage{multirow}
\usepackage{natbib}
\usepackage{hyperref}
\usepackage{adjustbox}

\usepackage{amsmath,amssymb,amsfonts}
\usepackage{amsthm}
\usepackage{mathrsfs}
\usepackage{mathtools}
\usepackage{fix-cm}
\usepackage{makecell}
\usepackage[nameinlink]{cleveref}

\usepackage[title]{appendix}
\usepackage[dvipsnames]{xcolor}
\usepackage{textcomp}
\usepackage{manyfoot}
\usepackage{booktabs}
\usepackage{algorithm}
\usepackage{algorithmicx}
\usepackage{algpseudocode}
\usepackage{listings}
\usepackage{bm}
\usepackage{subcaption}
\usepackage{hyperref}
\usepackage{enumitem}
\usepackage{multirow}
\usepackage{comment}
\usepackage{tikz}
\usepackage{lmodern}
\usepackage{fix-cm}
\usepackage{threeparttable}

\usetikzlibrary{arrows.meta, positioning, calc}

\newtheorem{theorem}{Theorem}
\newtheorem{proposition}[theorem]{Proposition}

\newtheorem{remark}{Remark}
\newtheorem{lemma}[theorem]{Lemma} 
\newtheorem{corollary}[theorem]{Corollary}

\hypersetup{
    colorlinks=true,
    citecolor=blue,
    linkcolor=blue,
    urlcolor=blue
}

\title{A nonparametric two-sample test using a parametric integral probability metric \\[1em]}

\author{
Yuha Park\thanks{Department of Mathematics, University of Hamburg, Hamburg, Germany. Email: yuha.park@uni-hamburg.de}\\
University of Hamburg
\and
Yongdai Kim\thanks{Department of Statistics, Seoul National University, Seoul, Republic of Korea. Email: ydkim0903@gmail.com}\\
Seoul National University
}

\date{}

\begin{document}

\maketitle

\vspace{2em}

\begin{abstract}
Detecting distributional differences between two independent samples is a fundamental problem in statistics and machine learning.
Nonparametric two-sample testing provides a principled framework for determining whether two samples are drawn from the same underlying distribution, without assuming any specific parametric form for the distribution.
In this study, we propose a new two-sample test statistic based on a newly introduced integral probability metric (IPM), using a specially designed parametric discriminator class with a single node of a neural network.
We show that the resulting test statistic, called PReLU-IPM, is nonparametric and establish theoretical guarantees for the associated two-sample testing procedure, PReLU-TST, including its consistency and asymptotical equivalence to nonparametric IPM-based tests under regularity conditions.
By analyzing multiple simulated and real benchmark datasets, we demonstrate that PReLU-TST achieves higher power across a range of alternatives or performs comparably to its competitors, for finite samples.
\end{abstract}

\vspace{1cm}

convergence rates, asymptotic analysis

\vspace{2em}

\section{Introduction}

We study the nonparametric two-sample testing problem, a fundamental task in the field of statistical hypothesis testing. 
The goal of the nonparametric two-sample test is to determine whether two independent samples originate from the same distribution, without assuming any specific parametric form for the underlying distributions.
This problem arises in various scientific and applied contexts, such as comparing treatment effects or detecting distributional shifts \citep{keselman1997specialized, crump2008nonparametric, lee2009non, lipton2018detecting, rabanser2019failing}.
In addition, it is widely used for machine learning applications including adversarial attack \citep{Gao2021, Xu2022}, change-point detection \citep{Xie2013, Juditsky2016, Cao2018}, model generation \citep{Binkowski2018}, model criticism \citep{Lloyd2015, Chwialkowski2016}, distribution shift detection \citep{kubler2022}, synthetic data detection \citep{Lopez2016, Liu2020, Wang2022}, out-of-distribution detection \citep{Jia2021, Dong2022ood, Wu2022}, causal inference \citep{Lopez2016}, and differential private hypothesis tests \citep{Balle2020hypothesis, Dong2022dp, Alabi2023}.

Let $\mathbf{X}$ and $\mathbf{Y}$ be random vectors defined on the same data space $\mathcal{X}$ with respect to the probability distributions $\mathrm{P}$ and $\mathrm{Q}$, respectively. 
Suppose that $\mathbf{X}^n \coloneqq \{\mathbf{X}_i\}_{i=1}^n$ and $\mathbf{Y}^m \coloneqq \{\mathbf{Y}_i\}_{i=1}^m$ are \textit{independently and identically distributed} (\textit{i.i.d.}) random vectors drawn from $\mathrm{P}$ and $\mathrm{Q}$, respectively. The null and alternative hypotheses for the two-sample test are given as
\begin{align} \label{H1:fixed}
    \mathbb{H}_0: \mathrm{P}=\mathrm{Q} \qquad \text{ vs. } \qquad \mathbb{H}_1: \mathrm{P} \neq \mathrm{Q},
\end{align}
and we are to construct a level-$\alpha$ test $T: \mathcal{X}^n \times \mathcal{X}^m \rightarrow \{0,1\},$
with which we reject $\mathbb{H}_0$ when $T(\mathbf{X}^n, \mathbf{Y}^m) = 1,$
such that the Type-I error satisfies
$\mathbb{P}_{\mathrm{P}=\mathrm{Q}}\big\{ T(\mathbf{X}^n, \mathbf{Y}^m) = 1 \big\}\le \alpha$
for all $\mathrm{P}=\mathrm{Q}$
and the power $\mathbb{P}_{\mathrm{P}\ne \mathrm{Q}}\big\{ T(\mathbf{X}^n, \mathbf{Y}^m) = 1 \big\}$
becomes as large as possible. We say that a given test $T$ is nonparametric 
if $\mathbb{P}_{\mathrm{P}\ne \mathrm{Q}}\big\{ T(\mathbf{X}^n, \mathbf{Y}^m) = 1 \big\} \rightarrow 1$
as $\min\{n,m\}\rightarrow \infty$ for any $\mathrm{P}\ne \mathrm{Q}.$

The aim of this paper is to propose a new nonparametric two-sample test statistic based on a newly proposed integral probability metric (IPM), so called the Parametric ReLU-IPM (PReLU-IPM), which uses the \textit{parametric ReLU (PReLU) activation function}\footnote{Note that the meaning of \textit{parametric} as proposed for activation functions differs from that in our notion of parametric IPM. In \citet{HeZR015}, the term ``parametric'' was intended to indicate that the parameters of function are \textit{learnable}.} proposed by \citet{HeZR015} in the discriminator class of the IPM. We prove that the proposed test statistic is nonparametric
and show by analyzing multiple datasets that
it has higher powers under various alternatives than its competitors for finite samples.

The IPM \citep{muller1997} measures the distance between two probability distributions as the maximal difference in the expected values over a given class of discriminator functions. To be more specific,
let $\mathscr{F}$ be a class of (real-valued) bounded measurable functions $f$ from $\mathcal{X}$ to $\mathbb{R}$. For a given class $\mathscr{F}$, the IPM between given two probability distributions $\mathrm{P}$ and $\mathrm{Q}$, denoted as $D_{\mathscr{F}}(\mathrm{P}, \mathrm{Q})$, is defined as 
\begin{equation} \label{equ:IPM}
    D_{\mathscr{F}}(\mathrm{P}, \mathrm{Q}) \coloneqq \sup_{f \in \mathscr{F}} \left| \int f(\mathbf{X})
    \left (d\mathrm{P}-d\mathrm{Q} \right) \right|,
\end{equation}
where the functions $f$ are referred to as \textit{discriminators}, \textit{critics} or \textit{adversaries}, depending on the context of their specific applications.
When $\mathscr{F}$ is the set of $1$-Lipschitz continuous functions, the IPM becomes the well-known Wasserstein distance \citep{KR:58}. In this paper, we call $\mathscr{F}$ the discriminator class.

For a given IPM $D_{\mathscr{F}}(\cdot, \cdot),$  we consider the two-sample test $T: \mathcal{X}^n \times \mathcal{X}^m \rightarrow \{0,1\},$
\begin{equation}
\label{eq:test-stat}
T(\mathbf{X}^n, \mathbf{Y}^m)=\mathbb{I}\big(D_{\mathscr{F}}(\mathrm{P}_n, \mathrm{Q}_m) \ge \eta_{\alpha,n,m}\big),
\end{equation}
where the threshold $\eta_{\alpha,n,m}$ is chosen so that the Type-I error is less than or equal to $\alpha.$
We denote the two-sample test based on the IPM $D_{\mathscr{F}}$ in (\ref{eq:test-stat}) as $D_{\mathscr{F}}$-TST.
Various $D_{\mathscr{F}}$-TSTs have been proposed by selecting $\mathscr{F}$ accordingly.
Examples are $\mathscr{F}_{k} \coloneqq \left\{  f: \|f\|_{\mathscr{F}_k} \leq 1 \right\}$, where $\mathscr{F}_k$ is the reproducing kernel Hilbert space associated with a given kernel $k$ \citep{gretton2012kernel, li2019},
which we denote MMD-TST,
and $\mathscr{H}^{\beta, d} \coloneqq \{ f: \|f\|_{\mathscr{H}^{\beta, d}} \leq 1 \}$ where $\mathscr{H}^{\beta, d}$ is the H\"{o}lder function class \citep{Wang2023manifold}, which we denote H\"{o}lder-TST. 

An intuition would be that a $D_{\mathscr{F}}$-TST could detect more complex differences between $\mathrm{P}$ and $\mathrm{Q}$
when $\mathscr{F}$ becomes larger. Popularly used $D_{\mathscr{F}}$-TSTs including MMD-TST \citep{gretton2012kernel, li2019},
H\"{o}lder-TST \citep{Wang2023manifold} and Wass-TST (Wassestein-TST) \citep{Ramdas15, Wang2022}
use infinite dimensional discriminator classes.
However, when $\mathscr{F}$ is too complex, 
$D_{\mathscr{F}}$-TST might not detect simple alternatives
such as location and scale shifts. For example, as noted by \citet{song2024generalized},
MMD-TST is not good at detecting the scale alternative. In addition,
MMD-TST and Wass-TST would suffer from low power in moderate to high-dimensional settings due to the curse of dimensionality \citep{Reddi2015, Fournier2013Wasserstein}.
H\"{o}lder-TST \citep{Wang2023manifold} is not popularly used since its computation is challenging. 
In practice, H\"{o}lder functions are approximated by deep neural networks (DNN) but the choice of the optimal architecture is not easy.

In this paper, we also introduce a new nonparametric two-sample test based on a new
IPM so called PReLU-IPM, where the discriminator class is parametric.
The proposed PReLU-TST effectively resolves aforementioned problems in the existing $D_{\mathscr{F}}$-TSTs.
In particular, by using a specially designed parametric discriminator class, 
PReLU-TST can improve powers for simple alternatives (e.g., location and scale alternatives) 
but can maintain asymptotical consistency for nonparametric alternatives (i.e., nonparametric test).
In addition, we show that PReLU-TST is consistent for the local alternative with respect to the H\"{o}lder ball.
Surprisingly, PReLU-TST is asymptotically equivalent to H\"{o}lder-TST when $\mathrm{P}$ and $\mathrm{Q}$
are sufficiently smooth, even if the discriminator class of PReLU-TST is parametric.

Our main contributions are summarized as follows:
\begin{itemize}
    \item We propose a new IPM defined over a parametric discriminator class, which employs a single-layer neural network with the PReLU activation function \citep{HeZR015}. Then, based on this IPM, we develop a new two-sample test, PReLU-TST.
    \item We establish theoretical guaranties for PReLU-TST, demonstrating its consistency and asymptotical
    equivalence to H\"{o}lder-TST under regularity conditions.   
    \item Through analyzing simulated data as well as real benchmark data, we demonstrate that PReLU-TST has higher powers for certain parametric alternatives including location and scale alternatives, and 
    is competitive to other $D_{\mathscr{F}}$-TSTs for complex and/or high-dimensional alternatives.
\end{itemize}

The remainder of this paper is organized as follows: Section \ref{sec_2} provides preliminaries, including notation and related works. In Section \ref{sec_3}, we introduce PReLU-IPM and corresponding PReLU-TST
for the nonparametric two-sample test, and study its theoretical properties.
Section \ref{sec_4} presents the settings and results of empirical studies, and finally, Section \ref{sec_5} concludes the paper and suggests directions for future research.

\section{Preliminaries}
\label{sec_2}

\subsection{Notation}
    We use bold letters to denote vectors.
    Let $\mathbb{N}=\{1,2,\ldots \}$ be the set of natural numbers.
    For given two sequences $\{a_n\}$ and $\{b_n\}$, the notations $a_n \lesssim b_n$ and $a_n \gtrsim b_n$ indicate that there exists a positive constant $c$ such that for all $n$, $a_n \leq c \, b_n$ and $a_n \geq c \, b_n$, respectively. 
    For $a \in \mathbb{R}$, let $\lfloor a \rfloor$ and $\lceil a \rceil$ denote the floor and ceiling functions, which round $a$ to the next strictly smaller and larger integer, respectively.
    We use $\otimes$ to denote a product distribution over the coordinate axes.
    We denote $\mathbf{0}_d$ and $\mathbf{1}_d$ as the $d$-dimensional vectors of all zeros and all ones, respectively.
    We denote $\mathrm{I}_d$ as the $d \times d$ dimensional identity matrix.
    For a vector $\mathbf{v}$, we write $\operatorname{diag}(\mathbf{v})$ as the diagonal matrix with the entries of $\mathbf{v}$ on its diagonal. 
    For a given condition or event $A$, $\mathbb{I}(A)$ denotes the indicator function, where $\mathbb{I}(A) = 1$ if $A$ is true and $\mathbb{I}(A) = 0$ if $A$ is false. We define $(\cdot)_{+} \coloneqq \max(0, \cdot)$ and $(\cdot)_{-} \coloneqq \min(0, \cdot)$.
    Let $\mathbb{B}^d = \{\bm{z} \in \mathbb{R}^d : \|\bm{z}\|_2 \leq 1 \}$ and  $\mathbb{S}^{d-1} = \{\mathbf{z} \in \mathbb{R}^d : \|\mathbf{z}\|_2 = 1 \}$ be the unit ball and the unit sphere in $\mathbb{R}^{d},$ respectively.
    For a matrix $M$, let $\|M\|_{\operatorname{col},2}$ denote the maximum $\ell_2$-norm of its columns.
    We write $\mathscr{H}^{\beta, d} \coloneqq \{ f: \mathbb{B}^d \to \mathbb{R}, \|f\|_{\mathscr{H}^{\beta, d}} \leq 1 \}$ as the H\"{o}lder function class,    
    where the H\"{o}lder norm $\|\cdot\|_{\mathscr{H}^{\beta, d}}$ is defined as
    $$   \|f\|_{\mathscr{H}^{\beta, d}} = \sum_{\|\bm{\alpha}\|_1=\lfloor \beta \rfloor} \max_{\mathbf{x},\mathbf{y} \in \mathbb{B}^d, \mathbf{x} \neq \mathbf{y}} \frac{|f^{(\bm{\alpha})}(\mathbf{x})-f^{(\bm{\alpha})}(\mathbf{y})|}{\|\mathbf{x}-\mathbf{y}\|_2^{\beta - \lfloor \beta \rfloor}} + \sum_{\|\bm{\alpha}\|_1 \leq \lfloor \beta \rfloor} \max_{\mathbf{x} \in \mathbb{B}^d} |f^{(\bm{\alpha})}(\mathbf{x})|,
    $$
    with $f^{(\bm{\alpha})}=\frac{\partial^{\|\bm{\alpha}\|_1} f}{\partial x_1^{\alpha_1} \partial x_2^{\alpha_2} \ldots \partial x_d^{\alpha_d}}$ for $\bm{\alpha} = (\alpha_1, \alpha_2, \ldots, \alpha_d)^{\top}$.

\subsection{Related Works}    
    Two primary approaches of the two-sample test are parametric and non-parametric ones. Classical parametric methods including the \textit{t-test} \citep{Student1908} and  \textit{Hotelling's} $T^2$ test \citep{Hotelling1931}, rely on strong assumptions about the data distribution, often assuming it follows a Gaussian distribution. Therefore, it is well known to be very efficient and powerful when data follow a Gaussian distribution.  However, these tests may fail when the distributional assumptions do not hold.
    
    On the other hand, classical nonparametric methods include the \textit{Kolmogorov-Smirnov test} \citep{Kolmogorov1933, Smirnov1939},  \textit{Cramér-von Mises test} \citep{Cramer1928, Mises2013wahrscheinlichkeit, Anderson1962},  \textit{Anderson-Darling test} \citep{Pettitt1976}, and their variants \citep{Dempster1958, kuiper1960tests, NAAMAN2021109088}. These nonparametric tests make no specific distributional assumptions, offering broader applicability in various scenarios and tending to be more robust. 
    These tests are commonly applied in one-dimensional cases. Although generalized versions for higher dimensions have been proposed, such as those by \citet{Bickel1969Smirnov} and \citet{Mardia1967}, generalization to multivariate cases remains challenging.

    Nonparametric approaches based on rank statistics are also classical and well-established in the two-sample testing literature, such as the \textit{Mann-Whitney $U$ test} \citep{mann1947test}, also known as \textit{Wilcoxon-Mann-Whitney test} and the \textit{Wilcoxon test} \citep{wilcoxon1945individual}. 
    They are standard rank-based procedures, appreciated for their robustness and distribution-free properties under the null; however, their power may be limited when discrepancies between distributions extend beyond pure location alternatives \citep{divine2018wilcoxon}.
    Motivated by the need to detect more general alternatives from the null, several rank-based procedures have been proposed, especially in two-sample location-variance problem, including the \textit{Lepage} \citep{lepage1971combination} and \textit{Cucconi} \citep{cucconi1968nuovo} tests, which were originally developed for the joint detection of location and scale differences but have been shown to be powerful against a wide range of general alternatives.
    Moreover, multivariate extensions of these ideas have been developed; for example, \citet{marozzi2020interpoint} propose interpoint-distance-based multivariate versions of the Lepage and Cucconi tests that are applicable in high-dimensional settings, including scenarios with heavy-tailed distributions and dependent variables.

    In recent years, a variety of fields   
    have shown growing interest in developing algorithms based on IPMs.
    Examples can be found in areas such as 
    domain adaptation \citep{NIPS2012_ca8155f4, 8099590}, generative modeling \citep{pmlr-v37-li15, arjovsky2017wasserstein, li2017mmd}, 
    reinforcement learning \citep{bellemare2017distributional},
    distributionally robust optimization \citep{blanchet2019robust, husain2020distributional}, 
    fair representation learning \citep{pmlr-v162-kim22b, kong2025fair},   
    sampling \citep{fan2023optimizing, vargas2023transport}, 
    causal inference \citep{pmlr-v202-kong23d, wang2024rcfr} and quantization \citep{belhadji2025weighted}.
    Different choices of discriminator classes lead to IPMs with distinct theoretical and practical properties.
    When the discriminator class is taken to be the set of 1-Lipschitz continuous functions, the resulting IPM becomes the Wasserstein distance \citep{KR:58}. 
    Alternatively, using a reproducing kernel Hilbert space as the discriminator class gives rise to the Maximum Mean Discrepancy (MMD) \citep{gretton2012kernel}.
    Another option for the discriminator class is to use a specific parametric family of discriminators with desirable theoretical properties \citep{pmlr-v162-kim22b, park2025reluintegralprobabilitymetric}. 

There have been several recent developments for the nonparametric two-sample test based on IPMs, including MMD \citep{gretton2012kernel}, the Wasserstein distance \citep{Ramdas15, Wang2022} and DNN-IPM \citep{Wang2023manifold} (an IPM using deep neural networks for the discriminator class).
While these testing methods exhibit favorable theoretical properties and empirical performances, each has certain limitations.
\citet{song2024generalized} points out that MMD-TST performs well in detecting location shifts, but its power deteriorates when the difference lies in the scale or variance. 
Also, MMD-TST suffers from low power in moderate to high-dimensional settings due to the curse of dimensionality, as their power can decay polynomially or even exponentially with increasing data dimension \citep{Reddi2015}.
Although the Wasserstein distance is widely used in machine learning, its convergence rate is slow in high-dimensional settings \citep{Fournier2013Wasserstein}, indicating that it may not be well-suited for high-dimensional data.
DNN-IPM enjoys a desirable relationship with H\"{o}lder-IPM. However, to satisfy certain theoretical guaranties, the size of deep neural networks used in the discriminator class must grow with the sample size, and associated optimization can be challenging \citep{Wang2023manifold}.

\section{PReLU-TST: two-sample test using PReLU-IPM}
\label{sec_3}

In this section, we introduce a new nonparametric two-sample test based on a new parametric IPM. 
Section \ref{sec_3_1} outlines the motivation behind our approach, highlighting the limitations of existing IPMs. In Section \ref{sec_3_2}, we formally define PReLU-IPM, a parametric IPM that uses the PReLU function
\citep{HeZR015} in the discriminator class, which enables improved model expressiveness and performance via data-driven adaptation compared to existing parametric IPMs.
Section \ref{sec_3_3} studies asymptotic properties of the empirical PReLU-IPM and describes an optimization procedures.
Finally, Section \ref{sec_3_4} presents theoretical guarantees of PReLU-TST, including consistency and asymptotic power analysis under both fixed and local alternatives (with respect to the H\"{o}lder ball).
Unless otherwise specified, we assume that $\mathrm{P}$ and $\mathrm{Q}$ are supported on the unit ball $\mathbb{B}^d$.

\subsection{Motivation} \label{sec_3_1}

IPMs with parametric discriminator classes have been successfully applied to machine learning problems.
Examples are  sigmoid-IPM \citep{pmlr-v162-kim22b} and ReLU-IPM \citep{park2025reluintegralprobabilitymetric},
whose discriminator classes are $\mathscr{F}_{\textsf{sig}} = \left\{ f :\mathbb{B}^d \rightarrow \mathbb{R} \, \big| \, f(\mathbf{z}) = \right.$ $\left.\sigma (\boldsymbol{\theta}^{\top}\mathbf{z} + \mu): 
\mu\in \mathbb{R}, \, \boldsymbol{\theta}\in \mathbb{R}^d \right\},$ 
where $\sigma : x \mapsto (1 + \exp(-x))^{-1}$ is the sigmoid function, and 
$\mathscr{F}_{\textsf{relu}} = \left\{ f :\mathbb{B}^d \rightarrow \mathbb{R} \, | \, f(\mathbf{z}) = (\boldsymbol{\theta}^{\top}\mathbf{z} + \mu)_+, \, \mu \in [-1,1], \right.$ 
$\left. \boldsymbol{\theta} \in \mathbb{S}^{d-1} \right\},$ 
where $(\cdot)_{+}$ is the ReLU function, respectively. Although these discriminator classes are parametric,
the corresponding IPMs are nonparametric in the sense that they can distinguish two distinct probability measures.
That is, $D_{\mathscr{F}}(\mathrm{P},\mathrm{Q})=0$ if and only if $\mathrm{P}=\mathrm{Q}$ for any two probability measures $\mathrm{P}$ and $\mathrm{Q}$.

Even though they work well for estimation problems,
these parametric IPMs are not good at the nonparametric two-sample test since
the powers of the corresponding tests are relatively low for specific alternatives. 
To see why the power of ReLU-TST ($D_{\mathscr{F}_{\textsf{relu}}}$-TST), for example, is unsatisfactory for the location and/or scale alternatives,
we consider the Leaky ReLU-IPM with slope parameter $\ell \in [-1,1]$ (LeLU$_{\ell}$-IPM), where $\ell$ determines the trend of the negative part, corresponding to the IPM with the discriminator class
\begin{align*}
    \mathscr{F}_{\textsf{lelu},\ell}
    \coloneqq \left\{ f : \mathbb{B}^d \rightarrow \mathbb{R} \, \big| \,f(\mathbf{z}) = 
    \phi_{\ell}(\bm{\theta}^{\top}\mathbf{z} + \mu) \ : \ \mu\in [-1,1], \, \bm{\theta} \in \mathbb{S}^{d-1} \right\},
\end{align*}
where $\phi_{\ell}(x) = \max(0, x) + \ell \min(0, x)$
is the Leaky ReLU activation function \citep{maas2013rectifier}.
For simplicity, we denote $\mathscr{F}_{\textsf{lelu},\ell}$ as $\mathscr{F}_{\ell}$ in the remainder of the paper.
Note that ReLU-IPM is the same as the LeLU$_{0}$-IPM. 
To make the discussion simpler, we let $\mu=0$ in $\mathscr{F}_{\ell}.$ 
Let $\mathcal{X}=\mathbb{B}^d$ and let
$\mathrm{P}(\cdot)=F(\cdot)$ for a given distribution function $F$
and let $\mathrm{Q}(\cdot)=F((\cdot-\eta)/\sigma).$
When $\eta>0$ and $\sigma^2=1$ (i.e., the location alternative), the power
of ReLU-TST is expected to be lower than LeLU$_1$-TST ($D_{\mathscr{F}_{\ell}}$-TST with $\ell=1$) since
$D_{\mathscr{F}_{1}}$ simply compares the difference of the means that is known to be best for the location alternative \citep{lehmann2005testing}.
On the other hand, when $\eta=0$ and $\sigma^2>1$ (i.e., the scale alternative), the
power of ReLU-TST is expected to be inferior to LeLU$_{-1}$-TST  ($D_{\mathscr{F}_{\ell}}$-TST with $\ell=-1$)
because  $D_{\mathscr{F}_{-1}}$ compares the mean absolute deviations.

The above discussions suggest treating the slope parameter $\ell$ in the LeLU$_{\ell}$-IPM as a learnable parameter.
By doing so, we expect to select $\ell$ in a data-adaptive manner to achieve high power for both location and scale alternatives
than LeLU$_{\ell}$-IPM with a fixed $\ell,$ which leads us to consider PReLU-IPM, which treats $\ell$ as a learnable parameter.


\subsection{Definition of PReLU-IPM}
\label{sec_3_2}

We define PReLU-IPM as an IPM with the discriminator class
$\mathscr{F}_{\textsf{prelu}}=\cup_{\ell \in [\ell_{\min},\ell_{\max}]} \mathscr{F}_{\ell}.$
That is,
$$
D_{\mathscr{F}_{\textsf{prelu}}}(\mathrm{P}, \mathrm{Q})
= \sup_{\ell \in [\ell_{\min},\ell_{\max}]} \sup_{\substack{\boldsymbol{\theta} \in \mathbb{S}^{d-1}, \\ \mu \in [-1,1]}} \left|
\int \phi_{\ell}(\bm{\theta}^{\top}\mathbf{X} + \mu) \, (d\mathrm{P}-d\mathrm{Q}) \right|.
$$

The following proposition provides a sufficient condition 
for LeLU$_\ell$-IPM to be consistent, which
gives a clue to choosing $\ell_{\min}$ and $\ell_{\max}$ in PReLU-IPM.

\medskip

\begin{proposition} \label{prop_discriminative}
For any 
$\ell \in [-1,1)$ and 
two probability measures $\mathrm{P}$ and $\mathrm{Q}$,
$
D_{\mathscr{F}_{\ell}}(\mathrm{\mathrm{P}}, \mathrm{\mathrm{Q}})=0$ if and only if $\mathrm{\mathrm{P}} \equiv \mathrm{\mathrm{Q}}$.
\end{proposition}

\begin{remark}
Proposition \ref{prop_discriminative} is valid when $|\ell|>1.$
That is, for any $\ell \in \mathbb{R}$, there exists $\ell^{\prime} \in [-1,1]$ and a constant $c \in \mathbb{R}$ such that for all $u \in \mathbb{R}$, $\phi_\ell(u) = c\,\phi_{\ell'}(-u).$
In particular, for $|\ell|>1$, one can take $\ell' = 1/\ell$ and $c = -\ell$, which implies 
$
D_{\mathscr{F}_{\ell}}(\mathrm P,\mathrm Q) =
|c|\,D_{\mathscr{F}_{\ell'}}(\mathrm P,\mathrm Q).
$
Therefore, it is sufficient to restrict $\ell \in [-1,1]$.
\end{remark}

Proposition \ref{prop_discriminative}
implies that $D_{\mathscr{F}_{\ell}}$ is nonparametric unless the activation function is linear.
If $\ell = 1$, then  $\phi_{\ell}(\cdot)$ becomes a linear function, and $D_{\mathscr{F}_{\ell}}$ is no longer consistently separates the null from certain non-null alternatives.
This result suggests that $\ell_{\min}=-1$ and $\ell_{\max}= \xi$ for some $\xi\in [0,1)$
would be a good choice for PReLU-IPM,
as including $\ell \approx 1$ would degrade the testing power due to near-linearity.
In turn, our empirical studies suggest $\xi=0.5.$
See Appendix~\ref{app_slope_range} for empirical evidences and further discussions.

To sum up, PReLU-IPM is LeLU$_{\ell}$-IPM with $\ell \in [-1,0.5).$
An obvious corollary of Proposition \ref{prop_discriminative}
is that PReLU-IPM is nonparametric.
\medskip
\begin{corollary} \label{cor_discriminative}
For any two probability distributions $\mathrm{P}$ and $\mathrm{Q}$, 
$D_{\mathscr{F}_{\textsf{prelu}}}(\mathrm{P}, \mathrm{Q})=0$ if and only if $\mathrm{P} \equiv \mathrm{Q}$. 
\end{corollary}

We now further discuss the interesting relation between PReLU-IPM and other related IPMs, including ReLU-IPM and H\"{o}lder-IPM. We first focus on the relation between ReLU-IPM and PReLU-IPM, which illustrates how PReLU-IPM extends ReLU-IPM.
That is, 
\begin{equation}
\label{eq:relu-prelu}
D_{\mathscr{F}_{\textsf{relu}}}(\mathrm{P},\mathrm{Q}) \leq D_{\mathscr{F}_{\textsf{prelu}}}(\mathrm{P},\mathrm{Q}) \leq \big(1+ \max(|\ell_{\min}|,|\ell_{\max}|)\big) \cdot D_{\mathscr{F}_{\textsf{relu}}}(\mathrm{P},\mathrm{Q}).
\end{equation}
This is because
$\phi_\ell (\boldsymbol{\theta}^{\top}\mathbf{z} + \mu) = (\boldsymbol{\theta}^{\top}\mathbf{z} + \mu)_{+} - \ell (-\boldsymbol{\theta}^{\top}\mathbf{z} - \mu)_{+}$, which shows that each PReLU-IPM can be written as a linear combination of two ReLU-IPMs. Consequently, the induced discriminator class of PReLU-IPM contains that of ReLU-IPM and is itself upper bounded, up to a multiplicative constant, by the ReLU-IPM.
Thus, PReLU-IPM has asymptotic properties similar to those of ReLU-IPM, even though there is an additional parameter $\ell$
in the discriminator class of PReLU-IPM.

Corollary \ref{cor_holder}, which proves that PReLU-IPM can be served as a surrogate IPM of H\"{o}lder-IPM, is a direct consequence of \eqref{eq:relu-prelu} and Theorem 2 of \cite{park2025reluintegralprobabilitymetric}. 
These properties play a crucial role in the theoretical analysis of PReLU-TST. 

\medskip
\begin{corollary}\label{cor_holder}
    For any two probability distributions $\mathrm{P}$ and $\mathrm{Q}$,\\
    (i) if $\beta < \frac{d+3}{2}$, 
    $$
    D_{\mathscr{H}^{\beta, d}}(\mathrm{P},\mathrm{Q}) \leq c_1 \cdot D_{\mathscr{F}_{\textsf{prelu}}}(\mathrm{P},\mathrm{Q})^{\frac{2\beta}{d+3}},
    $$
    and 
    (ii) if $\beta > \frac{d+3}{2}$,
    $$
    D_{\mathscr{H}^{\beta, d}}(\mathrm{P},\mathrm{Q}) \leq c_1 \cdot D_{\mathscr{F}_{\textsf{prelu}}}(\mathrm{P},\mathrm{Q}),
    $$
    with some positive constant $c_1$ that depends on $d$ and $\beta$.
\end{corollary}

\noindent This connection reveals a link between PReLU-IPM and H\"{o}lder-IPM.
Although H\"{o}lder-IPM is defined over a substantially larger function class, Corollary \ref{cor_holder} shows that discrepancies measured by this larger class can already be controlled by the much simpler PReLU-IPM.

\subsection{Estimation of PReLU-IPM}
\label{sec_3_3}

When we observe random samples $\mathbf{X}_1, \ldots, \mathbf{X}_n \overset{\text{i.i.d.}}{\sim} \mathrm{P}$ and  $\mathbf{Y}_1, \ldots, \mathbf{Y}_m \overset{\text{i.i.d.}}{\sim} \mathrm{Q}$, 
it is natural to estimate $D_{\mathscr{F}}(\mathrm{P},\mathrm{Q})$
by $D_{\mathscr{F}}(\widehat{\mathrm{P}}_n,\widehat{\mathrm{Q}}_m),$
where $\widehat{\mathrm{P}}_n$ and $\widehat{\mathrm{Q}}_m$
are empirical measures of $\mathrm{P}$ and $\mathrm{Q},$ respectively.

\subsubsection{Convergence rate} 

The use of the parametric class of discriminators in PReLU-IPM leads to several favorable statistical properties.
The following theorem shows that the empirical PReLU-IPM converges in probability to its population counterpart at a 
parametric rate.

\medskip

\begin{theorem} \label{thm_conv_prelu}
    Let $\mathrm{P}, \mathrm{Q}, \widehat{\mathrm{P}}_n, \widehat{\mathrm{Q}}_m$ be defined on $\mathbb{B}^d$. Then, for some constant $c_2 = \mathcal{O}(\sqrt{d})$,
    $$
    \Big| D_{\mathscr{F}_{\textsf{prelu}}}(\mathrm{P},\mathrm{Q}) - D_{\mathscr{F}_{\textsf{prelu}}}(\widehat{\mathrm{P}}_n, \widehat{\mathrm{Q}}_m) \Big| \leq c_2 \left( \frac{1}{\sqrt{n}} + \frac{1}{\sqrt{m}} \right) + \epsilon,
    $$
    with probability at least $1 - 2 \exp \left(-\frac{\epsilon^2 nm}{8 \Lambda^2 (n+m)} \right)$, where $\Lambda \coloneqq \max(1,|\ell_{\min}|, |\ell_{\max}|)$.
\end{theorem}

The proof of Theorem \ref{thm_conv_prelu} is provided in Appendix \ref{pf_thm_conv_prelu}. When $m=n,$ the convergence rate of the empirical PReLU-IPM is $\mathcal{O}(n^{-1/2}).$ 
Note that the empirical PReLU-IPM achieves a parametric convergence rate
even if it is nonparametric (in the sense of Corollary \ref{cor_discriminative}).
An empirical MMD achieves a similar convergence rate \citep{gretton2012kernel}.
In contrast, the convergence rate of the empirical Wasserstein distance  
is (almost) parametric when $d\le 2$ but it is  $\mathcal{O}(n^{-1/d})$ when $d > 2$ \citep{Sriperumbudur2009IPM, sriperumbudur2012}, and the convergence rate of the empirical DNN-IPM is $\mathcal{O}(n^{-\beta/(2\beta + d)} \log^2 n )$ \citep{Wang2023manifold}.

\subsubsection{Computation}
In practical computations, evaluating  PReLU-IPM involves numerically solving an optimization problem that maximizes the discrepancy of the moments over the discriminator class. This is typically done via a gradient ascent algorithm, which updates the parameters $(\ell, \mu, \boldsymbol{\theta})$ in the discriminator.
However, the following proposition, whose proof is provided in Appendix \ref{pf_thm_max_equal}, shows that it suffices to evaluate the objective only at the two boundary values of $\ell$, thereby significantly simplifying the optimization.

\medskip
\begin{proposition} \label{thm_max_equal}
For any two probability distributions $\mathrm{P}$ and $\mathrm{Q}$ and any $\ell_{\min} \leq \ell_{\max}$,
$$D_{\mathscr{F}_{\textsf{prelu}}}(\mathrm{P}, \mathrm{Q}) = \max\big(D_{\mathscr{F}_{\ell_{\min}}}(\mathrm{P}, \mathrm{Q}), D_{\mathscr{F}_{\ell_{\max}}}(\mathrm{P}, \mathrm{Q})\big).$$
\end{proposition}

For $\mathbf{v} \coloneqq (\mu, \boldsymbol{\theta})$,
let $f_{\phi_\ell, \mathbf{v}}(\mathbf{z}) \coloneqq \phi_\ell (\boldsymbol{\theta}^{\top}\mathbf{z} + \mu)$.
Now our goal is to find $\mathbf{v} \in [-1,1] \times \mathbb{S}^{d-1}$ which maximizes
$|\frac{1}{n} \sum_{i=1}^n f_{\phi_\ell, \mathbf{v}}(\mathbf{X}_i) - \frac{1}{m} \sum_{i=1}^m f_{\phi_\ell, \mathbf{v}}(\mathbf{Y}_i)|$, for each $\ell \in \{\ell_{\min}, \ell_{\max}\}$.
We follow the gradient projection approach on the logarithmically transformed objective,
$$L(\mathbf{v}) = \log \left(\left| \frac{1}{n} \sum_{i=1}^n f_{\phi_\ell, \mathbf{v}}(\mathbf{X}_i) - \frac{1}{m} \sum_{i=1}^m f_{\phi_\ell, \mathbf{v}}(\mathbf{Y}_i) \right|\right).$$
We take the logarithmic transformation of the objective since it has been observed to improve optimization stability and reduce sensitivity to the learning rate \citep{paik2025integralprobabilitymetricsmeet}. 
Given the current parameters $\mu^{\textup{old}}$ and $\boldsymbol{\theta}^{\textup{old}}$, we compute their gradients and update them to new values, $\mu^{\textup{new}}$ and $\boldsymbol{\theta}^{\textup{new}}$, using a gradient ascent step.
We then project these updated values onto the set $[-1,1] \times \mathbb{S}^{d-1}$.
We repeat this step for $T$ iterations.

Since the gradient projection method may converge to a local maximum,
we run the gradient projection method starting from $S$ different random initial points $\{\mathbf{v}_s\}_{s=1}^S :=\{(\mu_s, \boldsymbol{\theta}_s)\}_{s=1}^S$, and then choose the largest IPM value among the results.
The total computational complexity is $\mathcal{O}(S T N d),$
where $S$ is the number of random initializations, $T$ denotes the number of gradient iterations per initialization, and $N$ is the total sample size and $d$ is the dimension of data.
Increasing $S$ improves robustness by reducing sensitivity to poor local optima, while the computational cost increases linearly with $S$.
However, as shown in Fig.~\ref{fig_nprelu_ipm} of Appendix~\ref{app_abl_study}, employing multiple initializations substantially improves test power while incurring only a minor increase in computational time.
This is because $f_{\phi_{\ell_1}, \mathbf{v}_1}, \ldots, f_{\phi_{\ell_S}, \mathbf{v}_S}$ can be computed in parallel (see Appendix~\ref{app_abl_init} for details).
We use $S=100$ as a conservative default that balances stability of the statistic and computational efficiency.
A pseudo-code for estimating the empirical PReLU-IPM is described in Algorithm \ref{alg_prelutst}.

\begin{algorithm}[H]
\small
\caption{Empirical estimate of PReLU-IPM}
\begin{algorithmic}[1]
\Require $S \in \mathbb{N}$, $T \in \mathbb{N}$, $\ell_{\min}, \ell_{\max} \in \mathbb{R}$ s.t. $\ell_{\min} \leq 0 \leq \ell_{\max}$.  
\State \textbf{Input:} $\mathbf{X}^{n} \coloneqq (\mathbf{X}_1, \ldots, \mathbf{X}_n)$ and $\mathbf{Y}^{m} \coloneqq (\mathbf{Y}_1, \ldots, \mathbf{Y}_m)$
\For{$\ell = \ell_{\min}, \ell_{\max}$}
\State Initialize $V = (\mathbf{v}_1, \ldots, \mathbf{v}_S)$
\For{$t = 1, 2, \ldots, T$}
    \State 
    $L(V) = \sum_{s=1}^S \log \left(\left| \frac{1}{n} \sum_{i=1}^n f_{\phi_{\ell}, \mathbf{v}_s}(\mathbf{X}_i) - \frac{1}{m} \sum_{j=1}^m f_{\phi_{\ell}, \mathbf{v}_s}(\mathbf{Y}_j) \right|\right)$
    \State $V \gets \operatorname{Proj}\big(V + \eta \nabla_{V} L(V)\big)$
\EndFor
\State $D_{\mathscr{F}_{\ell}}(\widehat{\mathrm{P}}_n, \widehat{\mathrm{Q}}_m) = 
\max_{s \in [S]} \left| \frac{1}{n} \sum_{i=1}^n f_{\phi_{\ell}, \mathbf{v}_s}(\mathbf{X}_i) - \frac{1}{m} \sum_{j=1}^m f_{\phi_{\ell}, \mathbf{v}_s}(\mathbf{Y}_j) \right|$
\EndFor
\State $D_{\mathscr{F}_{\textsf{prelu}}}(\widehat{\mathrm{P}}_n, \widehat{\mathrm{Q}}_m) = \max\left(D_{\mathscr{F}_{\ell_{\min}}}(\widehat{\mathrm{P}}_n, \widehat{\mathrm{Q}}_m), D_{\mathscr{F}_{\ell_{\max}}}(\widehat{\mathrm{P}}_n, \widehat{\mathrm{Q}}_m)\right)$\\
\Return $D_{\mathscr{F}_{\textsf{prelu}}}(\widehat{\mathrm{P}}_n, \widehat{\mathrm{Q}}_m)$
\end{algorithmic} 
\label{alg_prelutst}
\end{algorithm}

\subsection{PReLU-TST and its theoretical properties}
\label{sec_3_4}

For the sake of simplicity, we assume $n=m$. A level-$\alpha$ PReLU-TST is defined as 
\begin{align}
    T_{\textsf{prelu}, \alpha}(\mathbf{X}^n, \mathbf{Y}^n) \coloneqq \mathbb{I}\big(D_{\mathscr{F}_{\textsf{prelu}}}(\widehat{\mathrm{P}}_n, \widehat{\mathrm{Q}}_n) \ge\eta_{\alpha,n}\big), \label{def_prelutest}
\end{align} 
where the threshold $\eta_{\alpha,n}(=\eta_{\alpha,n,n})$ is chosen to have Type-I error be (asymptotically) $\alpha.$
In this section, we propose three ways of choosing $\eta_{\alpha,n}$ and
study asymptotic powers of PReLU-TST.

\subsubsection{Choice of the threshold}
\label{sec_3_4_1}

We describe three approaches for selecting $\eta_{\alpha,n},$
each providing a valid threshold for controlling the Type-I error.
The first approach uses Theorem \ref{thm_conv_prelu}, which provides a finite-sample upper bound on the test statistic under the null hypothesis. This bound directly yields a valid choice of $\eta_{\alpha,n}$ with guaranteed Type-I error control. The second approach relies on the asymptotic distribution of the test statistic. By characterizing its limiting behavior under the null hypothesis, we can approximate 
$\eta_{\alpha,n}$ using the corresponding asymptotic quantile. Finally, the third approach employs a permutation test, where repeated random permutations of the pooled sample approximate the null distribution empirically. $\eta_{\alpha,n}$ is obtained as the empirical $(1-\alpha)$-quantile of the permuted statistics.

\paragraph{(1) Using Theorem \ref{thm_conv_prelu}:}
We let
$$
\eta_{\alpha, n} \coloneqq \frac{2c_2 + 4 \Lambda \sqrt{\log(2/\alpha)}}{\sqrt{n}},
$$
with the constants $c_2 = \mathcal{O}(\sqrt{d})$ and $\Lambda \coloneqq \max(1,|\ell_{\min}|, |\ell_{\max}|)$ defined in Theorem \ref{thm_conv_prelu}.
Theorem \ref{thm_conv_prelu} implies that the Type-I error of $T_{\textsf{prelu}, \alpha}(\mathbf{X}^n, \mathbf{Y}^m)$ is less than equal to $\alpha.$
A problem, however, of this threshold is that the Type-I error could be too small since
we have chosen the threshold based on the upper bound of the tail probability
of PReLU-IPM which may be quite loose.

\paragraph{(2) Using the asymptotic distribution:}
The following theorem describes the asymptotic behavior of the empirical PReLU-IPM  under the null hypothesis.

\medskip
\begin{theorem}
\label{thm_prelu_tst_consis}
    Under the null hypothesis $\mathbb{H}_0: \mathrm{P}=\mathrm{Q}$,
    $$\sqrt{n} \cdot D_{\mathscr{F}_{\textsf{prelu}}}(\widehat{\mathrm{P}}_n, \widehat{\mathrm{Q}}_n) \overset{d}{\to} \sup_{\ell \in \{\ell_{\min}, \ell_{\max}\}} 
    \sup_{\mu \in [-1,1]} 
    \sup_{ \bm{\theta} \in \mathbb{S}^{d-1}} \big|G(\ell, \mu, \bm{\theta})\big|,$$
    where $G(\ell,\cdot,\cdot) :  [-1,1] \times \mathbb{S}^{d-1} \to \mathbb{R}$ is the Gaussian process with the mean function identically zero and the covariance function $K$ defined by  
    $$K \big((\ell, \mu,\bm{\theta}^{\top})^{\top}, (\tilde{\ell}, \tilde{\mu}, \tilde{\bm{\theta}}^{\top})^{\top}\big) = \mathbb{E}_{\mathbf{X} \sim \mathrm{P}} \Big[2 \phi_\ell (\boldsymbol{\theta}^{\top}\mathbf{X} + \mu)\phi_{\tilde{\ell}} (\tilde{\boldsymbol{\theta}}^{\top}\mathbf{X} + \tilde{\mu}) \Big].  $$    
\end{theorem}

The proof of Theorem \ref{thm_prelu_tst_consis} is provided in Appendix \ref{pf_thm_prelu_tst_consis}.
We use the upper $\alpha$-quantile of $Z:=\sup_{\ell \in \{\ell_{\min}, \ell_{\max}\}} \sup_{\mu \in [-1,1]}  \sup_{ \bm{\theta} \in \mathbb{S}^{d-1}} \big|G(\ell, \mu, \bm{\theta})\big|$ for $\eta_{\alpha, n}$,
which could be obtained by a Monte-Carlo simulation, where
$\mathrm{P}$ is replaced by the empirical distribution of the pooled sample
$(\mathbf{X}^n, \mathbf{Y}^n).$ Generating $Z,$ however, would be computationally demanding, since generating a sample path of a  multivariate Gaussian process is itself computationally intensive.

\paragraph{(3) Using the permutation test:}
The permutation test is a nonparametric resampling-based procedure \citep{pesarin2010permutation,
Kim2022Minimax}. It constructs the \textit{resampling distribution}, also called the \textit{permutation distribution}, of a chosen test statistic by repeatedly permuting (i.e., randomly reassigning) the labels of the observations. Let $\mathbf{U}=(\mathbf{U}_1, \ldots, \mathbf{U}_{2n}) \coloneqq (\mathbf{X}^n, \mathbf{Y}^n)$ denote the pooled sample. Given a permutation $\pi = (\pi_1, \ldots, \pi_{2n})$ of $\{1, \ldots, {2n}\}$, let $\mathbf{U}_{\pi}=(\mathbf{U}_{\pi_{1}}, \ldots, \mathbf{U}_{\pi_{2n}})$ denote the permuted sample. 
Under $\mathbb{H}_0: \mathrm{P} = \mathrm{Q}$, the joint distribution of $\mathbf{U}$ is invariant to that of $\mathbf{U}_{\pi}$, which is a property known as exchangeability.

Let $\widehat{\mu}_{\pi}$ and $\widehat{\nu}_{\pi}$ be the empirical distributions using the permuted samples
$(\mathbf{U}_{\pi_{1}}, \ldots, \mathbf{U}_{\pi_{n}})$ and $(\mathbf{U}_{\pi_{n+1}}, \ldots, \mathbf{U}_{\pi_{2n}})$, respectively. 
The value of the statistic $D_{\mathscr{F}_{\textsf{prelu}}}(\hat{\mu}_\pi,\hat{\nu}_\pi)$ depends only on the subset of indices assigned to each group, and not on their ordering within the groups.
Let $\Pi_{2n}$ denote the subset of all permutations of $\{1,\ldots,2n\}$ such that $\pi_1<\ldots<\pi_n$ and $\pi_{n+1}<\ldots<\pi_{2n}$.
Let $F_{n}$ be the empirical  distribution function of $\{D_{\mathscr{F}_{\textsf{prelu}}}(\widehat{\mu}_{\pi}, \, \widehat{\nu}_{\pi}), \pi \in \Pi_{2n}\}$ defined as
$$
F_{n}(t) = \frac{1}{\binom{2n}{n}} \sum_{\pi \in \Pi_{2n}} \mathbb{I}[D_{\mathscr{F}_{\textsf{prelu}}}(\widehat{\mu}_{\pi}, \widehat{\nu}_{\pi}) \leq t].
$$
Then, we let $\eta_{n,\alpha}$ be  the $1-\alpha$ upper quantile of $F_{n},$ that is,  
$$
\eta_{\alpha, n} \coloneqq \inf \{ t: F_{n}(t) \geq 1-\alpha \}.
$$
The exchangeability ensures that the permutation test $\mathbb{I}(D_{\mathscr{F}_{\textsf{prelu}}}(\widehat{\mathrm{P}}_n, \widehat{\mathrm{Q}}_n) > \eta_{\alpha, n})$ can achieve exact control of the Type-I error rate at level $\alpha$ when testing $\mathbb{H}_0: \mathrm{P} = \mathrm{Q}$ \citep{lehmann2005testing, ChungRomano2013}.
In practice, one can draw and use a random subset of permutations from $\Pi_{2n}$  to save computation time \citep{gretton2012kernel, li2019, Holt_Sullivan_2023}.

\subsubsection{Asymptotic power analysis}

The next theorem proves that PReLU-TST is consistent.
The proof is provided in Appendix \ref{pf_theo_fix}.
\medskip

\begin{theorem}[Under a fixed alternative hypothesis] \label{theo_fix}
    Consider PReLU-TST in \eqref{def_prelutest} with $\eta_{n,\alpha} = o(1)$.
    For any two distributions $\mathrm{P} \neq \mathrm{Q}$, there exists a constant $c_3>0$ such that
    $$
    \mathbb{P}_{\mathrm{P} \neq \mathrm{Q}} \big( T_{\textsf{prelu}, \alpha}(\mathbf{X}^n, \mathbf{Y}^n)=1 \big) \geq 1 -\frac{c_3}{\sqrt{n}},
    $$
    for all sufficiently large $n$.
\end{theorem}

Note that all threshold values proposed in the previous section satisfy $\eta_{\alpha, n} = O(1/\sqrt{n})$. 
Theorem \ref{theo_fix} implies that PReLU-TST is nonparametric, i.e.,
$$\lim\limits_{\substack{n \rightarrow \infty}} \mathbb{P}_{\mathrm{P} \neq \mathrm{Q}} \big( T_{\textsf{prelu}, \alpha}(\mathbf{X}^n, \mathbf{Y}^n)=1 \big) = 1,$$
without requiring any distributional assumptions on $\mathrm{P}$ and $\mathrm{Q}$.

For more refined asymptotic properties of PReLU-TST, we investigate the consistency of PReLU-TST under local alternatives.
We consider the testing problem
\begin{equation} \label{H1:local}
    \mathbb{H}_0: \mathrm{P}=\mathrm{Q} \quad \mathrm{ vs. } \quad \mathbb{H}_1(\kappa_n): D(\mathrm{P},\mathrm{Q}) \geq \kappa_n
\end{equation}
with $\kappa_n \to 0$ as $n \to \infty$, where $D(\cdot, \cdot)$ is a given discrepancy measure
and $\kappa_n$ is called a separation gap.
For the local alternative, we are interested in whether the test $T$ is uniformly consistent over all the alternatives in \eqref{H1:local} as the sample size increases.

For a given test $T,$
the goal is to find the smallest deviation $\kappa_n(T)$, called the \textit{minimum detection boundary} or \textit{minimum separation gap}, between $\mathrm{P}$ and $\mathrm{Q}$, which is the minimum separation gap $\kappa_n$ 
under which the test $T$ is uniformly consistent under the local alternative in (\ref{H1:local}).
That is, $\kappa_n(T)$ is defined as
$$
\kappa_n(T) \coloneqq \inf \left\{ \kappa_n : \lim_{n \xrightarrow{} \infty} 
\inf_{(\mathrm{P},\mathrm{Q})\in \mathbb{H}_1(\kappa_n)} \mathbb{P}_{(\mathrm{P},\mathrm{Q})}\big(T(\mathbf{X}^n, \mathbf{Y}^n)=1 \big) = 1 \right\},
$$
for a given discrepancy measure $D(\cdot)$. 
If the minimum separation gap  of $T$ is equal to the minimax lower bound of the separation gap, we say the $T$
is minimax optimal.
The minimax lower bound of $\kappa_n$ is frequently applied to quantify the intrinsic difficulty of a testing problem  \citep{Ingster1987, li2019, BUTUCEA2006597,Kim2022Minimax}.

For the discrepancy measure $D$ in \eqref{H1:local}, we consider the H\"{o}lder-IPM 
because PReLU-IPM is closely related to H\"{o}lder-IPM as stated in Corollary \ref{cor_holder}.
The theorem below provides the minimum separation gap of PReLU-TST under the H\"{o}lder-IPM local alternatives.

\medskip
\begin{theorem}[Under local alternatives]
\label{thm_minimax}
For any significance level $\alpha \in (0, 1)$,
(i) if $\beta > \frac{d+3}{2}$ and $\kappa_n \gtrsim n^{-1/2} \log n$, or
(ii) if $\beta < \frac{d+3}{2}$ and $\kappa_n \gtrsim n^{-\beta/(d+3)} \log n$,
then the power of PReLU-TST against the local alternative hypothesis 
$\mathbb{H}_1(\kappa_n): D_{\mathscr{H}^{\beta, d}}(\mathrm{P},\mathrm{Q}) \geq \kappa_n$ satisfies
$$
\lim_{n \xrightarrow{} \infty} 
\inf_{(\mathrm{P},\mathrm{Q})\in \mathbb{H}_1(\kappa_n)} \mathbb{P}_{(\mathrm{P},\mathrm{Q})}\big(T_{\textsf{prelu}, \alpha}(\mathbf{X}^n, \mathbf{Y}^n)=1 \big) = 1.$$
\end{theorem}

The proof of Theorem \ref{thm_minimax} is provided in Appendix \ref{pf_thm_minimax}.
Theorem \ref{thm_minimax} demonstrates that, up to a logarithmic factor, PReLU-TST attains the detection boundary of $\mathcal{O} (n^{-\frac{\beta}{d+3}} \vee n^{-\frac{1}{2}}).$  
In comparison, Theorem 1 of \citet{tang2023} establishes the minimax lower bound
of $\kappa_n$ is $O ( n^{-\frac{2 \beta}{d}} \vee n^{-\frac{1}{2}}).$
Therefore, PReLU-TST is minimax optimal when $\beta > \frac{d+3}{2}.$

In contrast, when $\beta < \frac{d+3}{2}$, the detection rate of PReLU-TST becomes suboptimal.
This is nonetheless expected, as PReLU-IPM uses a simple parametric discriminator class.
The suboptimality of PReLU-TST when $\beta < \frac{d+3}{2}$ implies that the power of PReLU-TST may not be good when the difference of the two distributions is highly complicated and subtle. 
We believe that this result is not particularly disappointing, given the simplicity of the discriminator class used in PReLU-TST.
Our numerical studies in Section \ref{sec_4} provide empirical support for this claim. They show that PReLU-TST outperforms other nonparametric tests when the alternative is relatively simple, such as location, scale, and XOR alternatives, 
while remaining competitive in detecting synthetic data generated by complicated neural network models.
In addition, PReLU-TST is simple which would lead its finite sample performance better than its asymptotics. For example,
PReLU-TST dominates Besov-TST by \citet{tang2023}, a minimax optimal test, with large margins.

\section{Empirical experiments}
\label{sec_4}

All experiments are implemented in \texttt{PyTorch}, and PReLU-TST is optimized using the \texttt{torch.optim.} \texttt{SGD} optimizer.
We set the learning rates to $0.1$, and the epochs to 10. To mitigate the impact of local optima, we use $S=100$ random initializations. 
For the computation of PReLU-IPM, we set the slope parameters of the PReLU to $l_{\min} = -1$ and $l_{\max} = 0.5.$

For baselines, we consider IPM-TSTs including
ReLU-IPM \citep{paik2025integralprobabilitymetricsmeet, park2025reluintegralprobabilitymetric},
SIPM \citep{pmlr-v162-kim22b},
MMD \citep{gretton2012kernel}, 
the Wasserstein distance \citep{Ramdas15, Wang2022},
DNN-IPM \citep{Wang2023manifold}
and the negative Besov norm \citep{tang2023}. 
For clarity, we denote the tests based on each IPM as ReLU-TST, SIPM-TST, MMD-TST, Wass-TST, DNN-TST and Besov-TST. 
In addition, we consider multivariate extensions of rank-based tests. Specifically, we consider the Cucconi test \citep{cucconi1968nuovo}, the Lepage test \citep{lepage1971combination}, and the Kolmogorov–Smirnov (KS) test \citep{Kolmogorov1933, Smirnov1939}, as extended by \citet{marozzi2020interpoint}.
We denote these tests as Cucconi-TST, Lepage-TST, and KS-TST.

MMD-TST employs a Gaussian kernel with the bandwidth chosen via the median heuristic \citep{garreau2018largesampleanalysismedian}.
Wass-TST computes the empirical Wasserstein distance using the Python Optimal Transport (\texttt{POT}) library \citep{flamary2021pot}.
For ReLU-TST, SIPM-TST and DNN-TST, although differing in the structure of their discriminators, namely, a single-layer ReLU network (ReLU), a sigmoid-activated network (SIPM), and a two-layer ReLU network (DNN), all of them use similar optimization frameworks as that used for PReLU-TST.
For ReLU-TST and SIPM-TST, we use a learning rate of $0.1$ and train for 10 epochs, with $S=100$ random initials, while
 we use a learning rate of $0.01$, train for 100 epochs for DNN-TST. 
 For Besov-TST, we use the official implementation on GitHub\footnote{\url{https://github.com/rtang1997/Two_sample_test_adversarial}.} and reimplement the method in PyTorch with the wavelet expansion level set to $J=3.$

We compare the powers of PReLU-TST with those of the other test methods, with the significance level being fixed at $\alpha = 0.05$. 
The threshold value of each test is determined via the permutation procedure with 1,000 random permutations. 
The rejection rate is then computed on 1,000 new sample sets to report the mean and standard error of the rejection rates for all test methods.
We normalized the data such that all samples lie within a $d$-dimensional ball of radius $\sqrt{d}$ via the mapping $(\mathbf{X}^n, \mathbf{Y}^m) \mapsto \sqrt{d}(\mathbf{X}^n, \mathbf{Y}^m)/\|(\mathbf{X}^n, \mathbf{Y}^m)\|_{\operatorname{col},2}$.

\subsection{Simulation settings}
\label{sec_4_1}

In this section, we evaluate powers of PReLU-TST by simulation.
For simulated data, we consider three alternatives:
\begin{itemize}
    \item \textit{Location-based shifts}: including a simple mean shift (\texttt{L1}), a mixture-induced shift (\texttt{L2}), and a skewness-driven shift (\texttt{L3}).
    \item \textit{Scale-based changes}: involving a variance change along a single direction (\texttt{S1}), heterogeneous scaling across all directions (\texttt{S2}), and contrasting tail behaviors using a heavy-tailed component via a $t$-distribution (\texttt{S3}).
    \item \textit{XOR-type structures}: synthetic distributions inspired by the XOR problem, where $\mathrm{P}$ and $\mathrm{Q}$ differ through opposing covariances (\texttt{XOR1}) or by forming cluster-based mixtures (\texttt{XOR2}).
\end{itemize}
Each setting is parameterized by values such as $\zeta$, $\alpha$, $\sigma^2$, $v$ and $\rho$, which characterize the discrepancy between $\mathrm{P}$ and $\mathrm{Q}$.
These scenarios are intended to systematically assess the performance of PReLU-TST under a broad spectrum of distributional shifts and discrepancies. 
We evaluate performance across total sample sizes $n + m \in \{200, 400, 600, 800, 1000\}$, with $n \!=\! m$, and dimensions $d \in \{2, 4, 8, 16\}$ in Section \ref{sec_4_2}-\ref{sec_4_4}. 
In Section \ref{sec_4_5}, we consider the small-sample regimes, high dimensional covariates and heavy-tailed distribution settings. Unless otherwise specified, we consider equal sample sizes ($n \!=\! m$).
The task of testing between the two distributions becomes progressively more challenging as the dimension $d$ increases.


\subsection{Simulation 1: Location-based shifts} \label{sec_4_2}

\subsubsection{Settings}
\label{sec_4_2_1}

We let
$
\mathrm{P} = \mathcal{N}(\mathbf{0}_d, \mathrm{I}_d),
$
and consider three different distributions for $\mathrm{Q}$ that represent three different location-based shifts,
\begin{align*}
\texttt{L1}: \mathrm{Q} &= \mathcal{N}\big((\zeta,\mathbf{0}_{d-1}^\top)^\top, \mathrm{I}_d\big), 
& \zeta &= 0.3, \\
\texttt{L2}: \mathrm{Q} &= \tfrac{1}{2}\,\mathcal{N}(\mathbf{0}_d, \mathrm{I}_d) 
+ \tfrac{1}{2}\,\mathcal{N}\big((\zeta,\mathbf{0}_{d-1}^\top)^\top, \mathrm{I}_d\big), 
& \zeta &= 0.5, \\
\texttt{L3}: \mathrm{Q} &= \mathsf{SkewNormal}(0,1,\alpha) 
\;\otimes\; \mathcal{N}(\mathbf{0}_{d-1}, \mathrm{I}_{d-1}), 
& \alpha &= 0.4.
\end{align*}
Here, $\otimes$ indicates a product distribution across axes (e.g., the first axis or first two axes follow the specified distribution, while the remaining $d-1$ or $d-2$ axes are independent standard normals) and we use $\mathcal{N}(\zeta, \Sigma)$ for a multivariate Gaussian distribution with mean $\zeta$ and covariance $\Sigma$, 
and
$\mathsf{SkewNormal}(0,1,\alpha)$ for a univariate skewed normal distribution with location $0$, scale $1$, and skewness parameter $\alpha$.
Thus, $\text{L1}$ corresponds to a mean shift in the first coordinate, $\text{L2}$ is a mixture of the multivariate distributions and a shifted component, and $\text{L3}$ introduces skewness in the first coordinate while keeping the other coordinates unchanged.

\subsubsection{Results}
\label{sec_4_2_2}

Fig. \ref{fig_result1} compares the powers of the test methods across all simulation settings (rows) and data dimensions (columns).
In each panel, the $x$-axis is the sample size and the $y$-axis is the power.  
For location-shift settings,
PReLU-TST consistently outperforms all competing methods across all three scenarios, as shown in Fig. \ref{fig_result1}.
Notably, the performance gap between PReLU-TST and other baselines becomes larger when the dimension $d$ increases. 
In contrast, while ReLU-TST performs comparably to PReLU-TST when $d = 2$, 
its performance deteriorates as $d$ increases, though it still surpasses other baseline methods using infinite dimensional discriminator classes.
Across the various location-shift scenarios, Besov-TST consistently exhibit the lowest powers.
We attribute this to their complex discriminator classes, which tend to yield lower power compared to other methods when the underlying testing problem is relatively simple.

\begin{figure}[h]
    \centering    
    \begin{subfigure}[t]{\textwidth}
        \centering
        \includegraphics[width=\textwidth]{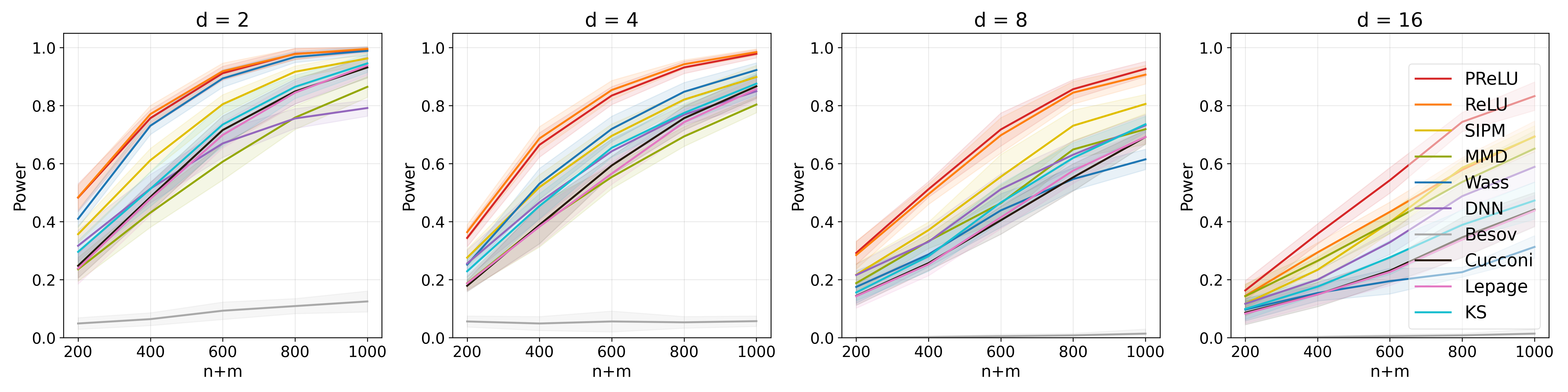}
        \caption{\texttt{L1} (a simple mean shift): $\mathrm{P} \sim \mathcal{N}(\mathbf{0}_d, \mathrm{I}_d), \, \mathrm{Q} \sim \mathcal{N}\big((\zeta,\mathbf{0}_{d-1}^\top)^\top, \mathrm{I}_d\big)$}
        \label{fig:ball}
    \end{subfigure}
    \begin{subfigure}[t]{\textwidth}
        \centering
        \includegraphics[width=\textwidth]{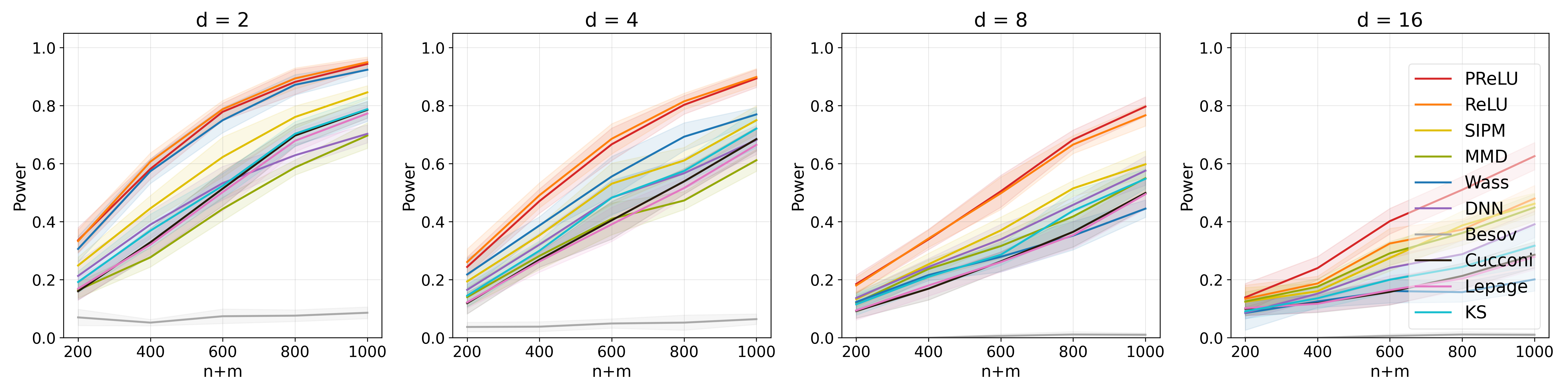}
        \caption{\texttt{L2} (a mixture-induced shift): $\mathrm{P} \sim \mathcal{N}(\mathbf{0}_d, \mathrm{I}_d), \, \mathrm{Q} \sim \frac{1}{2} \mathcal{N}(\mathbf{0}_d, \mathrm{I}_d) + \frac{1}{2}\mathcal{N}\big((\zeta,\mathbf{0}_{d-1}^\top)^\top, \mathrm{I}_d\big)$}
        \label{fig:mixture}
    \end{subfigure}
    \begin{subfigure}[t]{\textwidth}
        \centering
        \includegraphics[width=\textwidth]{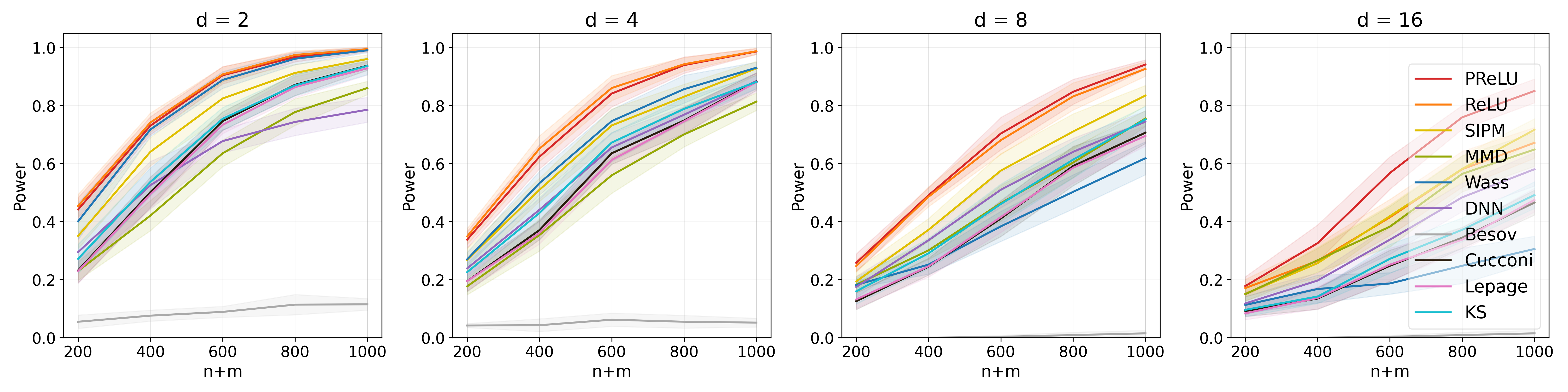}
        \caption{\texttt{L3} (a skewness-driven shift): $\mathrm{P} \sim \mathcal{N}(\mathbf{0}_d, \mathrm{I}_d), \, \mathrm{Q} \sim \mathsf{SkewNormal}(0,1,\alpha) \otimes \mathcal{N}(\mathbf{0}_{d-1}, \mathrm{I}_{d-1})$}
        \label{fig:skewed}
    \end{subfigure}
    \caption{
    Power vs.\ $n+m$ under the location-based shifts for various $d$.
    }
    \label{fig_result1}
\end{figure}


\subsection{Simulation 2: Scale-based changes}
\label{sec_4_3}

\subsubsection{Settings}
\label{sec_4_3_1}

We again take the standard Gaussian for 
$
\mathrm{P} = \mathcal{N}(\mathbf{0}_d, \mathrm{I}_d),
$
and consider three variants of the distribution $\mathrm{Q}$ that differ in their scaling behaviors,
$$
\begin{aligned}
\texttt{S1}: \mathrm{Q} &= \mathcal{N}\big(\mathbf{0}_d, \operatorname{diag}(\sigma^2,\mathbf{1}_{d-1}^\top)\big), 
& \sigma^2 &= 1.8, \\
\texttt{S2}: \mathrm{Q} &= \mathcal{N}\big(\mathbf{0}_d, \operatorname{diag}(\gamma_1,\dots,\gamma_d)\big), 
& \gamma_i &\sim \mathrm{Gamma}(d, 1/d), \\
\texttt{S3}: \mathrm{Q} &= t_v\;\otimes\;\mathcal{N}(\mathbf{0}_{d-1}, \mathrm{I}_{d-1}), 
& v &= 3.0.
\end{aligned}
$$
Here, $\texttt{S1}$ changes the variance in a single coordinate, $\texttt{S2}$ applies heterogeneous scalings across the coordinates with Gaussian-Gamma mixtures, and $\texttt{S3}$ introduces a heavy-tailed component in the first coordinate by replacing the Gaussian distribution with a Student-$t$ distribution with $v$ degrees of freedom, denoted as $t_v$.

\begin{figure}[t]   
    \centering    
    \begin{subfigure}[b]{\textwidth}
        \centering
        \includegraphics[width=\textwidth]{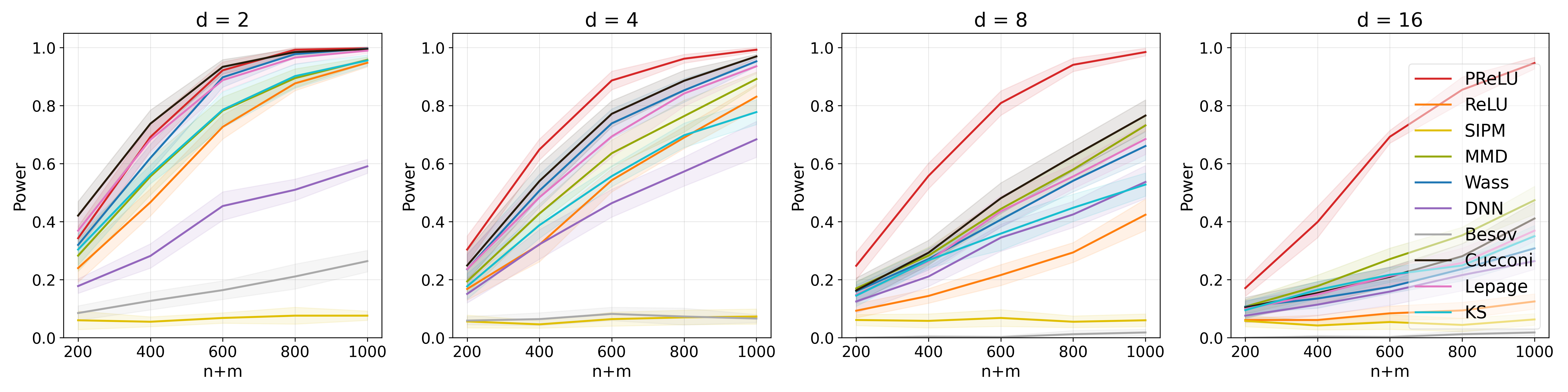}
        \caption{\texttt{S1} (a variance change along a single direction): $\mathrm{P} \sim \mathcal{N}(\mathbf{0}_d, \mathrm{I}_d), \, \mathrm{Q} \sim \mathcal{N}\big(\mathbf{0}_d, \operatorname{diag}(\sigma^2,\mathbf{1}_{d-1}^{\top})\big)$}
        \label{fig:varone}
    \end{subfigure}
    \vspace{0.5em}
    \begin{subfigure}[b]{\textwidth}
        \centering
        \includegraphics[width=\textwidth]{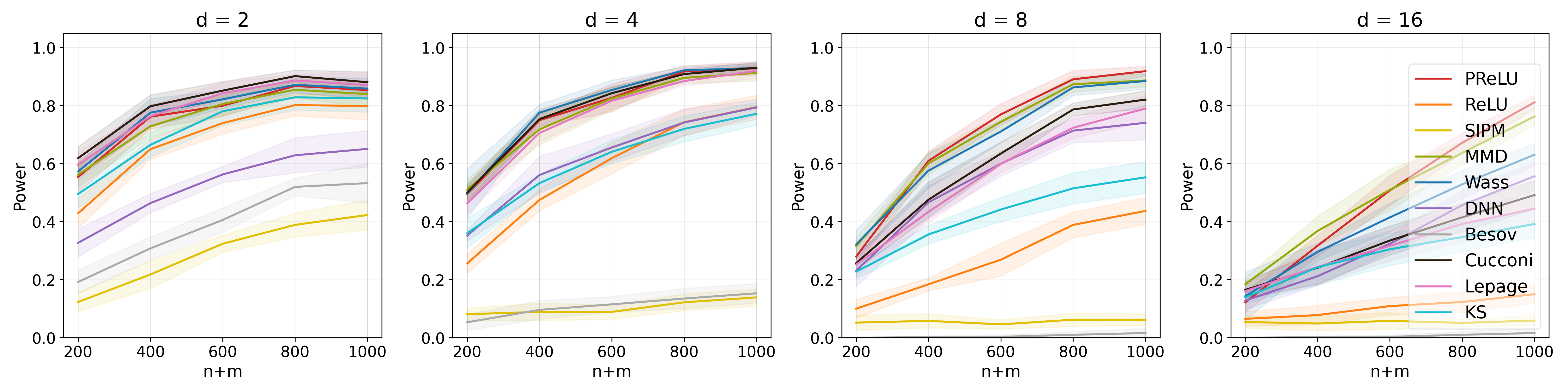}
        \caption{\texttt{S2} (heterogeneous scaling): $\mathrm{P} \sim \mathcal{N}(\mathbf{0}_d, \mathrm{I}_d), \, \mathrm{Q} \sim \mathcal{N}\big(\mathbf{0}_d, \operatorname{diag}(\gamma_1,\dots,\gamma_d)\big), 
        $}
        \label{fig:gamma}
    \end{subfigure}
    \vspace{0.5em}
    \begin{subfigure}[b]{\textwidth}
        \centering
        \includegraphics[width=\textwidth]{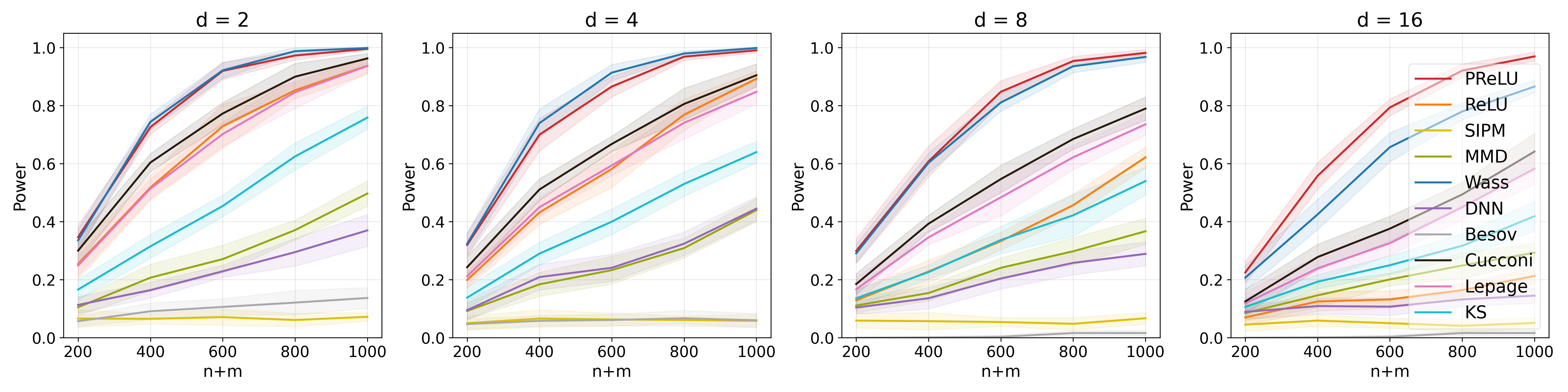}
        \caption{\texttt{S3} (contrasting tail behaviors): $\mathrm{P} \sim \mathcal{N}(\mathbf{0}_d, \mathrm{I}_d), \, \mathrm{Q} \sim t_v \otimes \mathcal{N}(\mathbf{0}_{d-1}, \mathrm{I}_{d-1})$}
        \label{fig:tcoord}
    \end{subfigure}
    \caption{
    Power vs.\ $n+m$ under scale-based change for various $d$.
    }
    \label{fig_result2}
\end{figure}

\subsubsection{Results}
\label{sec_4_3_2}
For the scale-based change settings,
PReLU-TST consistently shows near-best performance across all settings and all dimensions $d$, as shown in Fig.~\ref{fig_result2}.
WASS-TST, Cucconi-TST and Lepage-TST also work reasonably well but their performance quickly drops as the dimension increases, especially in the \texttt{S1} setting.
On the other hand, while PReLU-TST also shows some decline as $d$ increases in the \texttt{S2} setting, it still performs strongly compared to other methods when the sample size is large enough.
Recall that in the location-shift settings, ReLU-TST and SIPM-TST performed similarly to PReLU-TST.
However, in the scale-shift settings, their powers drop clearly, showing that they are not  effective in these types of problem.
Across the various scale-shift scenarios, Besov-TST consistently exhibit the lowest powers.

\subsection{Simulation 3: XOR-type structures}
\label{sec_4_4}

\subsubsection{Settings}
\label{sec_4_4_1}

For the XOR-type experiments, we consider two-dimensional correlated structures while keeping the remaining coordinates independent Gaussians.
First, for $\texttt{XOR1}$, $\mathrm{P}$ and $\mathrm{Q}$ are given as
$$
\begin{aligned}
\mathrm{P} &= \mathcal{N}\!\left(\mathbf{0}_2, 
\begin{bmatrix}
1 & \rho \\ \rho & 1
\end{bmatrix}\right) \;\otimes\; \mathcal{N}(\mathbf{0}_{d-2}, \mathrm{I}_{d-2}), \\
\mathrm{Q} &= \mathcal{N}\!\left(\mathbf{0}_2, 
\begin{bmatrix}
1 & -\rho \\ -\rho & 1
\end{bmatrix}\right) \;\otimes\; \mathcal{N}(\mathbf{0}_{d-2}, \mathrm{I}_{d-2}), 
\quad \rho = 0.2.
\end{aligned}
$$
For $\texttt{XOR2}$, the first two coordinates of $\mathrm{P}$ is a mixture of two Gaussian clusters, while $\mathrm{Q}$ flips the clustering patterns across the axes,
$$
\begin{aligned}
\mathrm{P} &= \left( \tfrac{1}{2}\,\mathcal{N}([\zeta,\zeta]^\top, \mathrm{I}_2) 
+ \tfrac{1}{2}\,\mathcal{N}([-\zeta,-\zeta]^\top, \mathrm{I}_2) \right) \otimes\; \mathcal{N}(\mathbf{0}_{d-2}, \mathrm{I}_{d-2}), \\
\mathrm{Q} &= \left( \tfrac{1}{2}\,\mathcal{N}([\zeta,-\zeta]^\top, \mathrm{I}_2) 
+ \tfrac{1}{2}\,\mathcal{N}([-\zeta,\zeta]^\top, \mathrm{I}_2) \right) \otimes\; \mathcal{N}(\mathbf{0}_{d-2}, \mathrm{I}_{d-2}), 
\quad \zeta = 0.5.
\end{aligned}
$$
Thus, XOR1 differs only by the signs of covariances, whereas XOR2 introduces a cluster-based XOR structure.

\subsubsection{Results}
\label{sec_4_4_2}

The most noteworthy results arise in the XOR-type structure settings. As $\mathscr{F}_{\textsf{prelu}}$ is capable of solving the XOR problem \citep{pinto2024prelu}, PReLU-TST exhibits a particular advantage for XOR-type problems, as clearly demonstrated in Fig.~\ref{fig_result3}. 
Although PReLU-TST performs similarly to Wass-TST and Cucconi-TST when $d = 2$, its superiority becomes more pronounced as the dimension increases and the testing task becomes more challenging due to the increasing difficulty of detecting differences between $\mathrm{P}$ and $\mathrm{Q}$. 
In addition, SIPM-TST, Besov-TST, and KS-TST perform notably poorly in these XOR-type settings, highlighting its weakness in handling such complex distributional structures.

\begin{figure}[t]   
    \centering    
    \begin{subfigure}[b]{\textwidth}
        \centering
        \includegraphics[width=\textwidth]{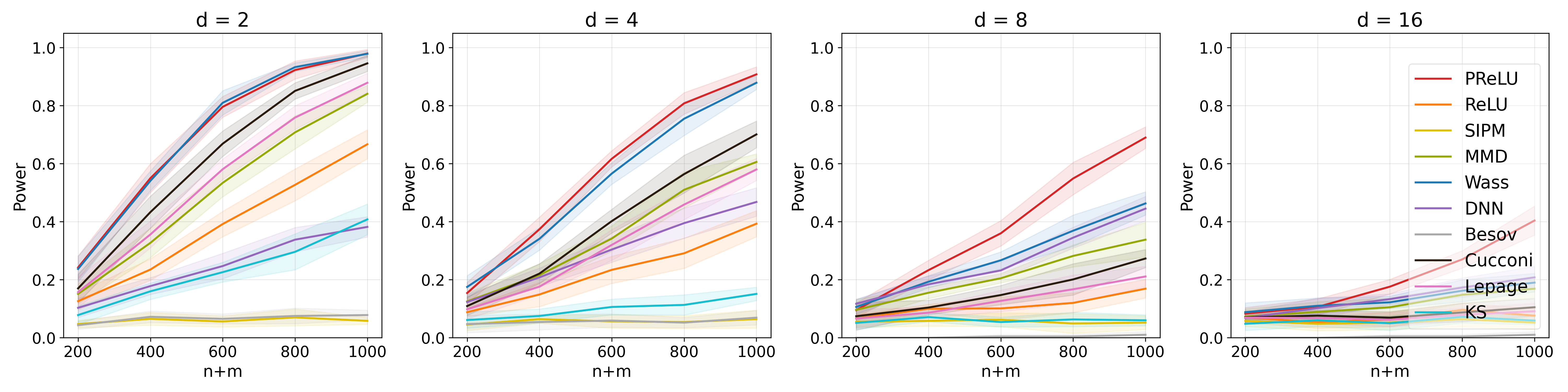}
        \caption{\texttt{XOR1}: $\mathrm{P} \sim \mathcal{N}\big(\mathbf{0}_2,
        \bigl[\begin{smallmatrix}
        1 & \rho \\
        \rho & 1
        \end{smallmatrix}\bigr]\big) \otimes \mathcal{N}(\mathbf{0}_{d-2}, \mathrm{I}_{d-2}), \,
        \mathrm{Q} \sim \mathcal{N}\big(\mathbf{0}_2,
        \bigl[\begin{smallmatrix}
        1 & -\rho \\
        -\rho & 1
        \end{smallmatrix}\bigr]\big) \otimes \mathcal{N}(\mathbf{0}_{d-2}, \mathrm{I}_{d-2})$}
        \label{fig:xor1}
    \end{subfigure}
    \vspace{0.5em}
    \begin{subfigure}[b]{\textwidth}
        \centering
        \includegraphics[width=\textwidth]{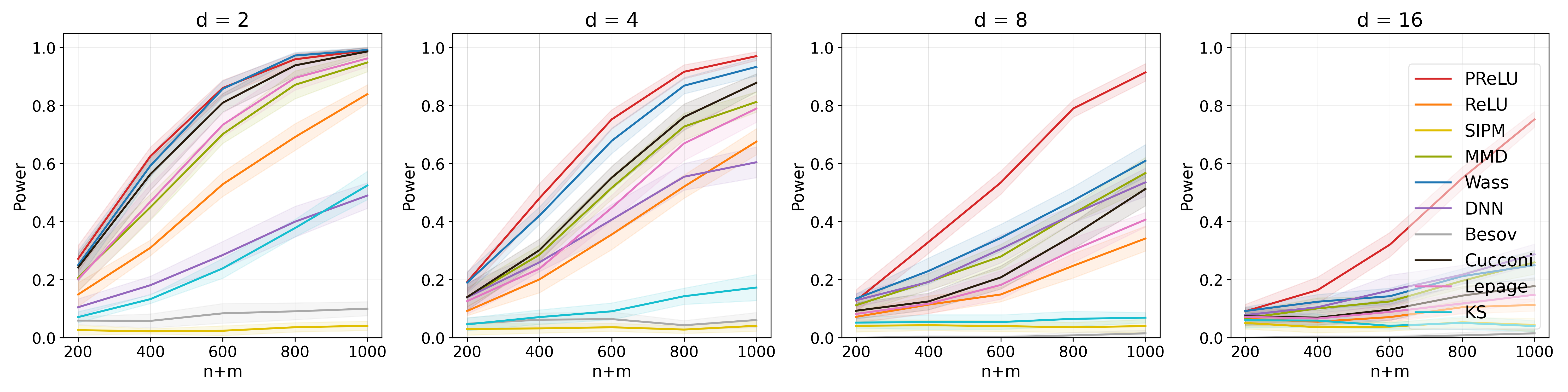}
        \caption{
        \texttt{XOR2}: 
        $ \quad
        \begin{aligned}
        \mathrm{P} &\sim \left( \frac{1}{2}\mathcal{N}\left([\zeta, \zeta]^\top, \mathrm{I}_2\right) + 
        \frac{1}{2}\mathcal{N}\left([-\zeta, -\zeta]^\top, \mathrm{I}_2\right) \right) 
        \otimes\, \mathcal{N}(\mathbf{0}_{d-2}, \mathrm{I}_{d-2}), \\
        \mathrm{Q} &\sim \left( \frac{1}{2}\mathcal{N}\left([\zeta, -\zeta]^\top, \mathrm{I}_2\right) + 
        \frac{1}{2}\mathcal{N}\left([-\zeta, \zeta]^\top, \mathrm{I}_2\right) \right) 
        \otimes\, \mathcal{N}(\mathbf{0}_{d-2}, \mathrm{I}_{d-2}).
        \end{aligned}
        $
        }
        \label{fig:xor2}
    \end{subfigure}
    \caption{
    Power vs.\ $n+m$ under XOR-type structures for various $d$.
    }
    \label{fig_result3}
\end{figure}

In summary, we evaluated PReLU-TST across a variety of scenarios, including location shifts, scale changes, and XOR-type structures, that become more challenging as the dimension $d$ increases. 
The proposed method consistently demonstrates superior performance across all experiments.
Notably, PReLU-TST performs exceptionally well on the XOR problem, a setting that has traditionally been used to highlight the limitations of single-layer networks. This problem had required add more hidden layer to solve effectively, but PReLU-TST overcomes this challenge without the added complexity.


\subsection{Simulation 4: Additional simulations}
\label{sec_4_5}

We further evaluate whether the proposed test maintains strong performance in nonstandard settings, that are not covered by our theoretical analysis but may arise in contemporary applications.
More specifically, we consider test problems under small sample sizes, high-dimensional covariates, and heavy-tailed distributions.
For readability, all figures and tables are deferred to Appendix \ref{app_add_exp}, and here we discuss the key findings.

\subsubsection{Small-sample regimes}

Instead of the sample sizes $n+m \in \{200, 400, 600, 800, 1000\}$ considered in the previous sections, we construct additional simulation settings with smaller sample sizes $n+m \in \{10, 20, 30, 40, 50\}$
and examine whether our proposed method also performs well in this small-sample regime.
When reducing the sample size without modifying the original simulation settings, all tests exhibited very low power, making it difficult to meaningfully compare the methods.
For this reason, we have adjusted the simulation parameters accordingly.
In particular, we consider the following simulation settings: \texttt{L1} with $\zeta=1.0$,
\texttt{L2} with $\zeta=2.0$,
\texttt{L3} with $\alpha=2.0$,
\texttt{S1} with $\sigma^2=6.0$,
\texttt{S3} with $v=1.0$ and
\texttt{XOR2} with $\zeta=1.0$.
We exclude \texttt{S2} and \texttt{XOR1}, as these settings are not easily adjustable.

The corresponding results are presented in Figs.~\ref{fig_small_result1} and \ref{fig_small_result2} in Appendix~\ref{app_add_exp_small_sample}. Overall, the proposed method maintains competitive power across all settings and dimensions even in small-sample regimes.
In particular, PReLU-TST, ReLU-TST, and SIPM-TST perform comparably best under location-based shift alternatives, while PReLU-TST achieves the best performance under scale-based alternatives, followed by WASS-TST.
In the XOR-type structure, MMD-TST, WASS-TST, Cucconi-TST, and Lepage-TST outperform PReLU-TST when $d=2$.
However, in $d=8$ and $d=16$, PReLU-TST achieves the best performance among all tests.

\subsubsection{High-dimensional covariates}

In addition, we introduce additional simulation settings with higher-dimensional covariates $d \in \{258, 512\}$ with the total sample size fixed at $n=m = 100$.
The purpose of these experiments is to examine whether PReLU-TST also performs well in high-dimensional settings where $d>n+m$.
Specifically, we consider the following parameter settings: $\zeta=1.0$ for \texttt{L1}, $\zeta=1.5$ for \texttt{L2}, $\alpha=1.4$ for \texttt{L3}, $\zeta=6.0$ for \texttt{S1}, $v=1.5$ for \texttt{S2}, and $\zeta=1.5$ for \texttt{XOR2}.

The corresponding results can be found in Table~\ref{tab_high_dim} in Appendix~\ref{app_add_exp_high_dim}.
PReLU-TST achieves the best performance under location-based shift settings, while SIPM-TST and MMD-TST attain comparable power.
Under scale-based shift settings, PReLU-TST shows superior performance when $d=256$.
However, its performance decreases substantially when $d=512$, making it inferior to rank-based methods and WASS-TST under \texttt{S1}, while ranking as the second-best method after WASS-TST under \texttt{S3}.
In the XOR-type structure, PReLU-TST and WASS-TST achieve the best performance overall, with PReLU-TST performing better at $d=256$ and WASS-TST outperforming it at $d=512$.
Interestingly, ReLU-TST no longer performs comparably to PReLU-TST across all settings.

\subsubsection{Heavy-tailed distributions}

Finally, we evaluate performance under heavy-tailed null and alternative distributions across total sample sizes $n + m \in \{200, 400, 600, 800, 1000\}$ and dimensions $d \in \{2, 4, 8, 16\}$.
Note that it is hard to extend our theoretical analysis to heavy-tailed distributions.
We consider the settings \texttt{L1-H}, \texttt{L2-H}, \texttt{S1-H}, and \texttt{XOR2-H}, which are modifications of \texttt{L1}, \texttt{L2}, \texttt{S1}, and \texttt{XOR2}, respectively, where the Gaussian distribution is replaced by a Student’s $t$-distribution with $2$ degrees of freedom.
See Appendix~\ref{app_heavy_tailed} for the detailed simulation settings.

 Fig.~\ref{fig_tdist_result} in Appendix~\ref{app_heavy_tailed} presents the corresponding results. 
The empirical results show that PReLU-TST performs relatively poorly in such settings. 
In these cases, MMD-TST as well as rank-based tests (in particular, Lepage-TST and Cucconi-TST) generally achieve better performance.
These results, however, are not very surprising since we assume that data are bounded in our theoretical studies. See Section 3.1 for this assumption.  Moreover, PReLU-TST uses piecewise linear discriminators which are unbounded and thus we could expect that PReLU-TST is not robust to outlier. In contrast, discriminators of MMD-TST are bounded.
The results suggest that, in practice, it may be beneficial to first assess whether the data follow heavy-tailed behavior using a separate test procedure (e.g., the Clauset--Shalizi--Newman test \citep{clauset2009power}), and if not, then apply our method accordingly.

\subsection{Real data analysis}
\label{sec_4_6}


By analyzing multiple real world datasets, we show that
PReLU-TST also works well for complicated alternatives.

\subsubsection{Detecting synthetically generated data}

We consider the problem of the detection of synthetically generated data. 
Specifically, we are given two data sets, the original data and synthetically generated data
by CTAB-GAN \citep{zhao2021ctab}, which is a conditional table GAN architecture that can effectively model diverse data types, including mixtures of continuous and categorical variables, and our goal in this subsection is to
test whether the generative distributions of these two data sets are the same. 

For experiments, we generate a total of four tabular benchmark datasets (\texttt{adult}, \texttt{covertype}, \texttt{loan}, \texttt{thyroid}) from pre-trained CTAB-GAN \citep{zhao2021ctab} models, with the exception of \texttt{adult}\footnote{The real and synthetic \texttt{adult} datasets from CTAB-GAN can be accessed from the following links: \url{https://github.com/Team-TUD/CTAB-GAN/tree/main/Real_Datasets} and \url{https://github.com/Team-TUD/CTAB-GAN/tree/main/Fake_Datasets/Adult}, respectively, from the official GitHub of \citet{zhao2021ctab}. The \texttt{loan} dataset is available at 
\url{https://www.kaggle.com/datasets/teertha/personal-loan-modeling}, 
and the remaining datasets can be downloaded from the 
UCI Machine Learning Repository.}. Table \ref{tab_descriptions} shows the description of the datasets used for the analysis.

\begin{table}[h]
\vspace{-3mm}
\centering
\caption{Description of real and synthetic benchmark datasets.}
\label{tab_descriptions}
\renewcommand{\arraystretch}{1.3}
\resizebox{0.75\textwidth}{!}{
\begin{tabular}{c||cccc}
    \toprule
    Dataset & The number of Real/Synthetic & Type of Target & The number of Inputs \\
    \hline\hline
    \texttt{adult} & 48,842/48,842 & Binary & 13 \\
    \texttt{covertype} & 581,012/581,012 & Multi-class & 54 \\
    \texttt{loan} & 5k/5k & Binary & 13 \\
    \texttt{thyroid} & 7,200/7,200 & Multi-class & 21 \\
    \bottomrule
\end{tabular}
}
\end{table}

\begin{table}[h]
\vspace{-3mm}
\centering
\caption{
Power vs.\ $n+m$ for four benchmark datasets. For each row, the best result is highlighted in bold and the second best result is underlined.
}
\label{tab_power_tabular_dataset}
\renewcommand{\arraystretch}{1.3}
\resizebox{\textwidth}{!}{
\begin{tabular}{c|c||c|cccccccc}
\toprule
Dataset & $n+m$ & PReLU & ReLU & SIPM & MMD & Wass & DNN & Cucconi & Lepage & KS \\
\hline\hline
\multirow{4}{*}{\texttt{adult}} 
 & 50 & \underline{0.341$\pm$0.038} & \textbf{0.342$\pm$0.049} & 0.239$\pm$0.035 & 0.169$\pm$0.038 & 0.280$\pm$0.053 & 0.252$\pm$0.048 & 0.151$\pm$0.063 & 0.140$\pm$0.048 & 0.173$\pm$0.050 \\
 & 100 & \underline{0.580$\pm$0.022} & \textbf{0.602$\pm$0.023} & 0.426$\pm$0.045 & 0.326$\pm$0.061 & 0.563$\pm$0.041 & 0.446$\pm$0.064 & 0.340$\pm$0.040 & 0.338$\pm$0.043 & 0.382$\pm$0.034 \\
 & 200 & \underline{0.926$\pm$0.027} & \textbf{0.928$\pm$0.024} & 0.703$\pm$0.028 & 0.663$\pm$0.037 & 0.899$\pm$0.023 & 0.756$\pm$0.039 & 0.730$\pm$0.037 & 0.717$\pm$0.046 & 0.704$\pm$0.046 \\
 & 400 & \underline{0.998$\pm$0.006} & \textbf{0.999$\pm$0.003} & 0.951$\pm$0.018 & 0.965$\pm$0.019 & 0.998$\pm$0.004 & 0.960$\pm$0.011 & 0.985$\pm$0.014 & 0.985$\pm$0.013 & 0.977$\pm$0.018 \\
\hline
\multirow{4}{*}{\texttt{covertype}}
 & 20 & \textbf{0.290$\pm$0.044} & \underline{0.288$\pm$0.041} & 0.134$\pm$0.033 & 0.222$\pm$0.030 & 0.277$\pm$0.031 & 0.129$\pm$0.025 & 0.137$\pm$0.040 & 0.140$\pm$0.035 & 0.124$\pm$0.027 \\
 & 40 & \textbf{0.592$\pm$0.060} & \underline{0.577$\pm$0.052} & 0.302$\pm$0.060 & 0.461$\pm$0.051 & 0.562$\pm$0.066 & 0.282$\pm$0.056 & 0.337$\pm$0.061 & 0.358$\pm$0.074 & 0.388$\pm$0.044 \\
 & 60 & \textbf{0.805$\pm$0.067} & \underline{0.803$\pm$0.061} & 0.497$\pm$0.048 & 0.663$\pm$0.077 & 0.780$\pm$0.058 & 0.433$\pm$0.057 & 0.555$\pm$0.078 & 0.573$\pm$0.063 & 0.610$\pm$0.064 \\
 & 80 & \textbf{0.921$\pm$0.039} & \underline{0.918$\pm$0.036} & 0.640$\pm$0.049 & 0.835$\pm$0.048 & 0.908$\pm$0.033 & 0.582$\pm$0.036 & 0.723$\pm$0.037 & 0.751$\pm$0.034 & 0.772$\pm$0.032 \\
 \hline
\multirow{4}{*}{\texttt{loan}}
 & 40 & 0.549$\pm$0.045 & 0.562$\pm$0.043 & 0.538$\pm$0.030 & \textbf{0.765$\pm$0.048} & \underline{0.725$\pm$0.065} & 0.511$\pm$0.032 & 0.412$\pm$0.060 & 0.406$\pm$0.053 & 0.411$\pm$0.043 \\
 & 60 & 0.759$\pm$0.030 & 0.769$\pm$0.028 & 0.764$\pm$0.043 & \textbf{0.932$\pm$0.016} & \underline{0.907$\pm$0.023} & 0.705$\pm$0.046 & 0.619$\pm$0.050 & 0.604$\pm$0.059 & 0.598$\pm$0.048 \\
 & 80 & 0.896$\pm$0.032 & 0.897$\pm$0.031 & 0.872$\pm$0.028 & \textbf{0.992$\pm$0.007} & \underline{0.981$\pm$0.010} & 0.833$\pm$0.028 & 0.778$\pm$0.048 & 0.777$\pm$0.040 & 0.746$\pm$0.033 \\
 & 100 & 0.967$\pm$0.014 & 0.949$\pm$0.013 & 0.944$\pm$0.018 & \textbf{0.998$\pm$0.004} & \underline{0.993$\pm$0.006} & 0.893$\pm$0.017 & 0.893$\pm$0.034 & 0.875$\pm$0.030 & 0.854$\pm$0.025 \\
\hline
\multirow{4}{*}{\texttt{thyroid}}
 & 50 & \textbf{0.212$\pm$0.053} & \underline{0.206$\pm$0.051} & 0.167$\pm$0.049 & 0.189$\pm$0.033 & 0.177$\pm$0.039 & 0.201$\pm$0.043 & 0.129$\pm$0.031 & 0.138$\pm$0.027 & 0.128$\pm$0.021 \\
 & 100 & \underline{0.387$\pm$0.054} & \textbf{0.391$\pm$0.049} & 0.242$\pm$0.029 & 0.383$\pm$0.066 & 0.347$\pm$0.042 & 0.326$\pm$0.036 & 0.297$\pm$0.069 & 0.313$\pm$0.059 & 0.234$\pm$0.048 \\
 & 200 & \underline{0.738$\pm$0.039} & 0.732$\pm$0.034 & 0.443$\pm$0.040 & \textbf{0.801$\pm$0.053} & 0.636$\pm$0.042 & 0.613$\pm$0.040 & 0.672$\pm$0.042 & 0.686$\pm$0.048 & 0.479$\pm$0.038 \\
 & 400 & 0.980$\pm$0.013 & 0.974$\pm$0.010 & 0.727$\pm$0.035 & \textbf{0.999$\pm$0.003} & 0.957$\pm$0.018 & 0.872$\pm$0.039 & 0.984$\pm$0.007 & \underline{0.986$\pm$0.009} & 0.875$\pm$0.031 \\
\bottomrule
\end{tabular}}
\end{table}

The results of the comparative analysis of the generative data detection problem are presented in Table \ref{tab_power_tabular_dataset}. The results of Besov-TST are omitted from the experimental comparison, as it consistently underperformed across all settings. As the sample size increases, all methods show improved testing power. 
PReLU-TST achieves either the best performance or competitive results across datasets and sample sizes. 
Although its performance on the \texttt{loan} dataset is relatively weaker compared to the other methods 
when the sample size is small, it shows improvement as the sample size increases. 
MMD-TST and Wass-TST achieve performance comparable to or better than PReLU-TST on \texttt{loan}.
DNN-TST and SIPM-TST generally exhibit lower performance than PReLU-TST and ReLU-TST. However, compared to the rank-based tests, they show relatively better performance on \texttt{adult} and \texttt{loan}. In contrast, on \texttt{covertype} and \texttt{thyroid}, their performance is generally weaker, though they remain competitive at some sample sizes.
Meanwhile, rank-based tests tend to improve in power as the sample size increases, but overall, they underperform compared to PReLU-TST.


\subsubsection{The Yearbook dataset}

The Yearbook dataset contains 37,921 frontal-facing American high-school portraits originally collected from 1905 to 2013\footnote{The original Yearbook dataset can be downloaded from \url{https://www.dropbox.com/s/ubjjoo0b2wz4vgz/faces_aligned_small_mirrored_co_aligned_cropped_cleaned.tar.gz?dl=0} of the official website \url{https://shiry.eecs.berkeley.edu/yearbooks/}.}, representative examples of which are shown in Fig.~\ref{fig_yearbook}. 
It is known to exhibit substantial temporal distribution shifts due to changes in cultural norms, fashion, and photographic conventions over the decades.
For example, \citet{7915779} analyzed systematic increases in smile intensity, changes in hairstyle prevalence, and shifts in the use of glasses over time.
These findings indicate that facial feature distributions evolve gradually but meaningfully across historical periods.
This dataset is widely used as a benchmark for studying distributional shift in images \citep{pmlr-v139-lei21a, yao2022wild, xie2023evolving, zhang2024hypertime}.
We use 512-dimensional feature vectors extracted from the second-to-last layer of a ResNet-18 model pretrained on ImageNet\footnote{We use the processed version provided in the GitHub dataset repository: \url{https://github.com/MachineLearningBCAM/Datasets/tree/main}.}, as standardized image embeddings for our analysis.

In this study, we focus specifically on two historical periods, (i) $1955\!-\!1974$ (4,832 men and 4,555 women) and (ii) $1975\!-\!1994$ (5,415 men and 5,073 women), and aim to determine whether the feature distributions corresponding to these two periods are statistically distinguishable. To avoid trivial demographic confounding, we pool male and female portraits together rather than conducting gender-specific analyses. For each period, we randomly sample $n=m=\{20,30,40,50\}$ images without replacement and perform two-sample testing on the resulting high-dimensional feature vectors.
The objective of this experiment is to assess potential distributional shifts between the two multi-decade time windows using PReLU-TST and to compare its performance with baseline methods.

\begin{figure}[H]
    \centering
    \begin{subfigure}{0.48\linewidth}
        \centering
        \includegraphics[width=\linewidth]{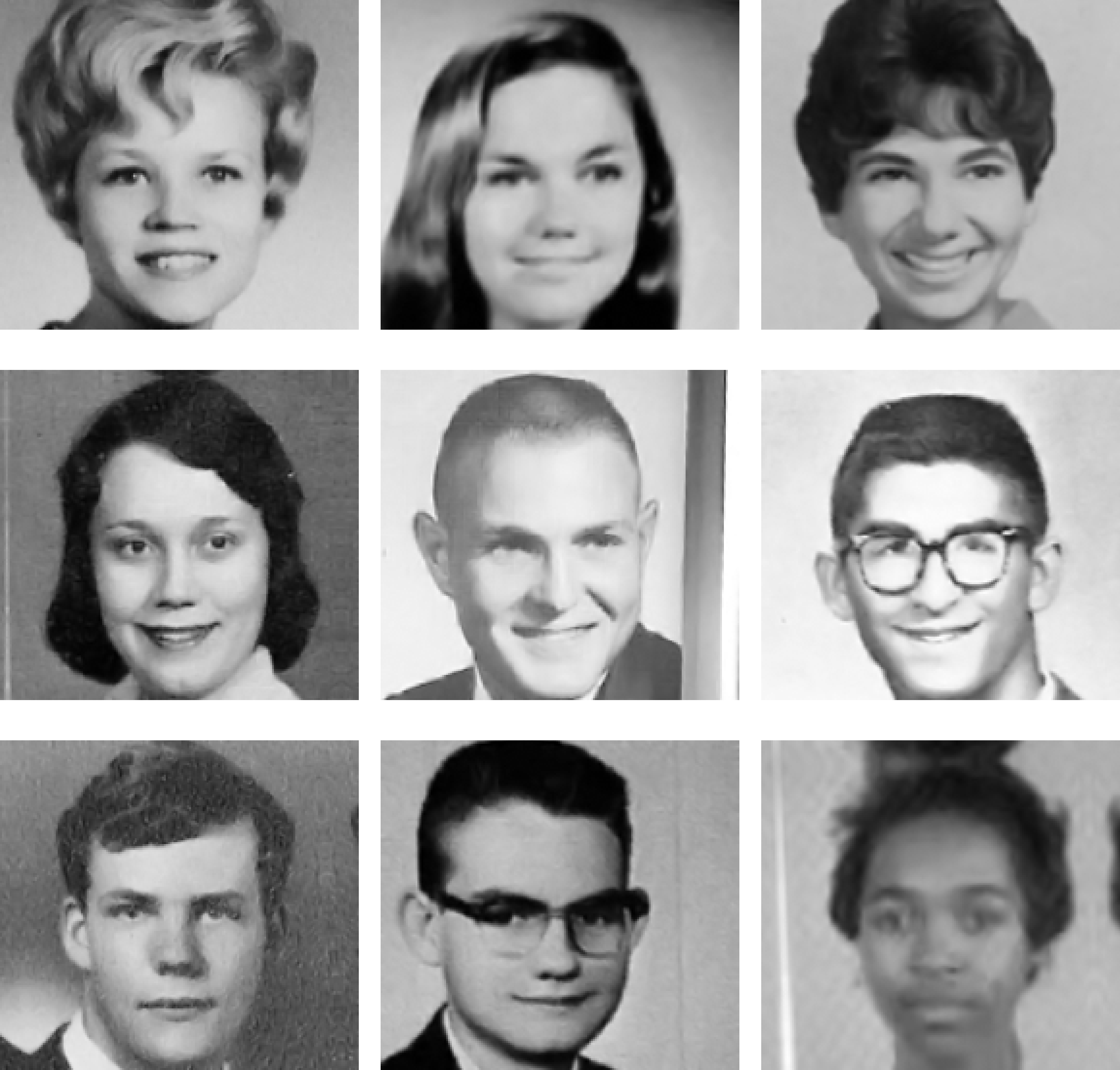}
        \caption{First period: $1955-1974$}
    \end{subfigure}
    \hfill
    \begin{subfigure}{0.48\linewidth}
        \centering
        \includegraphics[width=\linewidth]{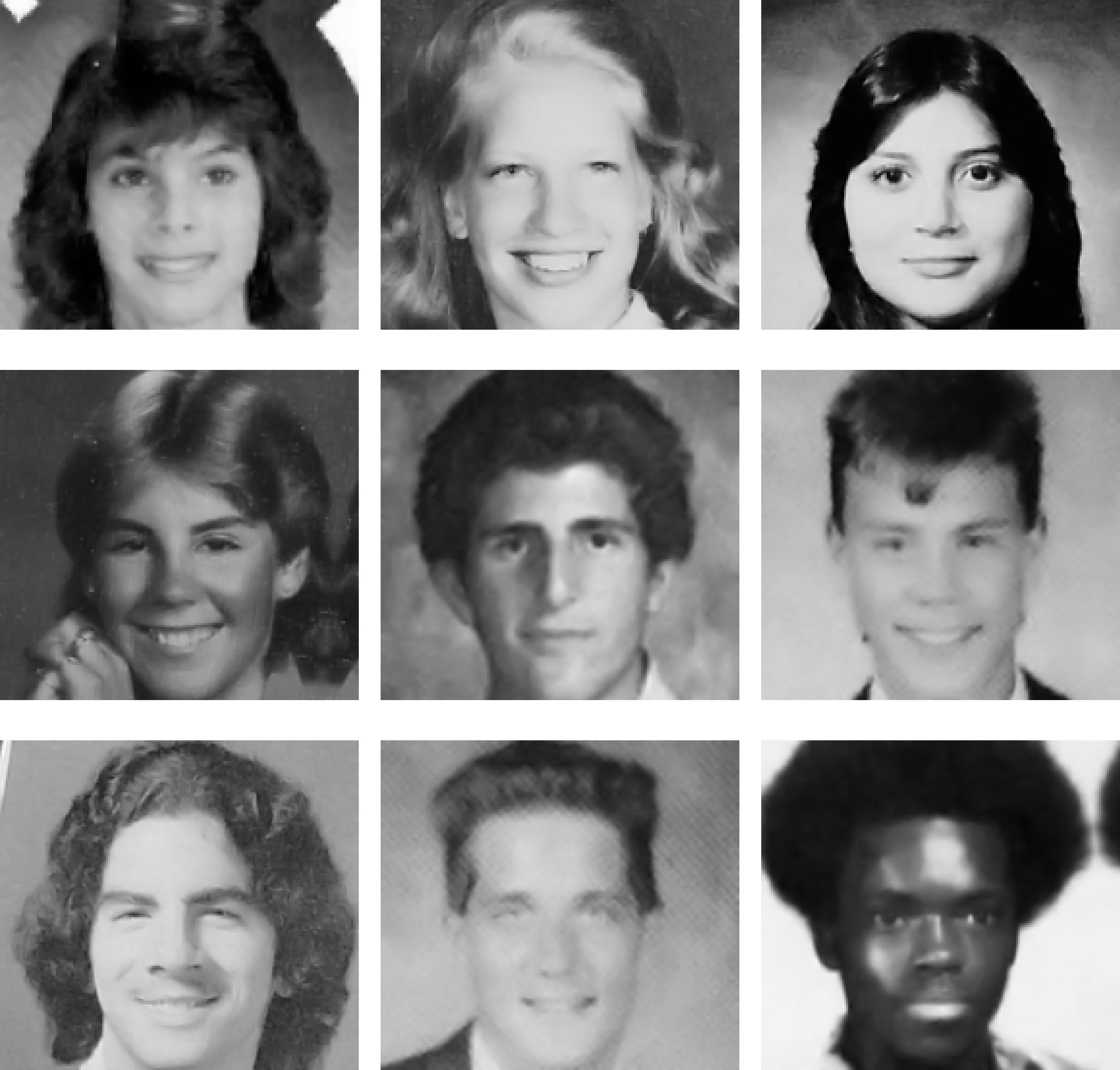}
        \caption{Second period: $1975-1994$}
    \end{subfigure}

    \caption{Portrait examples from two historical periods (i) $1955\!-\!1974$ and (ii) $1975\!-\!1994$ of the original Yearbook dataset.}
    \label{fig_yearbook}
\end{figure}

\begin{table}[h]
\centering
\caption{
Power vs.\ $n+m$ for the Yearbook dataset.
The objective is to assess distributional shifts between two periods $1955\!-\!1974$ and $1975\!-\!1994$.
For each row, the best result is highlighted in bold and the second best result is underlined.
}
\label{tab_yearbook_results}
\renewcommand{\arraystretch}{1.3}
\resizebox{\textwidth}{!}{
\begin{tabular}{c||c|cccccccc}
\toprule
n+m & PReLU & ReLU & SIPM & MMD & Wass & DNN & Cucconi & Lepage & KS \\
\hline \hline
40 & \textbf{0.509$\pm$0.062} & 0.497$\pm$0.060 & 0.189$\pm$0.041 & \underline{0.507$\pm$0.032} & 0.452$\pm$0.064 & 0.000$\pm$0.000 & 0.151$\pm$0.030 & 0.129$\pm$0.030 & 0.133$\pm$0.031 \\
60 & \textbf{0.781$\pm$0.044} & \underline{0.776$\pm$0.047} & 0.308$\pm$0.034 & 0.774$\pm$0.046 & 0.704$\pm$0.038 & 0.159$\pm$0.049 & 0.233$\pm$0.021 & 0.231$\pm$0.022 & 0.216$\pm$0.038 \\
80 & \underline{0.935$\pm$0.014} & \textbf{0.940$\pm$0.021} & 0.455$\pm$0.038 & 0.924$\pm$0.032 & 0.875$\pm$0.050 & 0.388$\pm$0.058 & 0.389$\pm$0.046 & 0.380$\pm$0.048 & 0.337$\pm$0.035 \\
100 & \textbf{0.983$\pm$0.013} & 0.978$\pm$0.012 & 0.599$\pm$0.036 & \underline{0.983$\pm$0.017} & 0.955$\pm$0.024 & 0.607$\pm$0.030 & 0.492$\pm$0.038 & 0.484$\pm$0.040 & 0.431$\pm$0.042 \\
\bottomrule
\end{tabular}
}
\end{table}

\begin{table}[h]
\centering
\caption{Type-I error vs. $n+m$ for the Yearbook dataset across two periods.}
\label{tab_yearbook_H0_results}
\renewcommand{\arraystretch}{1.3}
\resizebox{\textwidth}{!}{
\begin{tabular}{c|c||c|cccccccc}
\toprule
Period & $n+m$ & PReLU & ReLU & SIPM & MMD & Wass & DNN & Cucconi & Lepage & KS \\
\hline \hline
\multirow{4}{*}{$1955\!-\!1974$}
& 40  & 0.066$\pm$0.012 & 0.061$\pm$0.016 & 0.051$\pm$0.018 & 0.054$\pm$0.017 & 0.053$\pm$0.015 & 0.000$\pm$0.000 & 0.047$\pm$0.019 & 0.058$\pm$0.020 & 0.047$\pm$0.021 \\
& 60  & 0.034$\pm$0.022 & 0.043$\pm$0.028 & 0.042$\pm$0.017 & 0.036$\pm$0.024 & 0.047$\pm$0.027 & 0.053$\pm$0.028 & 0.048$\pm$0.018 & 0.049$\pm$0.014 & 0.050$\pm$0.015 \\
& 80  & 0.054$\pm$0.025 & 0.053$\pm$0.024 & 0.047$\pm$0.023 & 0.050$\pm$0.021 & 0.045$\pm$0.012 & 0.049$\pm$0.020 & 0.048$\pm$0.018 & 0.055$\pm$0.021 & 0.054$\pm$0.017 \\
& 100 & 0.049$\pm$0.010 & 0.046$\pm$0.015 & 0.048$\pm$0.021 & 0.049$\pm$0.019 & 0.049$\pm$0.024 & 0.067$\pm$0.019 & 0.068$\pm$0.018 & 0.069$\pm$0.020 & 0.050$\pm$0.030 \\
\hline
\multirow{4}{*}{$1975\!-\!1994$}
& 40  & 0.061$\pm$0.028 & 0.068$\pm$0.032 & 0.043$\pm$0.048 & 0.069$\pm$0.042 & 0.061$\pm$0.029 & 0.000$\pm$0.000 & 0.051$\pm$0.021 & 0.053$\pm$0.025 & 0.058$\pm$0.019 \\
& 60  & 0.044$\pm$0.032 & 0.055$\pm$0.027 & 0.060$\pm$0.019 & 0.051$\pm$0.024 & 0.058$\pm$0.034 & 0.078$\pm$0.013 & 0.047$\pm$0.035 & 0.061$\pm$0.025 & 0.053$\pm$0.021 \\
& 80  & 0.057$\pm$0.016 & 0.058$\pm$0.020 & 0.054$\pm$0.039 & 0.053$\pm$0.017 & 0.063$\pm$0.022 & 0.036$\pm$0.010 & 0.058$\pm$0.016 & 0.056$\pm$0.028 & 0.060$\pm$0.030 \\
& 100 & 0.053$\pm$0.039 & 0.057$\pm$0.036 & 0.056$\pm$0.018 & 0.061$\pm$0.052 & 0.060$\pm$0.028 & 0.042$\pm$0.012 & 0.054$\pm$0.025 & 0.052$\pm$0.026 & 0.065$\pm$0.033 \\
\bottomrule
\end{tabular}
}
\end{table}

In Table \ref{tab_yearbook_results}, we observe the testing power across different sample sizes for the Yearbook dataset\footnote{Besov-TST are omitted due to their consistent underperformance.}. As the total sample size $n+m$ increases, all methods exhibit improved power. In particular, PReLU-TST consistently achieves the highest power across all sample sizes, indicating that it is well suited for detecting distributional shifts in the high-dimensional Yearbook feature space. MMD-TST, ReLU-TST, and Wass-TST also perform competitively, though they remain inferior to PReLU-TST. In contrast, SIPM-TST, DNN-TST and the rank-based tests show significantly weaker performance, especially under small-sample regimes.
The strong performance of PReLU-TST can be attributed to its ability to automatically adjust directional sensitivity through the learnable slope parameter, allowing it to capture subtle distributional differences in the 512-dimensional ResNet-18 embeddings effectively. Although DNN-TST has high model capacity, its performance is likely hindered by instability in the training variance, resulting in relatively inefficient power.

In addition, to verify that the proposed test does not spuriously reject the null hypothesis in homogeneous settings, we conduct a negative control (null) experiment by randomly sampling images within the same time periods into two groups as shown in Table \ref{tab_yearbook_H0_results}.
Since both groups are drawn from the same distribution, the null hypothesis holds by construction.
The empirical rejection rates of PReLU-TST and other TSTs are consistent with the nominal significance level $(\alpha \!\approx\! 0.05)$, confirming proper Type-I error control in high-dimensional settings.


\section{Conclusions}
\label{sec_5}
In this paper, we introduced and studied a new nonparametric two-sample test called PReLU-TST based on a new parametric IPM, called PReLU-IPM.
The key features of PReLU-TST are twofolds:
(i) it offers theoretical guarantees for nonparametric alternatives even if the discriminator class is parametric,
 including consistency  under fixed alternatives, as well as good detection boundaries under local $\beta$-H\"{o}lder alternative in particular when  $\beta > \frac{d+3}{2}$, and
(ii)  it consistently outperforms existing IPM-based methods across a wide range of parametric alternatives
while it maintains powers for more complex alternatives at reasonable levels, provided that few outliers exist.
That is, PReLU-TST is good at detecting structured differences (e.g. location, scale and XOR alternatives), 
but it still performs reasonably well for detecting complicated differences (e.g. synthetic data detection and Yearbook dataset) provided the sample size is not small or the differences are significant. However, PReLU-TST is not recommended when
outliers exist. Thus, we suggest to use PReLU-TST when
the unknown alternative is expected not to be very complex and  outliers seldom exist.

It opens up several avenues for future work. One potential direction is to investigate extensions of PReLU-TST to conditional two-sample test, where covariate information needs to be adjusted. 
For example, one might combine PReLU-TST with other test procedures that are robust to outliers, as PReLU-TST does not perform optimally under heavy-tailed distributions.
Another is to explore  new discriminator classes which compromise advantages of parametric and nonparametric discriminator classes.
An example would a semi-parametric discriminator class where the activation function in PReLU-IPM is replaced
by nonparametric functions or a  deep neural network having special architectures (e.g. functional ANOVA composition).
Finally, connections to adversarial learning and generative models may be further exploited to design even more expressive test statistics.

\section{Disclosure statement}
\label{disclosure-statement}

No potential conflict of interest was reported by the author(s).

\section{Data availability statement}
\label{data-availability-statement}

The data that support the findings of this study are openly available at \url{https://github.com/Team-TUD/CTAB-GAN} (Adult dataset) and \url{https://shiry.eecs.berkeley.edu/yearbooks/} (Yearbook dataset).

\section{Acknowledgments}

This work was partly supported by the National Research Foundation of Korea(NRF) grant funded by the Korea government(MSIT) (RS-2025-00556079), Korea Institute of Marine Science \& Technology Promotion(KIMST) funded by the Ministry of Oceans and Fisheries(Grant RS-2024-00430547) and Next-generation Intelligence semiconductor R\&D Program through the National Research Foundation of Korea(NRF) funded by the Korea government(MSIT)(RS-2024-00431718).

\clearpage

\appendix

\section{Proofs}

\subsection{Proof of Proposition \ref{prop_discriminative}} \label{pf_prop1_and_2}


We begin with the following lemma, which plays a key role in the proof of Proposition \ref{prop_discriminative}.

\medskip

\begin{lemma} \label{lem_discr_lrelu}
    For two random vectors $\mathbf{X} \sim \mathrm{P}$ and $\mathbf{Y} \sim \mathrm{Q}$, and all $\ell \neq 1$,
    if $D_{\mathscr{F}_{\ell}}(\mathrm{P}, \mathrm{Q}) = 0$, then 
    \begin{align}
        \sup_{\bm{\theta} \in \mathbb{S}^{d-1}} \sup_{t \in \mathbb{R}} \left| \mathbb{P} \left( \bm{\theta}^{\top} \mathbf{X} \leq t \right) - \mathbb{P} \left( \bm{\theta}^{\top} \mathbf{Y} \leq t \right) \right| = 0.
    \end{align}
\end{lemma}

\begin{proof}[Proof of Lemma \ref{lem_discr_lrelu}]
    By the definition, $D_{\mathscr{F}_{\ell}}(\mathrm{P}, \mathrm{Q})=0$ implies 
    \begin{equation} \label{equaldistance}
        \left| \mathbb{E} \phi_{\ell}(\bm{\theta}^{\top}\mathbf{X}+\mu) - \mathbb{E} \phi_{\ell}(\bm{\theta}^{\top}\mathbf{Y}+\mu) \right| = 0,  
    \end{equation}
    for any $\bm{\theta} \in \mathbb{S}^{d-1}$ and $\mu \in \mathbb{R}$.
    For given $\bm{\theta} \in \mathbb{S}^{d-1}$, we consider $t \in \mathbb{R}$ such that random variables $\bm{\theta}^{\top} \mathbf{X}$ and $\bm{\theta}^{\top} \mathbf{Y}$ do not have a point mass at $t$.    
    For $\delta>0$, let $\Delta_{\delta}(\mathbf{X}) \coloneqq \phi_{\ell}(\bm{\theta}^{\top} \mathbf{X}-t+\delta) - \phi_{\ell}(\bm{\theta}^{\top} \mathbf{X}-t)$. Then, we can rewrite
    \begin{align}
        \Delta_{\delta}(\mathbf{X}) 
        & = (\bm{\theta}^{\top}\mathbf{X}-t+\delta)_{+} + \ell (\bm{\theta}^{\top}\mathbf{X}-t+\delta)_{-} -\left\{ (\bm{\theta}^{\top}\mathbf{X}-t)_{+} + \ell (\bm{\theta}^{\top}\mathbf{X}-t)_{-} \right\} \nonumber \\
        & = \ell \delta \cdot \mathbb{I} (\bm{\theta}^{\top} \mathbf{X} \leq t-\delta) + \left\{(1-\ell)(\bm{\theta}^{\top} \mathbf{X}-t)+\delta \right\} \cdot \mathbb{I}(t-\delta < \bm{\theta}^{\top} \mathbf{X} \leq t) + \delta \cdot \mathbb{I}(t < \bm{\theta}^{\top} \mathbf{X}). \label{Delta}
    \end{align}
    From (\ref{Delta}), we obtain    
    \begin{align}
        & \left|\mathbb{P} (\bm{\theta}^{\top} \mathbf{X} > t) - \frac{1}{\delta(1-\ell)} \left(\mathbb{E} \Delta_{\delta}(\mathbf{X}) - \ell \delta \right) \right| \nonumber \\        
        & = \left| \mathbb{P}(\bm{\theta}^{\top} \mathbf{X} > t) - \left\{ \frac{\ell}{1-\ell} \mathbb{P}(\bm{\theta}^{\top} \mathbf{X} \leq t-\delta) + \mathbb{E} \left( \frac{(1-\ell)(\bm{\theta}^{\top} \mathbf{X}-t)+\delta}{\delta(1-\ell)} \cdot \mathbb{I}(t-\delta < \bm{\theta}^{\top} \mathbf{X} \leq t) \right) \right. \right. \nonumber \\
        & \qquad + \left. \left. \frac{1}{1-\ell} \mathbb{P}(\bm{\theta}^{\top} \mathbf{X} > t) - \frac{\ell}{1-\ell} \right\} \right| \nonumber \\
        & = \left| \mathbb{P}(\bm{\theta}^{\top} \mathbf{X} > t) - \frac{\ell}{1-\ell} \mathbb{P}(\bm{\theta}^{\top} \mathbf{X} \leq t-\delta) - \frac{1}{1-\ell} \mathbb{P}(\bm{\theta}^{\top} \mathbf{X} > t) + \frac{\ell}{1-\ell}  \right. \nonumber \\
        & \qquad - \left. \mathbb{E} \left( \frac{(1-\ell)(\bm{\theta}^{\top} \mathbf{X}-t)+\delta}{\delta(1-\ell)} \cdot \mathbb{I}(t-\delta < \bm{\theta}^{\top} \mathbf{X} \leq t) \right) \right| \nonumber \\
        & = \left| - \frac{\ell}{1-\ell} \left\{ \mathbb{P}(\bm{\theta}^{\top} \mathbf{X} > t) + \mathbb{P}(\bm{\theta}^{\top} \mathbf{X} \leq t-\delta) -1 \right\} \right| \nonumber \\
        & \qquad + \left| \mathbb{E} \left( \frac{(1-\ell)(\bm{\theta}^{\top} \mathbf{X}-t)+\delta}{\delta(1-\ell)} \cdot \mathbb{I}(t-\delta < \bm{\theta}^{\top} \mathbf{X} \leq t) \right) \right| \nonumber \\
        & \leq \left| - \frac{\ell}{1-\ell} \left\{ 1- \mathbb{P}(t-\delta < \bm{\theta}^{\top} \mathbf{X} \leq t) -1 \right\} \right| 
        + \left| \frac{1}{1-\ell} \cdot \mathbb{P} (t-\delta < \bm{\theta}^{\top} \mathbf{X} \leq t) \right| \nonumber \\
        & \leq \left| \frac{\ell}{1-\ell} \right| \cdot \mathbb{P}(t-\delta < \bm{\theta}^{\top} \mathbf{X} \leq t) + \left| \frac{1}{1-\ell} \right| \cdot \mathbb{P}(t-\delta < \bm{\theta}^{\top} \mathbf{X} \leq t) \nonumber \\
        & \leq \frac{|\ell|+1}{|1-\ell|}  \cdot \mathbb{P}(t-\delta < \bm{\theta}^{\top} \mathbf{X} \leq t), \label{pro_exp1}
    \end{align}
    where we use $\ell \delta \leq (1-\ell)(\bm{\theta}^{\top} \mathbf{X}-t)+\delta \leq \delta$ for the first inequality. 
    
    \noindent Similarly, for $\Delta_{\delta}(\mathbf{Y}) \coloneqq \phi_{\ell}(\bm{\theta}^{\top} \mathbf{Y}-t+\delta)- \phi_{\ell}(\bm{\theta}^{\top} \mathbf{Y}-t)$, we have
    \begin{align}
        \left|\mathbb{P} (\bm{\theta}^{\top} \mathbf{Y} > t) - \frac{1}{\delta(1-\ell)} \left(\mathbb{E} \Delta_{\delta}(\mathbf{Y}) - \ell \delta \right) \right| 
        \leq \frac{|\ell|+1}{|1-\ell|}  \cdot \mathbb{P}(t-\delta < \bm{\theta}^{\top} \mathbf{Y} \leq t). \label{pro_exp2}
    \end{align}
    Therefore, from (\ref{equaldistance}), (\ref{pro_exp1}) and (\ref{pro_exp2}), we get
    \begin{align*}
        & \Big| \mathbb{P}(\bm{\theta}^{\top} \mathbf{X} \leq t) - \mathbb{P}(\bm{\theta}^{\top} \mathbf{Y} \leq t) \Big| \\
        & = \left| \mathbb{P}(\bm{\theta}^{\top} \mathbf{X} > t) - \mathbb{P}(\bm{\theta}^{\top} \mathbf{Y} > t) \right| \\
        & \leq \left|\frac{1}{\delta(1-\ell)} \left(\mathbb{E} \Delta_{\delta}(\mathbf{X}) - \ell \delta \right) 
        - \frac{1}{\delta(1-\ell)} \left(\mathbb{E} \Delta_{\delta}(\mathbf{Y}) - \ell \delta \right)  \right| \\
        & \quad + \frac{|\ell|+1}{|1-\ell|}  \cdot \mathbb{P}(t-\delta < \bm{\theta}^{\top} \mathbf{X} \leq t) \\
        & \quad + \frac{|\ell|+1}{|1-\ell|}  \cdot \mathbb{P}(t-\delta < \bm{\theta}^{\top} \mathbf{Y} \leq t) \tag{from \eqref{pro_exp1} and \eqref{pro_exp2}} \\
        & \leq \frac{1}{\delta (1-\ell)} \big| \mathbb{E} \phi_{\ell}(\bm{\theta}^{\top} \mathbf{X}-t+\delta) - \mathbb{E} \phi_{\ell}(\bm{\theta}^{\top} \mathbf{X}-t) 
        - \mathbb{E} \phi_{\ell}(\bm{\theta}^{\top} \mathbf{Y}-t+\delta) + \mathbb{E} \phi_{\ell}(\bm{\theta}^{\top} \mathbf{Y}-t) \big| \\
        & \quad + \frac{|\ell|+1}{|1-\ell|}  \cdot \left\{\mathbb{P}(t-\delta < \bm{\theta}^{\top} \mathbf{X} \leq t) 
        + \mathbb{P}(t-\delta < \bm{\theta}^{\top} \mathbf{Y} \leq t) \right\} \\
        & = \frac{|\ell|+1}{|1-\ell|}  \cdot \left\{\mathbb{P}(t-\delta < \bm{\theta}^{\top} \mathbf{X} \leq t) 
        + \mathbb{P}(t-\delta < \bm{\theta}^{\top} \mathbf{Y} \leq t) \right\}. \tag{by \eqref{equaldistance}}
    \end{align*}
    Since $\delta>0$ is arbitrary, for every $\ell \neq 1$, we obtain
    \begin{align}
        \Big| \mathbb{P}(\bm{\theta}^{\top} \mathbf{X} \leq t) - \mathbb{P}(\bm{\theta}^{\top} \mathbf{Y} \leq t) \Big| 
        & \leq \frac{|\ell|+1}{|1-\ell|}  \cdot \left\{ \lim_{\delta \downarrow 0} \mathbb{P}(t-\delta < \bm{\theta}^{\top} \mathbf{X} \leq t) 
        + \lim_{\delta \downarrow 0} \mathbb{P}(t-\delta < \bm{\theta}^{\top} \mathbf{Y} \leq t) \right\} \nonumber \\
        & = 0. \label{tmp_999}
    \end{align}    
    
    For the case where either $\bm{\theta}^{\top} \mathbf{X}$ or $\bm{\theta}^{\top} \mathbf{Y}$ has a point mass at $t$, we can construct a sequence $\{t_j\}_{j = 1}^{\infty}$ such that $t_j \downarrow t$ and neither $\bm{\theta}^{\top} \mathbf{X}$ nor $\bm{\theta}^{\top} \mathbf{Y}$ has a point mass at $\{t_j\}_{j = 1}^{\infty}$.
    Since $\mathbb{P}(\bm{\theta}^{\top} \mathbf{X} \leq \cdot)$ and  $\mathbb{P}(\bm{\theta}^{\top} \mathbf{Y} \leq \cdot)$ is right-continuous, 
    \begin{align*}
         \left| \mathbb{P}(\bm{\theta}^{\top} \mathbf{X} \leq t) - \mathbb{P}(\bm{\theta}^{\top} \mathbf{Y} \leq t) \right| = \lim_{j \rightarrow \infty} & \left| \mathbb{P}(\bm{\theta}^{\top} \mathbf{X} \leq t_j) - \mathbb{P}(\bm{\theta}^{\top} \mathbf{Y} \leq t_j) \right| = 0,
    \end{align*}
    and the proof is done.
\end{proof}

\begin{proof}[Proof of Proposition \ref{prop_discriminative}]
We prove the statement by showing both directions.

\noindent $(\Leftarrow)$
This holds since, for any $f \in \mathscr{F}_{\ell}$ and all $\ell \neq 1$, we have 
$$
\int f(\mathbf{X})(d\mathrm{\mathrm{P}}-d\mathrm{\mathrm{Q}}) = 0.
$$

\noindent $(\Rightarrow)$
Lemma \ref{lem_discr_lrelu} implies 
$\bm{\theta}^{\top} \mathbf{X} \stackrel{\textup{d}}{=} \bm{\theta}^{\top} \mathbf{Y}$ for all $\bm{\theta} \in \mathbb{R}^d$, 
which, in turn, implies $\mathrm{P} \equiv \mathrm{Q}$ by the uniqueness of the characteristic function.
\end{proof}


\subsection{Proof of Corollary \ref{cor_discriminative} and \ref{cor_holder}} \label{pf_cor1_and_2}


\begin{proof}[Proof of Corollary \ref{cor_discriminative}]
We prove the statement by showing both directions:

\noindent $(\Rightarrow)$ 
Since $D_{\mathscr{F}_{\textsf{prelu}}}(\mathrm{P}, \mathrm{Q})=0$ indicates that $D_{\mathscr{F}_{\textsf{relu}}}(\mathrm{P}, \mathrm{Q})=0$, 
we can complete the proof using
Proposition 1 of \citet{park2025reluintegralprobabilitymetric}.

\noindent $(\Leftarrow)$ This holds because, for any $f \in \mathscr{F}_{\textsf{prelu}}$, we have 
$
\int f(\mathbf{z})(d\mathrm{P}(\mathbf{z})-d\mathrm{Q}(\mathbf{z})) = 0.
$
\end{proof}


\begin{proof}[Proof of Corollary \ref{cor_holder}]
Since $D_{\mathscr{F}_{\textsf{relu}}}(\mathrm{P}, \mathrm{Q}) \leq D_{\mathscr{F}_{\textsf{prelu}}}(\mathrm{P}, \mathrm{Q})$, the proof can be easily completed using Theorem 2 of \citet{park2025reluintegralprobabilitymetric}.
\end{proof}


\subsection{Proof of Theorem \ref{thm_conv_prelu}}
\label{pf_thm_conv_prelu}

As the first step, we derive the covering number of $\mathscr{F}_{\textsf{prelu}}$.

\medskip

\begin{lemma}[Covering Number of $\mathscr{F}_{\textsf{prelu}}$] \label{lem_cover_number_prelu}
For any $\epsilon>0$, we have that
\begin{align*}
    \mathcal{N} \left( \epsilon, \mathscr{F}_{\textsf{prelu}}, \|\cdot\|_{\infty} \right) \leq \left( 1+ \frac{1}{\epsilon} \right)^{d+2}.
\end{align*}
\end{lemma}

\begin{proof}[Proof of Lemma \ref{lem_cover_number_prelu}]
We write $\Gamma := 1+\max(|\ell_{\min}|,|\ell_{\max}|)$.
It is easy to see that there exists $\epsilon/(3\Gamma)$-cover $\Theta_{\epsilon}$ of $\mathbb{S}^{d-1}$ in the metric $\|\cdot\|_2$ with cardinality $|\Theta_{\epsilon}| \leq \mathcal{O}(1+1/\epsilon)^{d}$, $\epsilon/(3\Gamma)$-cover $M_{\epsilon}$ of $[-1,1]$ in the metric $\|\cdot\|_1$ with cardinality $|M_{\epsilon}| \leq \mathcal{O}(1+1/\epsilon)$, and $\epsilon/6$-cover $L_{\epsilon}$ of $[\ell_{\min},\ell_{\max}]$ in the metric $\|\cdot\|_1$ with cardinality $|L_{\epsilon}| \leq \mathcal{O}(1+1/\epsilon)$.
Hence, for any $\boldsymbol{\theta} \in \mathbb{S}^{d-1}$, $|\mu| \leq 1$, and $\ell \in [\ell_{\min},\ell_{\max}]$, we can choose $\Tilde{\boldsymbol{\theta}} \in \Theta_{\epsilon}$, $\Tilde{\mu} \in [-1,1]$, and $\Tilde{\ell} \in [\ell_{\min},\ell_{\max}]$ such that $\|\boldsymbol{\theta} - \Tilde{\boldsymbol{\theta}}\|_2 \leq \epsilon/(3\Gamma)$, $|\mu - \Tilde{\mu}| \leq \epsilon/(3\Gamma)$, and $|\ell-\Tilde{\ell}| \leq \epsilon/6$. Therefore, letting $\Tilde{f}(\mathbf{x}) = \phi_{\Tilde{\ell}} \left( \Tilde{\boldsymbol{\theta}}^{\top}\mathbf{x} + \Tilde{\mu} \right)$, we have, for all $\mathbf{x} \in \mathbb{B}^d$,
\begin{align*}
    \left| f(\mathbf{x}) - \Tilde{f}(\mathbf{x}) \right| & = \left| \phi_{\ell} \left( \boldsymbol{\theta}^{\top}\mathbf{x} + \mu \right) - \phi_{\Tilde{\ell}} \left( \Tilde{\boldsymbol{\theta}}^{\top}\mathbf{x} + \Tilde{\mu} \right) \right| \\        
    & = \left| \left( \boldsymbol{\theta}^{\top}\mathbf{x} + \mu \right)_{+} - \ell \left( \boldsymbol{\theta}^{\top}\mathbf{x} + \mu \right)_{-} - \left\{ \left( \Tilde{\boldsymbol{\theta}}^{\top}\mathbf{x} + \Tilde{\mu} \right)_{+} - \Tilde{\ell} \left( \Tilde{\boldsymbol{\theta}}^{\top}\mathbf{x} + \Tilde{\mu} \right)_{-} \right\} \right|   \\
    & \leq \big| \left( \boldsymbol{\theta}^{\top}\mathbf{x} + \mu \right)_{+} - \left( \Tilde{\boldsymbol{\theta}}^{\top}\mathbf{x} + \Tilde{\mu} \right)_{+} \big| + \left| \ell \left( \boldsymbol{\theta}^{\top}\mathbf{x} + \mu \right)_{-}  - \Tilde{\ell} \left( \Tilde{\boldsymbol{\theta}}^{\top}\mathbf{x} + \Tilde{\mu} \right)_{-} \right| \\
    & \leq \left|  \boldsymbol{\theta}^{\top}\mathbf{x} + \mu  - \left( \Tilde{\boldsymbol{\theta}}^{\top}\mathbf{x} + \Tilde{\mu} \right) \right| + \left| \ell \left( \boldsymbol{\theta}^{\top}\mathbf{x} + \mu \right)_{-} - \ell \left( \Tilde{\boldsymbol{\theta}}^{\top}\mathbf{x} + \Tilde{\mu} \right)_{-} \right| \\
    & \qquad + \left| \ell \left( \Tilde{\boldsymbol{\theta}}^{\top}\mathbf{x} + \Tilde{\mu} \right)_{-}  - \Tilde{\ell} \left( \Tilde{\boldsymbol{\theta}}^{\top}\mathbf{x} + \Tilde{\mu} \right)_{-} \right|  \\
    & \leq \| \mathbf{x} \|_2 \cdot \big\| \boldsymbol{\theta} - \Tilde{\boldsymbol{\theta}} \big\|_2 + \big| \mu - \Tilde{\mu} \big| + \big| \ell \big| \cdot \left|  \boldsymbol{\theta}^{\top}\mathbf{x} + \mu  - \left( \Tilde{\boldsymbol{\theta}}^{\top}\mathbf{x} + \Tilde{\mu} \right) \right| \\
    & \qquad + \big| \ell - \Tilde{\ell} \big| \cdot \left| \Tilde{\boldsymbol{\theta}}^{\top}\mathbf{x} + \Tilde{\mu} \right| \\
    & \leq \| \mathbf{x} \|_2 \cdot \big\| \boldsymbol{\theta} - \Tilde{\boldsymbol{\theta}} \big\|_2 + \big| \mu - \Tilde{\mu} \big| + \big| \ell \big| \cdot \| \mathbf{x} \|_2 \cdot \big\| \boldsymbol{\theta} - \Tilde{\boldsymbol{\theta}} \big\|_2 + \big| \ell \big| \cdot \big| \mu - \Tilde{\mu} \big| \\
    & \qquad  + \big| \ell - \Tilde{\ell} \big| \cdot \left\{ \left| \Tilde{\boldsymbol{\theta}}^{\top}\mathbf{x} \right| + \big| \Tilde{\mu} \big| \right\} \\
    & \leq \Gamma \big(\big\| \boldsymbol{\theta} - \Tilde{\boldsymbol{\theta}} \big\|_2 + \big| \mu - \Tilde{\mu} \big| \big) + 2 \big| \ell - \Tilde{\ell} \big| \\
    & \leq \epsilon,
\end{align*}   
 where we apply the triangle inequality and the Cauchy–Schwarz inequality at each step. We complete the proof.
\end{proof}

The following lemmas will be used to derive the result of Theorem \ref{thm_conv_prelu}.

\begin{lemma}[McDiarmid's inequality \citep{McDiarmid_1989}] 
\label{lem_McD}    
    Let $f: \mathcal{X}_1 \times \mathcal{X}_2 \times \cdots \times \mathcal{X}_n \rightarrow \mathbb{R}$ satisfies the bounded differences property with bounds $\nu_1, \nu_2, \ldots, \nu_n$. Consider independent random vectors $\mathbf{X}_1, \ldots, \mathbf{X}_{n}$ where $\mathbf{X}_i \in \mathcal{X}_i$ for all $i$. Then, for any $\epsilon>0$,
    \begin{align*}
    \mathbb{P}\left\{ \left|f\left(\mathbf{X}_1, \ldots, \mathbf{X}_{n}\right)-\mathbb{E}\left[f\left(\mathbf{X}_1, \ldots, \mathbf{X}_{n}\right)\right]\right| \geq \epsilon \right\} \leq 2 \exp \left(-\frac{2 \epsilon^2}{\sum_{i=1}^n \nu_i^2}\right) . \qquad \qquad 
    \end{align*}
\end{lemma}

\begin{lemma}[Dudley's Entropy Integral Bound \citep{dudley1969}] \label{Dudley}
     Let $\mathscr{F}$ be a class of real-valued functions, $S = \{z_1, \ldots, z_m\}$ be random i.i.d. samples, and $\mathcal{N}(\epsilon,\mathscr{F},  \|\cdot\|_{2, S})$ be the size of minimal $\epsilon$-cover of $\mathscr{F}$ with respect to $\|\cdot\|_{2,S}$. Assuming
    $$
    \sup_{f \in \mathscr{F}} \left( \frac{1}{m} \sum_{i=1}^{m} f^2(z_i) \right)^{\frac{1}{2}} \leq c,
    $$
    then we have 
    $$
    \hat{R}_m(\mathscr{F}, S) \leq \inf_{\epsilon \in [0, c/2]} \left( 4\epsilon + \frac{12}{\sqrt{m}} \int_{\epsilon}^{c/2} \sqrt{\log \mathcal{N}(\nu,\mathscr{F},  \|\cdot\|_{2,S})} \, d\nu \right).
    $$
\end{lemma}


\begin{proof}[Proof of Theorem \ref{thm_conv_prelu}] 

For every $\ell \in [\ell_{\min},\ell_{\max}]$, $\phi_{\ell}$  is $\Lambda$-Lipschtz continuous. 
Using McDiarmid's inequality (Lemma \ref{lem_McD}), we prove Theorem \ref{thm_conv_prelu} by showing that
\begin{align} 
    \mathbb{P} \Bigg( \bigg| D_{\mathscr{F}_{\textsf{prelu}}}(\mathrm{P},\mathrm{Q}) -  D_{\mathscr{F}_{\textsf{prelu}}} &(\widehat{\mathrm{P}}_n, \widehat{\mathrm{Q}}_m) \bigg| - c_2 \left(\frac{1}{\sqrt{n}} + \frac{1}{\sqrt{m}} \right) > \epsilon \Bigg) \leq 2 \exp \left(-\frac{\epsilon^2nm}{8 \Lambda^2(n+m)} \right), \label{thm_conv_prelu-0}   
\end{align}
where a constant $c_2 = \mathcal{O}(\sqrt{d})$ only depending on $d$.

Let us begin with the absolute term on the left-hand side of (\ref{thm_conv_prelu-0}). 
\begin{align}
    \bigg| D_{\mathscr{F}_{\textsf{prelu}}} (\mathrm{P},\mathrm{Q}) & -  D_{\mathscr{F}_{\textsf{prelu}}}(\widehat{\mathrm{P}}_n, \widehat{\mathrm{Q}}_m) \bigg| \nonumber \\
    & = \left| \sup_{f \in \mathscr{F}_{\textsf{prelu}}} \left| \mathbb{E}_{\mathbf{X} \sim \mathrm{P}} f - \mathbb{E}_{\mathbf{Y} \sim \mathrm{Q}} f \right| - \sup_{f \in \mathscr{F}_{\textsf{prelu}}} \left| \frac{1}{n} \sum_{i=1}^n f(\mathbf{X}_i) + \frac{1}{m} \sum_{i=1}^m f(\mathbf{Y}_i) \right| \right| \nonumber \\
    & \leq \sup_{f \in \mathscr{F}_{\textsf{prelu}}} \left| \mathbb{E}_{\mathbf{X} \sim \mathrm{P}} f - \mathbb{E}_{\mathbf{Y} \sim \mathrm{Q}} f - \frac{1}{n} \sum_{i=1}^n f(\mathbf{X}_i) + \frac{1}{m} \sum_{i=1}^m f(\mathbf{Y}_i) \right| \nonumber \\
    & \eqqcolon S(\mathrm{P},\mathrm{Q},\mathbf{X}^n,\mathbf{Y}^m).
\end{align}
Also, for every $f \in \mathscr{F}_{\textsf{prelu}}$, 
\begin{align*}
    \left\| f \right\|_\infty 
    & = \sup_{\mathbf{x} \in \mathbb{B}^d} \left| \phi_{\ell} \left( \boldsymbol{\theta}^{\top}\mathbf{x} + \mu \right) \right| \\
    & \leq \Lambda \cdot \sup_{\mathbf{x} \in \mathbb{B}^d} \left|  \boldsymbol{\theta}^{\top}\mathbf{x} + \mu \right| \\
    & \leq 2 \Lambda.
\end{align*}
Let $\mathbf{X}_{(i)}^{n} \coloneqq \left\{ \mathbf{X}_1, \ldots, \mathbf{X}_{i-1}, \mathbf{X}_i^{\prime}, \mathbf{X}_{i+1}, \ldots, \mathbf{X}_{n} \right\}$ denote the dataset obtained from $\mathbf{X}^n$ by replacing only its $i$-th component, for $i = 1,\ldots,n$. Similarly, define $\mathbf{Y}_{(j)}^{m} \coloneqq \left\{ \mathbf{Y}_1, \ldots, \mathbf{Y}_{j-1}, \mathbf{Y}_j^{\prime}, \mathbf{Y}_{j+1}, \ldots, \mathbf{Y}_{m} \right\}$ as the dataset differing from $\mathbf{Y}^m$ only in its $j$-th component, for $j = 1,\ldots,m$. 
Then we have
\begin{align*}
    & S(\mathrm{P},\mathrm{Q},\mathbf{X}^n, \mathbf{Y}^m)  - S(\mathrm{P},\mathrm{Q}, \mathbf{X}_{(i)}^{n}, \mathbf{Y}^m) \\
    & = \sup_{f \in \mathscr{F}_{\textsf{prelu}}} \left| \mathbb{E}_{\mathbf{X} \sim \mathrm{P}} f - \mathbb{E}_{\mathbf{Y} \sim \mathrm{Q}} f - \frac{1}{n} \sum_{i=1}^n f(\mathbf{X}_i) + \frac{1}{m} \sum_{i=1}^m f(\mathbf{Y}_i) \right| \\
    & \quad - \sup_{f \in \mathscr{F}_{\textsf{prelu}}} \left| \mathbb{E}_{\mathbf{X} \sim \mathrm{P}} f - \mathbb{E}_{\mathbf{Y} \sim \mathrm{Q}} f - \frac{1}{n} \left( \sum_{k \neq i} f(\mathbf{X}_k) + f(\mathbf{X}_i^{\prime}) \right) + \frac{1}{m} \sum_{i=1}^m f(\mathbf{Y}_i) \right| \\
    & \leq \sup_{f \in \mathscr{F}_{\textsf{prelu}}} \left| - \frac{1}{n} \sum_{i=1}^n f(\mathbf{X}_i) + \frac{1}{n} \left( \sum_{k \neq i} f(\mathbf{X}_k) + f(\mathbf{X}_i^{\prime}) \right) \right| \\
    & = \sup_{f \in \mathscr{F}_{\textsf{prelu}}} \left| \frac{1}{n} f(\mathbf{X}_i^{\prime}) - \frac{1}{n} f(\mathbf{X}_i) \right| \\
    & \leq \frac{2}{n} \sup_{f \in \mathscr{F}_{\textsf{prelu}}} \|f\|_{\infty} \\
    & \leq \frac{4 \Lambda}{n}.
\end{align*}
Similarly, $S(\mathrm{P},\mathrm{Q},\mathbf{X}^n,\mathbf{Y}^m) - S(\mathrm{P},\mathrm{Q},\mathbf{X}^n,\mathbf{Y}_{(j)}^{\prime}) \leq 4 \Lambda/m$.
By applying the McDiarmid's inequality (Lemma \ref{lem_McD}) with $\nu_i = \frac{4 \Lambda}{n}$ ($i=1,\ldots,n$) for $\mathbf{X}$ and $\nu_j = \frac{4 \Lambda}{m}$ ($j=1,\ldots,m$) for $\mathbf{Y}$, we then obtain    
\begin{align}
    \mathbb{P} ( S(\mathrm{P},\mathrm{Q},\mathbf{X}^n,\mathbf{Y}^m) - \mathbb{E} S(\mathrm{P},\mathrm{Q},\mathbf{X}^n, \mathbf{Y}^m) > \epsilon ) \leq 2 \exp \left( - \frac{\epsilon^2 nm}{8 \Lambda^2 (n+m)} \right). \label{thm_conv_prelu_2}
\end{align}

To complete the proof, it suffices to derive an upper bound for $\mathbb{E} S(\mathrm{P},\mathrm{Q},\mathbf{X}^n,\mathbf{Y}^m)$,
$$
\mathbb{E} S(\mathrm{P},\mathrm{Q},\mathbf{X}^n,\mathbf{Y}^m) \leq c_2 \left( \frac{1}{\sqrt{n}} + \frac{1}{\sqrt{m}} \right).
$$

By the Dudley's entropy integral bound (Lemma \ref{Dudley}) with $\sup_{f \in \mathscr{F}_{\textsf{prelu}}} \|f\|_{\infty} \leq 2 \Lambda$ and Lemma \ref{lem_cover_number_prelu}, we have
\begin{align*}
    \widehat{\mathcal{R}}_n(\mathscr{F}_{\textsf{prelu}}) & \leq \inf_{t \in (0, \Lambda)} \left( 4t + \frac{12}{\sqrt{n}} \int_t^{\Lambda} \sqrt{\log \mathcal{N} (\epsilon, \mathscr{F}_{\textsf{prelu}}, \|\cdot\|_{\infty} ) } d\epsilon \right) \\
    & \leq \inf_{t \in (0, \Lambda)} \left( 4t + \frac{12}{\sqrt{n}} \int_t^{\Lambda} \sqrt{ (d+2) \log (1+1/\epsilon) } d\epsilon \right) \\
    & \leq c'_2 \inf_{t \in (0, \Lambda)} \left( t + \frac{\sqrt{d}}{\sqrt{n}} \int_t^{\Lambda} \sqrt{\log ((\Lambda+1)/\epsilon) } d\epsilon \right) \\
    & \leq c'_2 \inf_{t \in (0, \Lambda)} \left( t + \frac{\sqrt{d}}{\sqrt{n}} \int_t^{\Lambda} \big(\log (\Lambda+1) - \log(\epsilon)\big) d\epsilon \right) \\
    & = \frac{c''_2}{\sqrt{n}},
\end{align*}
where the infimum in the last step is by taking $t=n^{-1/2}$, $c'_2$ is a constant, and $c''_2$ is a constant with $\mathcal{O}(\sqrt{d})$. 

Finally, we will derive an upper bound of $\mathbb{E}_{\mathbf{X},\mathbf{Y}} S(\mathrm{P},\mathrm{Q},\mathbf{X}^n,\mathbf{Y}^m)$ in terms of the empirical Radamacher averages by use of the symmetrization technique \citep{van1996}.
Let $\mathbf{X}^{\prime,n} \coloneqq (\mathbf{X}_1^{\prime}, \ldots, \mathbf{X}_n^{\prime})$ be an independent sample whose distribution is the same as
that of $\mathbf{X}^n,$ and define $\mathbf{Y}^{\prime,m}$ similarly.   
Then we have
\begin{align}
    & \mathbb{E}_{\mathbf{X}, \mathbf{Y}} S(\mathrm{P},\mathrm{Q},\mathbf{X}^n,\mathbf{Y}^m) \nonumber \\
    & = \mathbb{E}_{\mathbf{X}^n, \mathbf{Y}^m} \sup _{f \in \mathscr{F}_{\textsf{prelu}}} \left| \mathbb{E}_{\mathbf{X} \sim \mathrm{P}} f(\mathbf{X}) - \mathbb{E}_{\mathbf{Y} \sim \mathrm{Q}} f(\mathbf{Y}) - \frac{1}{n} \sum_{i=1}^n f(\mathbf{X}_i) + \frac{1}{m} \sum_{j=1}^m f(\mathbf{Y}_j) \right| \nonumber \\ 
    & = \mathbb{E}_{\mathbf{X}^n, \mathbf{Y}^m} \sup _{f \in \mathscr{F}_{\textsf{prelu}}} \Bigg| \mathbb{E}_{\mathbf{X}^{\prime, n} \sim \mathrm{P}} \left( \frac{1}{n} \sum_{i=1}^n f(\mathbf{X}_i^{\prime}) \right) - \frac{1}{n} \sum_{i=1}^n f(\mathbf{X}_i) \nonumber \\
    & \hspace{3cm} - \mathbb{E}_{\mathbf{Y}^{\prime, m} \sim \mathrm{Q}} \left( \frac{1}{m} \sum_{j=1}^m f(\mathbf{Y}_j^{\prime}) \right) + \frac{1}{m} \sum_{j=1}^m f(\mathbf{Y}_j) \Bigg| \nonumber \\
    & \leq \mathbb{E}_{\mathbf{X}^n, \mathbf{Y}^m, \mathbf{X}^{\prime, n}, \mathbf{Y}^{\prime, m}} \sup _{f \in \mathscr{F}_{\textsf{prelu}}} \left| \frac{1}{n} \sum_{i=1}^n f(\mathbf{X}_i^{\prime}) - \frac{1}{n} \sum_{i=1}^n f(\mathbf{X}_i) - \frac{1}{m} \sum_{j=1}^m f(\mathbf{Y}_j^{\prime}) + \frac{1}{m} \sum_{j=1}^m f(\mathbf{Y}_j) \right| &&\tag{by Jensen's ineq.} \nonumber \\
    & = \mathbb{E}_{\mathbf{X}^n, \mathbf{Y}^m, \mathbf{X}^{\prime, n}, \mathbf{Y}^{\prime, m}, \varepsilon, \varepsilon^{\prime}} \sup _{f \in \mathscr{F}_{\textsf{prelu}}}\left| \frac{1}{n} \sum_{i=1}^n \varepsilon_i\left(f(\mathbf{X}_i^{\prime}) - f(\mathbf{X}_i)\right) + \frac{1}{m} \sum_{j=1}^m \varepsilon_j^{\prime}\left(f(\mathbf{Y}_j^{\prime}) - f(\mathbf{Y}_j)\right) \right| \nonumber \\        
    & \leq \mathbb{E}_{\mathbf{X}^n, \mathbf{X}^{\prime, n}, \varepsilon} \sup _{f \in \mathscr{F}_{\textsf{prelu}}} \left| \frac{1}{n} \sum_{i=1}^n \varepsilon_i \left( f(\mathbf{X}_i^{\prime}) - f(\mathbf{X}_i)\right) \right| \nonumber \\
    &  \qquad \qquad + \mathbb{E}_{\mathbf{Y}^m, \mathbf{Y}^{\prime, m}, \varepsilon} \sup _{f \in \mathscr{F}_{\textsf{prelu}}} \left| \frac{1}{m} \sum_{j=1}^m \varepsilon_j\left(f(\mathbf{Y}_j^{\prime}) - f(\mathbf{Y}_j) \right) \right|  &&\tag{by Triangle ineq.} \nonumber \\
    & = \mathbb{E}_{\mathbf{X}^n, \mathbf{X}^{\prime, n}} \mathbb{E}_{\varepsilon} \left[ \sup _{f \in \mathscr{F}_{\textsf{prelu}}} \left| \frac{1}{n} \sum_{i=1}^n \varepsilon_i \left( f(\mathbf{X}_i^{\prime}) - f(\mathbf{X}_i)\right) \right| \Big| \mathbf{X}_1, \ldots, \mathbf{X}_{n} \right] \nonumber \\
    & \qquad \qquad + \mathbb{E}_{\mathbf{Y}^m, \mathbf{Y}^{\prime, m}} \mathbb{E}_{\varepsilon} \left[\sup _{f \in \mathscr{F}_{\textsf{prelu}}} \left| \frac{1}{m} \sum_{j=1}^m \varepsilon_j\left(f(\mathbf{Y}_j^{\prime}) - f(\mathbf{Y}_j) \right) \right| \Big| \mathbf{Y}_1, \ldots, \mathbf{Y}_{m} \right] \nonumber \\
    & \leq 2\left\{ \widehat{\mathcal{R}}_n(\mathscr{F}_{\textsf{prelu}}) + \widehat{\mathcal{R}}_m(\mathscr{F}_{\textsf{prelu}}) \right\} \hspace{8cm} \nonumber \\
    & \leq c_2 \left( \frac{1}{\sqrt{n}} + \frac{1}{\sqrt{m}} \right), \label{thm_conv_prelu_3}
\end{align}
where $c_2$ is $\mathcal{O}(\sqrt{d})$.
By combining (\ref{thm_conv_prelu_2}) and (\ref{thm_conv_prelu_3}), we complete the proof of Theorem \ref{thm_conv_prelu}.

\end{proof}


\subsection{Proof of Proposition \ref{thm_max_equal}}
\label{pf_thm_max_equal}

In this section, we provide a detailed proof of Proposition \ref{thm_max_equal}. 

\begin{proof}
By the definition of $D_{\mathscr{F}_{\textsf{prelu}}} (\mathrm{P}, \mathrm{Q})$, for any $\ell_{\min}, \ell_{\max} \in \mathbb{R}$,
\begin{align}
    & D_{\mathscr{F}_{\textsf{prelu}}} (\mathrm{P}, \mathrm{Q}) \nonumber \\
    & = \sup_{\ell_{\min} \leq \ell \leq \ell_{\max}} \sup_{\boldsymbol{\theta}, \, \mu}  \left| \mathbb{E} \phi_{\ell} \left( \boldsymbol{\theta}^{\top}\mathbf{X} + \mu \right) - \mathbb{E} \phi_{\ell} \left( \boldsymbol{\theta}^{\top}\mathbf{Y} + \mu\right) \right| \nonumber \\
    & = \sup_{\boldsymbol{\theta}, \, \mu} \sup_{\ell_{\min} \leq \ell \leq \ell_{\max}} \left| \mathbb{E} \left[ (\boldsymbol{\theta}^{\top}\mathbf{X} + \mu)_{+} + \ell (\boldsymbol{\theta}^{\top}\mathbf{X} + \mu)_{-} \right] - \mathbb{E} \left[ (\boldsymbol{\theta}^{\top}\mathbf{Y} + \mu)_{+} + \ell (\boldsymbol{\theta}^{\top}\mathbf{Y} + \mu)_{-} \right] \right| \nonumber \\
    & = \sup_{\boldsymbol{\theta}, \, \mu} \sup_{\ell_{\min} \leq \ell \leq \ell_{\max}} \left| \mathbb{E} \left[ (\boldsymbol{\theta}^{\top}\mathbf{X} + \mu)_{+} -  (\boldsymbol{\theta}^{\top}\mathbf{Y} + \mu)_{+} + \ell \left\{ (\boldsymbol{\theta}^{\top}\mathbf{X} + \mu)_{-} -  (\boldsymbol{\theta}^{\top}\mathbf{Y} + \mu)_{-} \right\} \right] \right| \nonumber \\
    & = \sup_{\boldsymbol{\theta}, \, \mu} \max_{\ell \in \{\ell_{\min},\ell_{\max}\}} \left| \mathbb{E} \left[ (\boldsymbol{\theta}^{\top}\mathbf{X} + \mu)_{+} -  (\boldsymbol{\theta}^{\top}\mathbf{Y} + \mu)_{+} + \ell \left\{ (\boldsymbol{\theta}^{\top}\mathbf{X} + \mu)_{-} -  (\boldsymbol{\theta}^{\top}\mathbf{Y} + \mu)_{-} \right\} \right] \right| \label{tmp_10}\\
    &= \max_{\ell \in \{\ell_{\min} ,\ell_{\max}\}} \sup_{\boldsymbol{\theta}, \, \mu}  \left| \mathbb{E} \phi_{\ell} \left( \boldsymbol{\theta}^{\top}\mathbf{X} + \mu \right) - \mathbb{E} \phi_{\ell} \left( \boldsymbol{\theta}^{\top}\mathbf{Y} + \mu\right) \right| \nonumber \\
    & = \max\big( D_{\mathscr{F}_{\ell_{\min}}} (\mathrm{P}, \mathrm{Q}),  D_{\mathscr{F}_{\ell_{\max}}} (\mathrm{P}, \mathrm{Q})\big), \nonumber
\end{align}
where the equality \eqref{tmp_10} follows from the fact that the supremum of a linear function over a compact interval is attained at the boundary.
\end{proof}


\subsection{Proof of Theorem \ref{thm_prelu_tst_consis}} \label{pf_thm_prelu_tst_consis}

\begin{proof} 
From Lemma \ref{lem_cover_number_prelu} we can obtain 
$$\int_{0}^1 \sqrt{\log  \mathcal{N} ( \epsilon, \mathscr{F}_{\textsf{prelu}}, \|\cdot\|_{\infty} )} d \epsilon < \infty.$$
Hence, $\mathscr{F}_{\textsf{prelu}}$ is the P-Donsker class by Theorem 11.6 of \cite{sen2018gentle}. We obtain the first assertion by Theorem 9.1 of \cite{dudley2014uniform}.

For the second argument, we take $\epsilon = 4 \Lambda \sqrt{\log(2/\alpha)/n}$ of Theorem \ref{thm_conv_prelu}.
Then, 
    $$
    \mathbb{P}_{\mathbb{H}_0}\bigg(\Big| D_{\mathscr{F}_{\textsf{prelu}}}(\widehat{\mathrm{P}}_n, \widehat{\mathrm{Q}}_n) \Big| > c \left( \frac{1}{\sqrt{n}} + \frac{1}{\sqrt{n}} \right) + \epsilon\bigg) \leq 2 \exp \left(-\frac{\epsilon^2 n^2}{8 \Lambda^2 (2n)} \right) = \alpha.
    $$
Since $\eta_{\alpha, n} = \big(2c + 4 \Lambda \sqrt{\log(2/\alpha)}\big)/\sqrt{n}$,
we obtain the assertion.
\end{proof}


\subsection{Proof of Theorem \ref{theo_fix}}
\label{pf_theo_fix}

\begin{proof}
From $\mathrm{P} \neq \mathrm{Q}$ and Proposition \ref{prop_discriminative}, we obtain $D_{\mathscr{F}_{\textsf{prelu}}} (\mathrm{P}, \mathrm{Q})>0$.
Since $\eta_{\alpha, n} = o(1)$, there exists $N \in \mathbb{N}$ such that $\eta_{\alpha, n} \leq D_{\mathscr{F}_{\textsf{prelu}}} (\mathrm{P}, \mathrm{Q})/2$ for every $n \geq N$. 
Hence, for sufficiently large $n$ such that $n \geq N$, the Type II risk can be upper bounded by
\begin{align*}
    \mathbb{P}_{\mathrm{P} \neq \mathrm{Q}} \Big( D_{\mathscr{F}_{\textsf{prelu}}}  (\widehat{\mathrm{P}}_n, \widehat{\mathrm{Q}}_n) < \eta_{\alpha, n} \Big)   & = \mathbb{P}_{\mathrm{P} \neq \mathrm{Q}} \Big( D_{\mathscr{F}_{\textsf{prelu}}} (\mathrm{P}, \mathrm{Q}) - D_{\mathscr{F}_{\textsf{prelu}}}  (\widehat{\mathrm{P}}_n, \widehat{\mathrm{Q}}_n) > D_{\mathscr{F}_{\textsf{prelu}}}  (\mathrm{P}, \mathrm{Q})  - \eta_{\alpha, n} \Big) \\
    & \leq \mathbb{P}_{\mathrm{P} \neq \mathrm{Q}} \Big( \left| D_{\mathscr{F}_{\textsf{prelu}}} (\mathrm{P}, \mathrm{Q}) - D_{\mathscr{F}_{\textsf{prelu}}} (\widehat{\mathrm{P}}_n, \widehat{\mathrm{Q}}_n) \right| > D_{\mathscr{F}_{\textsf{prelu}}} (\mathrm{P}, \mathrm{Q}) - \eta_{\alpha, n} \Big) \\
    & \leq \frac{ \mathbb{E} \bigg[ \left| D_{\mathscr{F}_{\textsf{prelu}}} (\mathrm{P}, \mathrm{Q}) - D_{\mathscr{F}_{\textsf{prelu}}} (\widehat{\mathrm{P}}_n, \widehat{\mathrm{Q}}_n) \right| \bigg| \mathbb{H}_1: \mathrm{P} \neq \mathrm{Q} \bigg] }{ D_{\mathscr{F}_{\textsf{prelu}}} (\mathrm{P}, \mathrm{Q}) - \eta_{\alpha, n} } && \tag{by Markov's ineq.} \\
    & \leq \frac{ \mathbb{E} \bigg[ \left| D_{\mathscr{F}_{\textsf{prelu}}} (\mathrm{P}, \mathrm{Q}) - D_{\mathscr{F}_{\textsf{prelu}}} (\widehat{\mathrm{P}}_n, \widehat{\mathrm{Q}}_n) \right| \bigg| \mathbb{H}_1: \mathrm{P} \neq \mathrm{Q} \bigg] }{ D_{\mathscr{F}_{\textsf{prelu}}} (\mathrm{P}, \mathrm{Q})/2}\\
    \end{align*}
    \begin{align*}
    & \leq \frac{ 4 c_2  }{D_{\mathscr{F}_{\textsf{prelu}}} (\mathrm{P}, \mathrm{Q}) \sqrt{n} }, && \tag{by \eqref{thm_conv_prelu_3}} \\
    & = \frac{c_3}{\sqrt{n}},
\end{align*}
where $c_2$ is the constant defined on Theorem~\ref{thm_conv_prelu} and 
$c_3 \coloneqq 4c_2/D_{\mathscr{F}_{\textsf{prelu}}} (\mathrm{P}, \mathrm{Q})$ is a constant.
Therefore, the power of PReLU-TST $T_{\textsf{prelu}, \alpha}(\mathbf{X}^n, \mathbf{Y}^m)$ is at least $1 - c_3 n^{-1/2}$ and goes to 1 as $n \rightarrow \infty$.
\end{proof}


\subsection{Proof of Theorem \ref{thm_minimax}}
\label{pf_thm_minimax}

\begin{proof}
By the proof of Theorem \ref{theo_fix}, we have 
\begin{align}
    \mathbb{P}_{\mathrm{P} \neq \mathrm{Q}}(T_{\textsf{prelu}, \alpha}(\mathbf{X}^n, \mathbf{Y}^m)) \geq 1- \frac{4c}{\sqrt{n} \cdot D_{\mathscr{F}_{\textsf{prelu}}}(\mathrm{P}, \mathrm{Q})}, 
    \label{tmp_insss}
\end{align}
for any fixed alternative $\mathrm{P} \neq \mathrm{Q}$,
where $c=\mathcal{O}(\sqrt{d})$ is the same constant as in Theorem \ref{thm_conv_prelu}. 
Now, we show the consistency of PReLU-TST under the local H\"{o}lder alternative hypothesis, 
$\mathbb{H}_1(\kappa_n): D_{\mathscr{H}^{\beta, d}}(\mathrm{P},\mathrm{Q}) \geq \kappa_n$.

\begin{enumerate}
    \item[(i)] $\beta > \frac{d+3}{2}$ and $\kappa_n \gtrsim n^{-1/2} \log n$ : by Corollary \ref{cor_holder}, we have
$$
D_{\mathscr{F}_{\textsf{prelu}}}(\mathrm{P},\mathrm{Q}) \gtrsim n^{-\frac{1}{2}} \log n.
$$
By \eqref{tmp_insss}, the power of $T_{\textsf{prelu}, \alpha}(\mathbf{X}^n, \mathbf{Y}^m)$ under $\mathbb{H}_1(\kappa_n)$ is given by
$$
\mathbb{P}_{\mathbb{H}_1(\kappa_n)}(T_{\textsf{prelu}, \alpha}(\mathbf{X}^n, \mathbf{Y}^m)) \gtrsim 1 - (\log n)^{-1}.
$$
    \item [(ii)] $\beta < \frac{d+3}{2}$ and $\kappa_n \gtrsim n^{-\beta/(d+3)} \log n$ : by Corollary \ref{cor_holder}, we have
$$
D_{\mathscr{F}_{\textsf{prelu}}}(\mathrm{P},\mathrm{Q}) \gtrsim n^{-\frac{1}{2}} \log n.
$$
By \eqref{tmp_insss}, the power of $T_{\textsf{prelu}, \alpha}(\mathbf{X}^n, \mathbf{Y}^m)$ under $\mathbb{H}_1(\kappa_n)$ is given by
$$
\mathbb{P}_{\mathbb{H}_1(\kappa_n)}(T_{\textsf{prelu}, \alpha}(\mathbf{X}^n, \mathbf{Y}^m)) \gtrsim 1 - (\log n)^{-1}.
$$
\end{enumerate}

Hence, we conclude that the power of $T_{\textsf{prelu}, \alpha}(\mathbf{X}^n, \mathbf{Y}^m)$ goes to 1 as $n \rightarrow \infty$, by considering both cases (i) and (ii).   
\end{proof}

\newpage

\section{Additional experimental results}
\label{app_add_exp}

\subsection{Small-sample regime}
\label{app_add_exp_small_sample}

We evaluate the performance of the proposed test under small sample sizes, $n+m \in \{10,20,30,40,50\}$.
The corresponding results are presented in Figs.~\ref{fig_small_result1} and \ref{fig_small_result2}. 
The results of Besov-TST are omitted due to their consistent underperformance.
Overall, PReLU-TST maintains competitive power, even in challenging small-sample regimes.

\begin{figure}[H]
    \centering    
    \begin{subfigure}[t]{\textwidth}
        \centering
        \includegraphics[width=\textwidth]{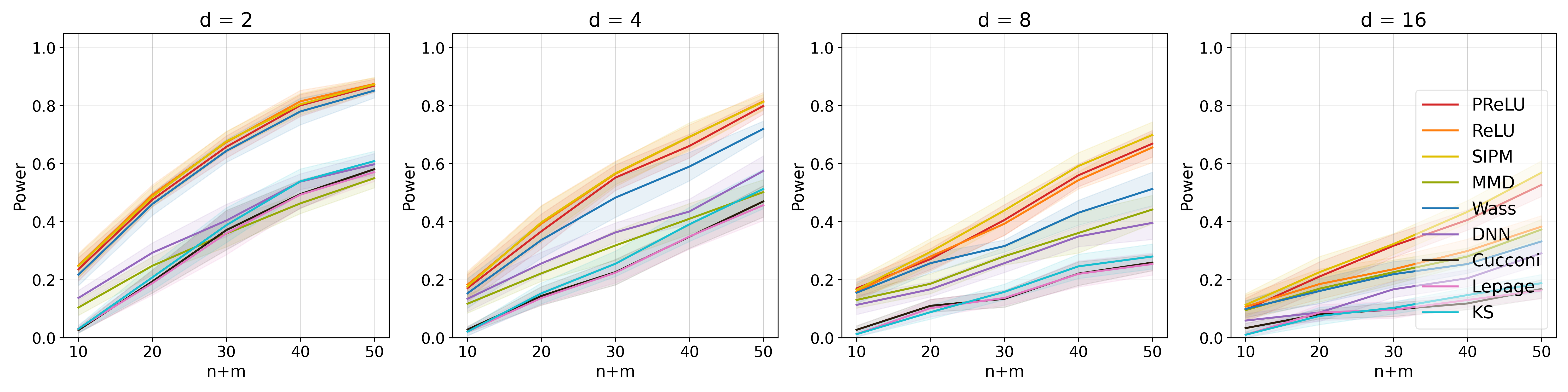}
        \caption{\texttt{L1} (a simple mean shift): $\mathrm{P} \sim \mathcal{N}(\mathbf{0}_d, \mathrm{I}_d), \, \mathrm{Q} \sim \mathcal{N}\big((\zeta,\mathbf{0}_{d-1}^\top)^\top, \mathrm{I}_d\big), \, \zeta=1.0$}
        \label{fig:ball_small}
    \end{subfigure}
    \begin{subfigure}[t]{\textwidth}
        \centering
        \includegraphics[width=\textwidth]{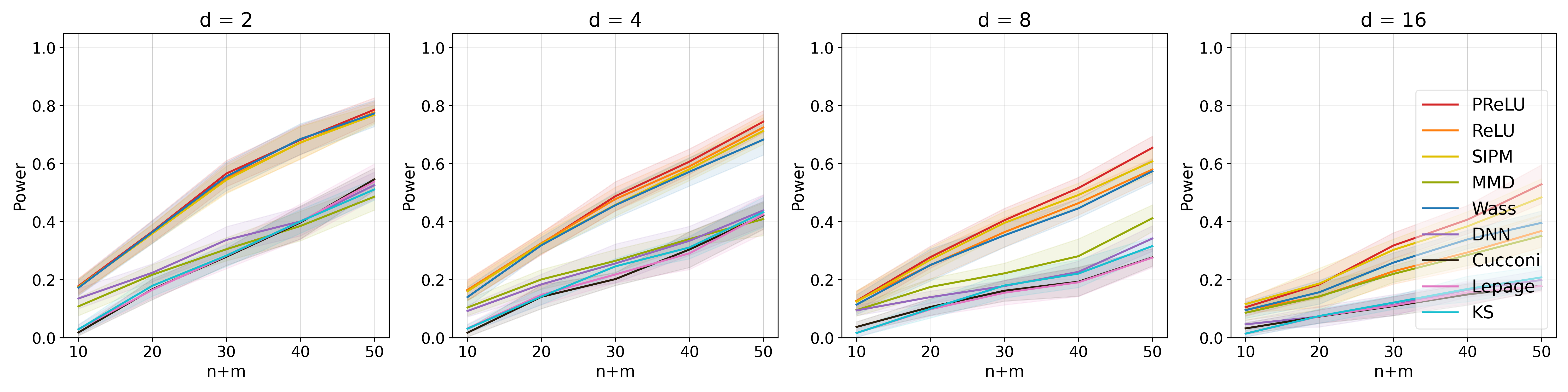}
        \caption{\texttt{L2} (a mixture-induced shift): $\mathrm{P} \sim \mathcal{N}(\mathbf{0}_d, \mathrm{I}_d), \, \mathrm{Q} \sim \frac{1}{2} \mathcal{N}(\mathbf{0}_d, \mathrm{I}_d) + \frac{1}{2}\mathcal{N}\big((\zeta,\mathbf{0}_{d-1}^\top)^\top, \mathrm{I}_d\big), \, \zeta = 2.0$}
        \label{fig:mixture_small}
    \end{subfigure}
    \begin{subfigure}[t]{\textwidth}
        \centering
        \includegraphics[width=\textwidth]{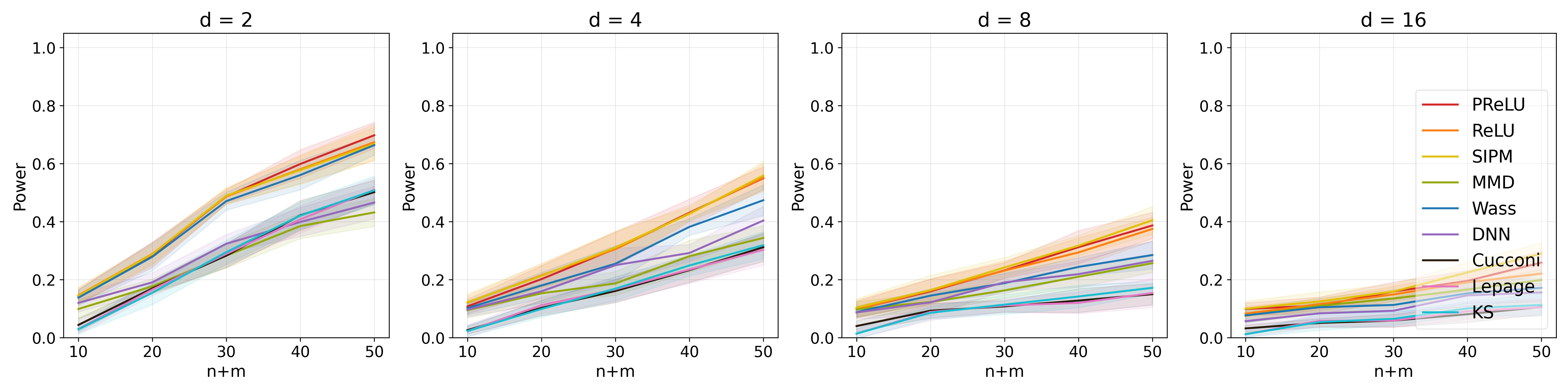}
        \caption{\texttt{L3} (a skewness-driven shift): $\mathrm{P} \sim \mathcal{N}(\mathbf{0}_d, \mathrm{I}_d), \, \mathrm{Q} \sim \mathsf{SkewNormal}(0,1,\alpha) \otimes \mathcal{N}(\mathbf{0}_{d-1}, \mathrm{I}_{d-1}), \, \alpha = 2.0$}
        \label{fig:skewed_small}
    \end{subfigure}
    \caption{Power vs.\ $n+m$ under small sample size for various $d$.}
    \label{fig_small_result1}
\end{figure}

\begin{figure}[H]   
    \centering
    \begin{subfigure}[b]{\textwidth}
        \centering
        \includegraphics[width=\textwidth]{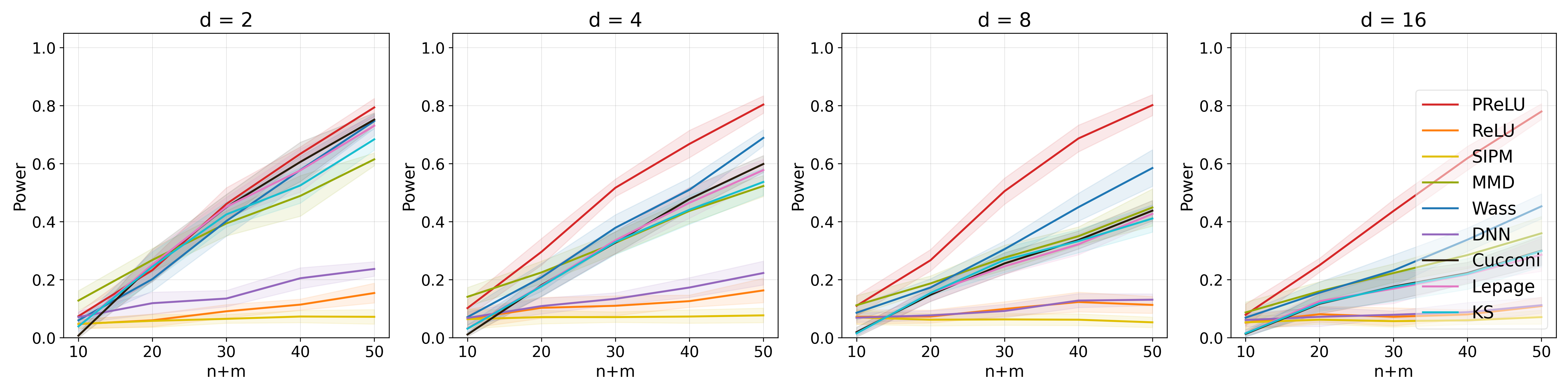}
        \caption{\texttt{S1} (a variance change along a single direction): $\mathrm{P} \!\sim\! \mathcal{N}(\mathbf{0}_d, \mathrm{I}_d), \mathrm{Q} \!\sim\! \mathcal{N}\big(\mathbf{0}_d, \operatorname{diag}(\sigma^2,\mathbf{1}_{d-1}^{\top})\big), \sigma^2\!=\!6$}
        \label{fig:varone_small}
    \end{subfigure}
    \begin{subfigure}[b]{\textwidth}
        \centering
        \includegraphics[width=\textwidth]{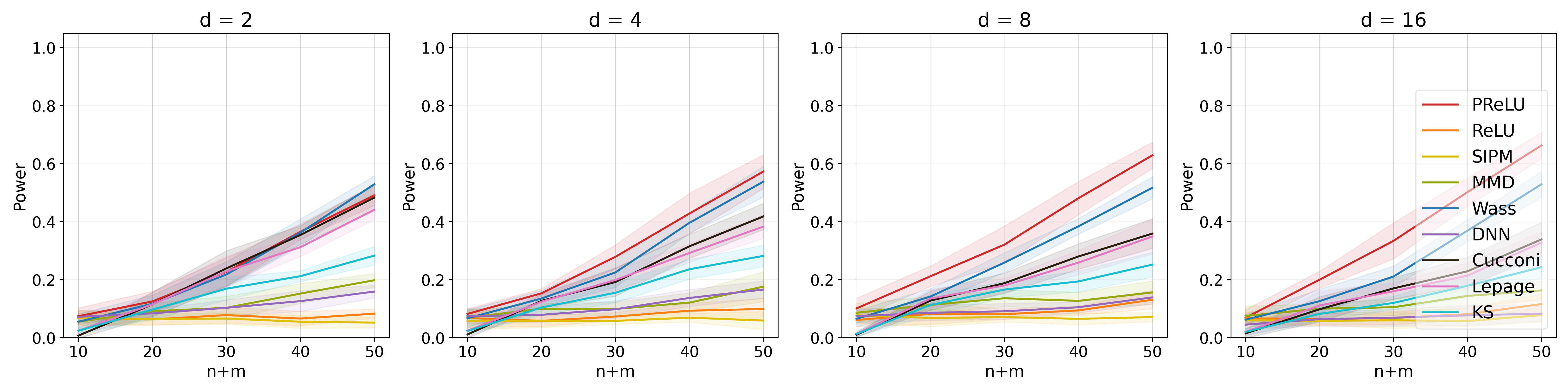}
        \caption{\texttt{S3} (contrasting tail behaviors): $\mathrm{P} \sim \mathcal{N}(\mathbf{0}_d, \mathrm{I}_d), \, \mathrm{Q} \sim t_v \otimes \mathcal{N}(\mathbf{0}_{d-1}, \mathrm{I}_{d-1}), \, v=1.0$}
        \label{fig:tcoord_small}
    \end{subfigure}
    \begin{subfigure}[b]{\textwidth}
        \centering
        \includegraphics[width=\textwidth]{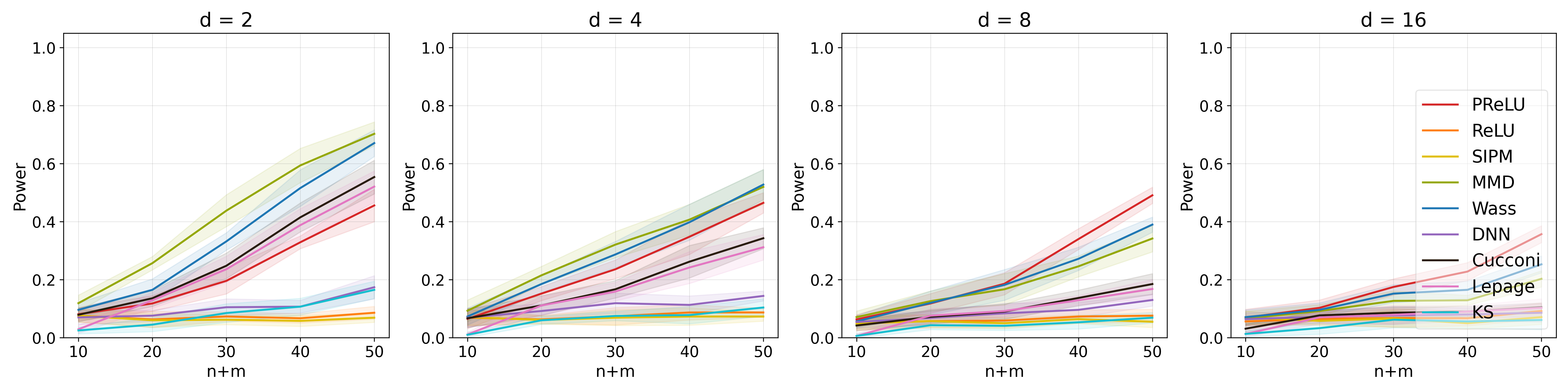}
        \caption{
        \texttt{XOR2}: 
        $ \quad
        \begin{aligned}
        \mathrm{P} &\sim \left( \frac{1}{2}\mathcal{N}\left([\zeta, \zeta]^\top, \mathrm{I}_2\right) + 
        \frac{1}{2}\mathcal{N}\left([-\zeta, -\zeta]^\top, \mathrm{I}_2\right) \right) 
        \otimes\, \mathcal{N}(\mathbf{0}_{d-2}, \mathrm{I}_{d-2}), \\
        \mathrm{Q} &\sim \left( \frac{1}{2}\mathcal{N}\left([\zeta, -\zeta]^\top, \mathrm{I}_2\right) + 
        \frac{1}{2}\mathcal{N}\left([-\zeta, \zeta]^\top, \mathrm{I}_2\right) \right) 
        \otimes\, \mathcal{N}(\mathbf{0}_{d-2}, \mathrm{I}_{d-2}), \, \zeta = 1.0
        \end{aligned}
        $
        }
        \label{fig:xor2_small}
    \end{subfigure}
    \caption{Power vs.\ $n+m$ under small sample size for various $d$.}
    \label{fig_small_result2}
\end{figure}

\newpage

\subsection{High-dimensional covariates}
\label{app_add_exp_high_dim}

We consider settings with high-dimensional covariates $d \in \{256,512\}$ with $n+m=200$. 
The results are summarized in Table~\ref{tab_high_dim}. PReLU-TST achieves strong power compared to competing methods.

\begin{table}[h]
\centering
\caption{
Testing powers under high-dimensional covariates with $n+m=200$.
The parameters are set as follows: $\zeta=1.0$ for \texttt{L1}, $\zeta=1.5$ for \texttt{L2}, $\alpha=1.4$ for \texttt{L3}, $\zeta=6.0$ for \texttt{S1}, $v=1.5$ for \texttt{S2}, and $\zeta=1.5$ for \texttt{XOR2}.
}
\label{tab_high_dim}
\renewcommand{\arraystretch}{1.3}
\resizebox{\textwidth}{!}{
\begin{tabular}{c|c||c|cccccccc}
\toprule
Dataset & $d$ & PReLU & ReLU & SIPM & MMD & Wass & Cucconi & Lepage & KS \\
\hline
\multirow{2}{*}{\texttt{L1}} 
 & 256 & \textbf{0.643$\pm$0.056} & 0.147$\pm$0.044 & \underline{0.618$\pm$0.061} & 0.616$\pm$0.054 & 0.161$\pm$0.020 & 0.082$\pm$0.274 & 0.079$\pm$0.270 & 0.089$\pm$0.285 \\
 & 512 & \textbf{0.464$\pm$0.053} & 0.117$\pm$0.029 & 0.439$\pm$0.046 & \underline{0.454$\pm$0.054} & 0.123$\pm$0.024 & 0.075$\pm$0.018 & 0.073$\pm$0.010 & 0.072$\pm$0.015 \\
\hline
\multirow{2}{*}{\texttt{L2}} 
 & 256 & \textbf{0.338$\pm$0.035} & 0.131$\pm$0.016 & 0.312$\pm$0.032 & \underline{0.314$\pm$0.030} & 0.113$\pm$0.037 & 0.077$\pm$0.267 & 0.070$\pm$0.255 & 0.076$\pm$0.265 \\
 & 512 & \textbf{0.255$\pm$0.045} & 0.108$\pm$0.032 & 0.227$\pm$0.056 & \underline{0.241$\pm$0.048} & 0.109$\pm$0.036 & 0.060$\pm$0.020 & 0.061$\pm$0.027 & 0.067$\pm$0.025 \\
\hline
\multirow{2}{*}{\texttt{L3}} 
 & 256 & \underline{0.269$\pm$0.029} & 0.126$\pm$0.022 & \textbf{0.280$\pm$0.029} & 0.260$\pm$0.037 & 0.086$\pm$0.014 & 0.066$\pm$0.248 & 0.075$\pm$0.263 & 0.071$\pm$0.257 \\
 & 512 & \textbf{0.176$\pm$0.033} & 0.090$\pm$0.018 & \underline{0.173$\pm$0.020} & 0.170$\pm$0.030 & 0.082$\pm$0.015 & 0.059$\pm$0.023 & 0.063$\pm$0.027 & 0.051$\pm$0.016 \\
\hline
\multirow{2}{*}{\texttt{S1}} 
 & 256 & \textbf{0.804$\pm$0.036} & 0.057$\pm$0.029 & 0.058$\pm$0.023 & 0.132$\pm$0.048 & \underline{0.207$\pm$0.042} & 0.180$\pm$0.384 & 0.175$\pm$0.380 & 0.176$\pm$0.381 \\
 & 512 & 0.104$\pm$0.028 & 0.053$\pm$0.032 & 0.053$\pm$0.029 & 0.082$\pm$0.035 & \underline{0.127$\pm$0.028} & 0.107$\pm$0.030 & 0.106$\pm$0.035 & \textbf{0.134$\pm$0.033} \\
\hline
\multirow{2}{*}{\texttt{S3}} 
 & 256 & \textbf{0.665$\pm$0.046} & 0.067$\pm$0.025 & 0.065$\pm$0.019 & 0.125$\pm$0.035 & \underline{0.582$\pm$0.056} & 0.166$\pm$0.372 & 0.153$\pm$0.360 & 0.137$\pm$0.344 \\
 & 512 & \underline{0.369$\pm$0.041} & 0.061$\pm$0.029 & 0.052$\pm$0.023 & 0.100$\pm$0.018 & \textbf{0.455$\pm$0.057} & 0.154$\pm$0.361 & 0.152$\pm$0.359 & 0.118$\pm$0.323 \\
\hline 
\multirow{2}{*}{\texttt{XOR2}} 
 & 256 & \textbf{0.996$\pm$0.007} & 0.075$\pm$0.016 & 0.049$\pm$0.014 & 0.162$\pm$0.044 & \underline{0.415$\pm$0.061} & 0.070$\pm$0.255 & 0.064$\pm$0.245 & 0.042$\pm$0.201 \\
 & 512 & \underline{0.119$\pm$0.029} & 0.059$\pm$0.012 & 0.076$\pm$0.024 & 0.097$\pm$0.026 & \textbf{0.202$\pm$0.038} & 0.057$\pm$0.026 & 0.057$\pm$0.025 & 0.058$\pm$0.015 \\
\bottomrule
\end{tabular}
}
\begin{threeparttable}
\begin{tablenotes}
\footnotesize
\item * Besov-TST and DNN-TST are omitted due to their consistent underperformance.
\end{tablenotes}
\end{threeparttable}
\end{table}

\subsection{Heavy-tailed distributions}
\label{app_heavy_tailed}

We say that a random vector $\mathbf{X} \in \mathbb{R}^d$ follows a multivariate Student's $t$-distribution with degrees of freedom $\nu$, location parameter $\zeta \in \mathbb{R}^d$, and scale matrix $\Sigma \in \mathbb{R}^{d\times d}$, written as
$$
\mathbf{X} \sim t_\nu(\zeta,\Sigma),
$$
if it admits the representation
$$
\mathbf{X} = \zeta + \frac{\mathbf{Z}}{\sqrt{\mathbf{W}/\nu}},
$$
where $\mathbf{Z} \sim N_d(0,\Sigma)$, $\mathbf{W} \sim \chi^2_\nu$, and $\mathbf{Z}$ and $\mathbf{W}$ are independent.
Let $\mathbf{e}_1 = (1,0,\dots,0)^\top \in \mathbb{R}^d$ denote the first canonical basis vector.
We consider modified settings of \texttt{L1}, \texttt{L2}, \texttt{S1}, and \texttt{XOR2}, denoted by \texttt{L1-H}, \texttt{L2-H}, \texttt{S1-H}, and \texttt{XOR2-H}, corresponding to the modified versions of \texttt{L1}, \texttt{L2}, \texttt{S1}, and \texttt{XOR2}, respectively.
For the null distributions of \texttt{L1-H}, \texttt{L2-H}, and \texttt{S1-H}, we generate $\mathrm{P}$ from $t_{2} (\mathbf{0}_d, \mathrm{I}_d)$ to investigate the performance of PReLU-TST under heavy-tailed distributions.

\begin{itemize}
    \item \texttt{L1-H}: We take
    $$
    \mathrm{Q} \sim t_{2} (\zeta \cdot \mathbf{e}_1, \mathrm{I}_d), \quad \zeta = 0.5,
    $$
    so that $\mathrm{Q}$ differs from $\mathrm{P}$ through a location shift in the first coordinate.

    \item \texttt{L2-H}: We consider the mixture alternative 
    $$
    \mathrm{Q} \sim \frac{1}{2}\, t_{2}(\mathbf{0}_d, \mathrm{I}_d)
      + \frac{1}{2}\, t_{2}(\zeta \cdot \mathbf{e}_1, \mathrm{I}_d),
      \quad  \zeta = 1.0.
    $$
    
    \item \texttt{S1-H}: Let
    $$
    \mathrm{Q} \sim t_{2}(\mathbf{0}_d, \operatorname{diag}(\sigma^2,1,\dots,1)), \quad \sigma^2 = 2.5,
    $$
    so that $\mathrm{Q}$ differs through scale inflation in the first coordinate.
    
    \item \texttt{XOR2-H}: The first two coordinates follow a two-component mixture with opposite sign patterns, while the remaining coordinates are generated independently from a $(d-2)$- variate $t$-distribution. Specifically,
    $$
    \mathrm{P} \sim
    \frac{1}{2}\, t_{2}\!\left(
    \begin{pmatrix}\zeta\\ \zeta\end{pmatrix}, \mathrm{I}_2
    \right)\otimes t_{2}(\mathbf{0}_d,\mathrm{I}_{d-2})
    +
    \frac{1}{2}\, t_{2}\!\left(
    \begin{pmatrix}-\zeta\\ -\zeta\end{pmatrix}, \mathrm{I}_2
    \right)\otimes t_{2}(\mathbf{0}_d,\mathrm{I}_{d-2}),
    $$
    whereas
    $$
    \mathrm{Q} \sim
    \frac{1}{2}\, t_{2}\!\left(
    \begin{pmatrix}\zeta\\ -\zeta\end{pmatrix}, \mathrm{I}_2
    \right)\otimes t_{2}(\mathbf{0}_d,\mathrm{I}_{d-2})
    +
    \frac{1}{2}\, t_{2}\!\left(
    \begin{pmatrix}-\zeta\\ \zeta\end{pmatrix}, \mathrm{I}_2
    \right)\otimes t_{2}(\mathbf{0}_d,\mathrm{I}_{d-2}), 
    \quad  \zeta = 0.5.
    $$
\end{itemize}

As shown in Fig. \ref{fig_tdist_result}, 
PReLU-TST performs relatively poorly in heavy-tailed settings. 
In these cases, MMD-TST and rank-based tests (in particular, Lepage-TST and Cucconi-TST) generally achieve better performance.
To be more specific, PReLU-TST shows moderate performance under location-shift and scale-change settings.
However, for \texttt{XOR2-H}, it fails to effectively detect alternatives, similarly to most other TSTs, except for methods such as MMD-TST, Lepage-TST, and Cucconi-TST.
This is likely because rank-based methods are inherently robust to outliers, and 
the discriminators in MMD-TST is bounced while the discriminators of PReLU-TST are piecewise linear and unbounded.
The results suggest that, in practice, it may be beneficial to first assess whether the data follow heavy-tailed behavior using a separate test procedure (e.g., the Clauset--Shalizi--Newman test \citep{clauset2009power}), and if not, then apply our method accordingly.

\begin{figure}[H]
    \centering    
    \begin{subfigure}[t]{\textwidth}
        \centering
        \includegraphics[width=\textwidth]{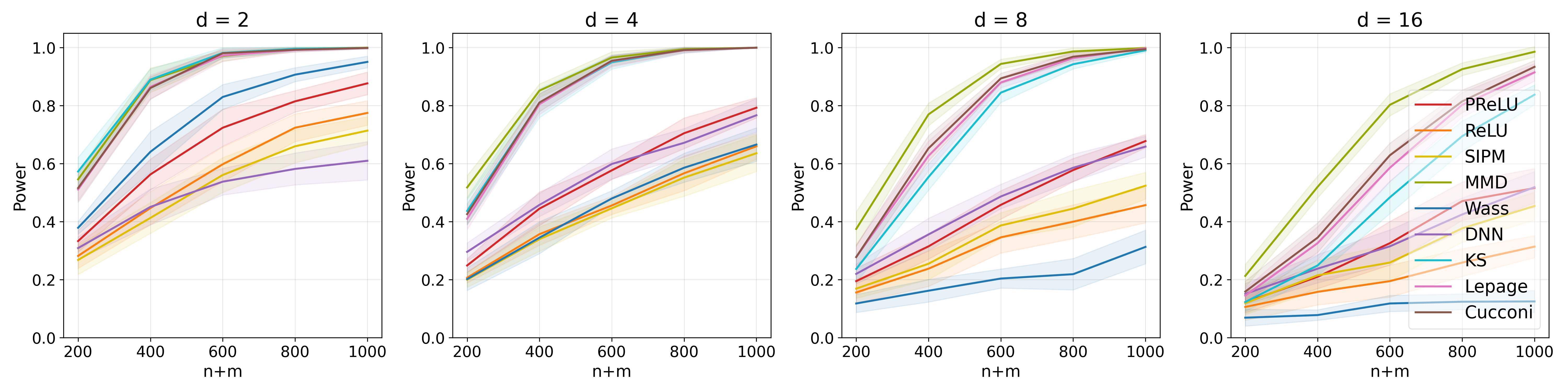}
        \caption{\texttt{L1-H} (a simple mean shift): 
        $\mathrm{P} \sim t_2, \, \mathrm{Q} \sim t_{2} (\zeta \cdot \mathbf{e}_1, \mathrm{I}_d), \mathrm{I}_d\big), \, \zeta = 0.5$}
        \label{fig:ball_tdist}
    \end{subfigure}
    \begin{subfigure}[t]{\textwidth}
        \centering
        \includegraphics[width=\textwidth]{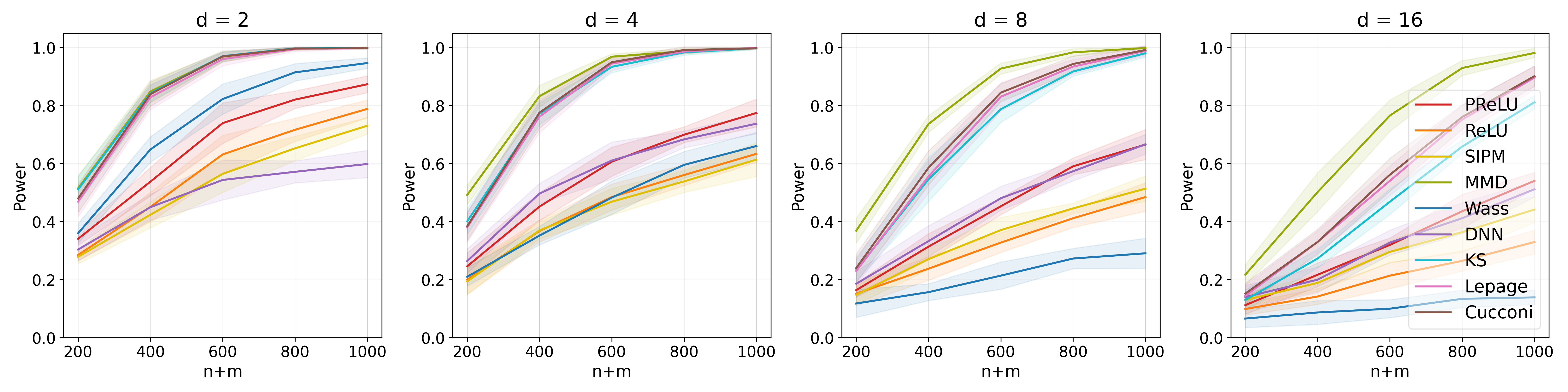}
        \caption{\texttt{L2-H} (a mixture-induced shift): 
        $\mathrm{P} \sim t_2, \, \mathrm{Q} \sim \frac{1}{2}\, t_{2}(\mathbf{0}_d, \mathrm{I}_d) + \frac{1}{2}\, t_{2}(\zeta \cdot \mathbf{e}_1, \mathrm{I}_d), \, \zeta = 1.0$}
        \label{fig:mixture_tdist}
    \end{subfigure}
    \begin{subfigure}[b]{\textwidth}
        \centering
        \includegraphics[width=\textwidth]{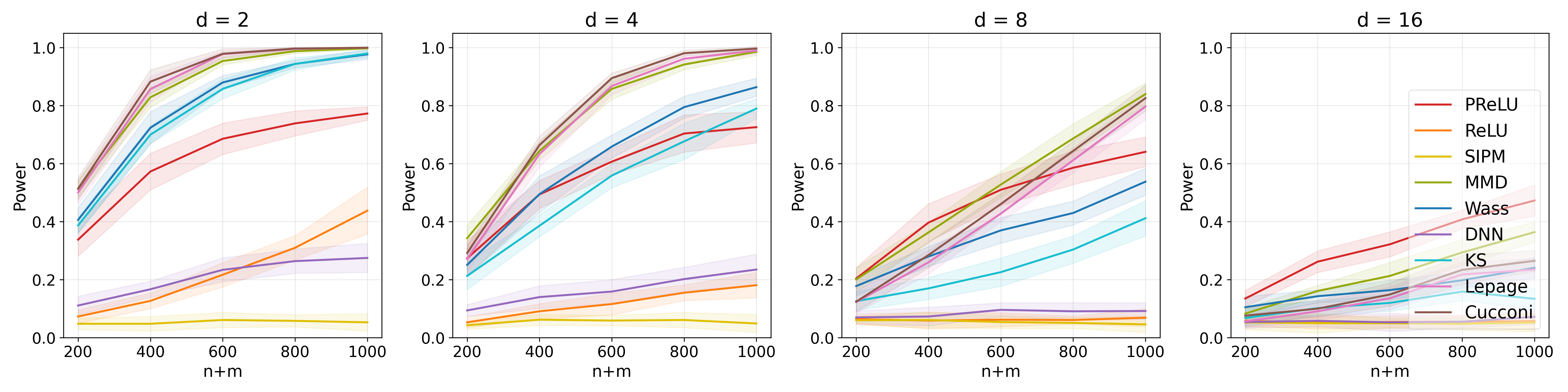}
        \caption{\texttt{S1-H} (a variance change): $\mathrm{P} \sim t_2, \mathrm{Q} \sim t_{2}(\mathbf{0}_d, \operatorname{diag}(\sigma^2,1,\dots,1)), \sigma^2 = 2.5$}
        \label{fig:varone_tdist}
    \end{subfigure}
    \begin{subfigure}[b]{\textwidth}
        \centering
        \includegraphics[width=\textwidth]{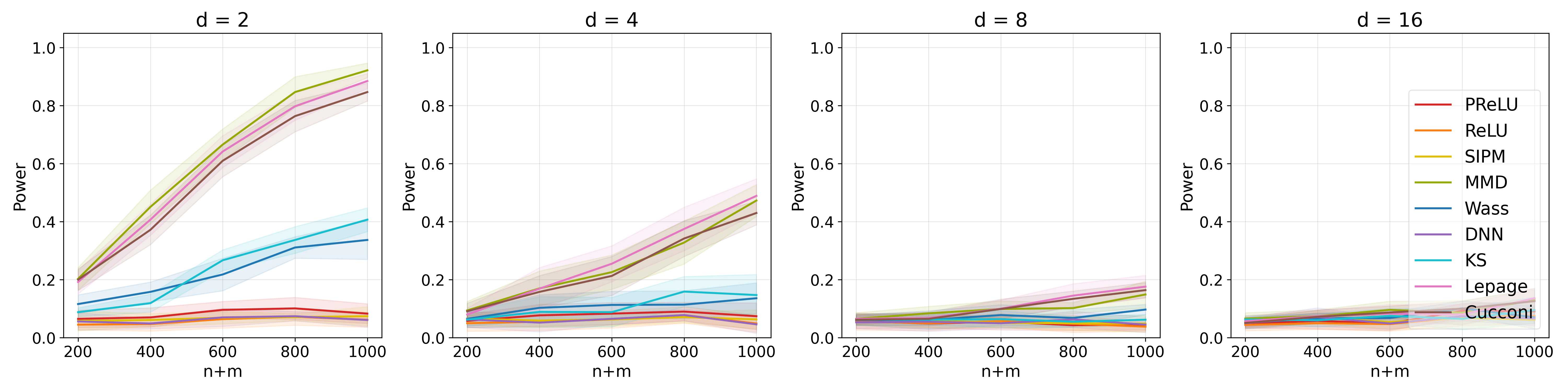}
        \caption{
        \texttt{XOR2-H}: 
        $ \quad
        \begin{aligned}
        \mathrm{P} &\sim \frac{1}{2}\, t_{2}\!\left(
        \begin{pmatrix}\zeta\\ \zeta\end{pmatrix}, \mathrm{I}_2
        \right)\otimes t_{2}(\mathbf{0}_d,\mathrm{I}_{d-2})
        +
        \frac{1}{2}\, t_{2}\!\left(
        \begin{pmatrix}-\zeta\\ -\zeta\end{pmatrix}, \mathrm{I}_2
        \right)\otimes t_{2}(\mathbf{0}_d,\mathrm{I}_{d-2}), \\
        \mathrm{Q} &\sim 
        \frac{1}{2}\, t_{2}\!\left(
        \begin{pmatrix}\zeta\\ -\zeta\end{pmatrix}, \mathrm{I}_2
        \right)\otimes t_{2}(\mathbf{0}_d,\mathrm{I}_{d-2})
        +
        \frac{1}{2}\, t_{2}\!\left(
        \begin{pmatrix}-\zeta\\ \zeta\end{pmatrix}, \mathrm{I}_2
        \right)\otimes t_{2}(\mathbf{0}_d,\mathrm{I}_{d-2}), \, 
        \zeta = 0.5
        \end{aligned}
        $
        }
        \label{fig:xor2_tdist}
    \end{subfigure}
    \caption{
    Power vs.\ $n+m$ under heavy-tailed distributions for various $d$. 
    The results of Besov-TST are omitted due to their consistent underperformance.
    }
    \label{fig_tdist_result}
\end{figure}

\subsection{Ablation studies}
\label{app_abl_study}

Now we provide additional experiments to better understand the role of slope selection in PReLU-TST and the optimization behavior of PReLU-IPM.

\subsubsection{Analysis for the slope selection of PReLU-TST}
\label{app_slope_range}
We evaluate the performance of PReLU-TST across different choices of 
$\ell_{\max}$.
Fig.~\ref{fig_slope_sensitivity} indicates that the optimal slope search range depends on the underlying problem setting, and different choices may yield the best performance in different scenarios. For example, the results for \texttt{L3} show that the default choice $(-1.0, 0.5)$ achieves the best or near-best performance across different sample sizes and data dimensions, although it may not always guarantee optimal performance, particularly as the dimension varies, as illustrated by the results for \texttt{S1} and \texttt{XOR1}. Nevertheless, the default choice $(-1.0, 0.5)$ provides reasonably good performance across a wide range of settings.

\begin{figure}[H]
    \centering
    \includegraphics[width=.95\linewidth]{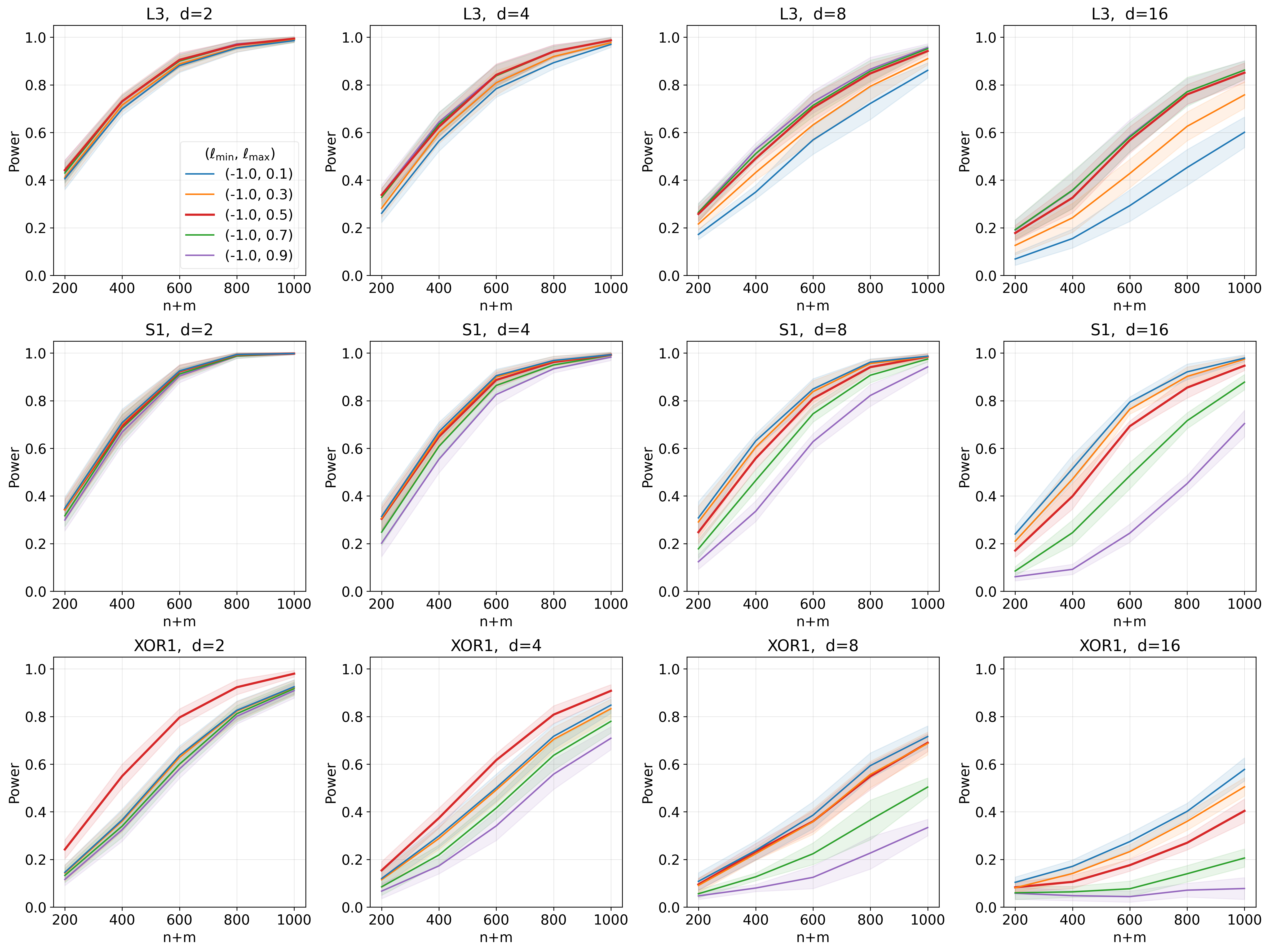}
    \caption{
    Performance of PReLU-TST across different choice of $\ell_{\max}$.
    The default setting $(-1.0, 0.5)$ is shown in red. Each experiment was repeated 1,000 times with 1,000 permutations. 
    The line shows the average testing power, with the shaded region indicating the standard deviation.}
    \label{fig_slope_sensitivity}
\end{figure}

\subsubsection{Multiple initializations}
\label{app_abl_init}

We investigate the effect of multiple random initializations on the optimization stability of the PReLU-IPM.
Specifically, we draw $n = m = 500$ samples from 
$$
\mathrm{P} \sim \mathcal{N}\bigl((-0.3, 0, \dots, 0),\, \mathrm{I}_d\bigr), \quad \mathrm{Q} \sim \mathcal{N}\bigl((0.3, 0, \dots, 0),\, \mathrm{I}_d\bigr)
$$
with $d \in \{2,4,8,16\}$. 
Then, for each fixed dataset, we compute
$\log D_{\mathscr{F}_{\mathsf{prelu}}}(\widehat{\mathrm{P}}_n, \widehat{\mathrm{Q}}_m)$ 10,000 times using Algorithm~\ref{alg_prelutst}, with $S \in \{1,10,100,200\}$. 
Since the datasets are fixed and we solve a maximization problem, larger estimated values indicate better optimization.

\begin{figure}[H]
    \centering
    \includegraphics[width=\linewidth]{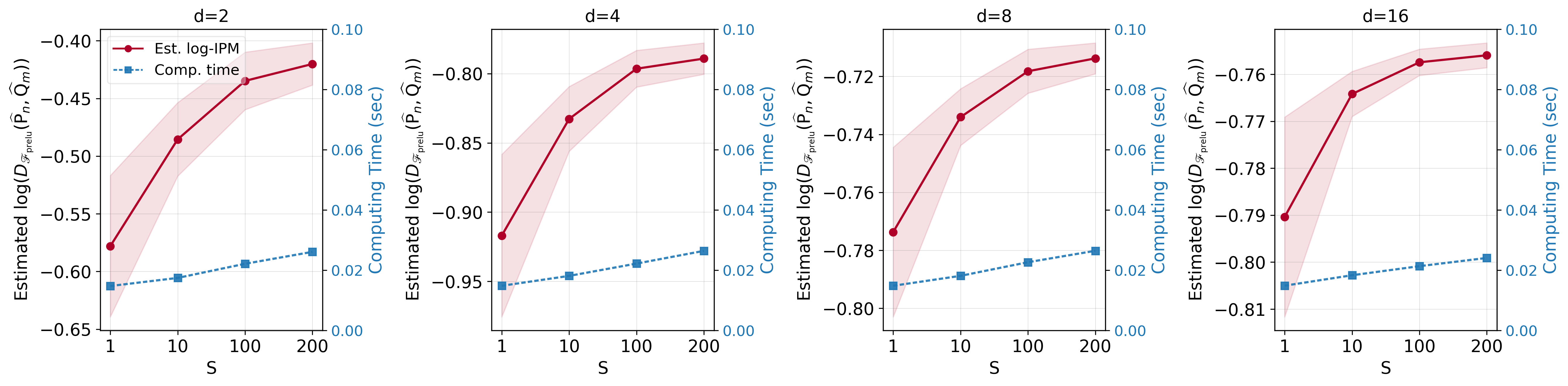}
    \caption{
    Estimated $\log D_{\mathscr{F}_{\mathsf{prelu}}}(\widehat{\mathrm{P}}_n, \widehat{\mathrm{Q}}_m)$ (solid, left axis) and computing time (dashed, right axis) vs. the number of random initializations $S$.
    }
    \label{fig_nprelu_ipm}
\end{figure}

In Fig.~\ref{fig_nprelu_ipm},
we evaluate the estimated $\log D_{\mathscr{F}_{\mathsf{prelu}}}(\widehat{\mathrm{P}}_n, \widehat{\mathrm{Q}}_m)$ together with the corresponding computing time as a function of the number of random initializations $S$. 
As $S$ increases, the estimated $\log D_{\mathscr{F}_{\mathsf{prelu}}}(\widehat{\mathrm{P}}_n, \widehat{\mathrm{Q}}_m)$ increases, while the standard deviation of the estimates decreases.
This suggests that random initializations help avoid poor local optima.
Importantly, this improvement comes at only a modest computational cost, as the discriminator $f \in \mathscr{F}_{\mathsf{prelu}}$ can be represented via a single affine map $\mathbf X \mapsto (\Theta \mathbf X + \boldsymbol{\mu})_+ \in \mathbb{R}^S$, where $\Theta = (\boldsymbol{\theta}_1, \dots, \boldsymbol{\theta}_S)^{\top}$ and $\boldsymbol{\mu} = (\mu_1, \dots, \mu_S)^{\top}$, backpropagation can be carried out in parallel by the python package without any special technique, and thus its computing time does not increase much
with respect to $S$.
We use $S=100$ as a conservative default that balances stability of the statistic and computational efficiency.

\subsection{Effect of the Number of Permutations}

We investigate the sensitivity of PReLU-TST to the number of permutations $n_{\mathrm{perm}}$ used in the simulations (\texttt{L1}, \texttt{L2}, \texttt{L3}, \texttt{S1}, \texttt{S2}, \texttt{S3}, \texttt{XOR1}, and \texttt{XOR2}) of Section~\ref{sec_4_1}. In particular, we examine whether the number of permutations affects the empirical performance. As shown in Figs.~\ref{fig_result1_perm200} and \ref{fig_result2_perm200}, using $n_{\mathrm{perm}} = 200$ yields performance that is highly similar to that obtained with $n_{\mathrm{perm}} = 1{,}000$ in the main experiments.

\begin{figure}[H]
    \centering    
    \begin{subfigure}[t]{\textwidth}
        \centering
        \includegraphics[width=\textwidth]{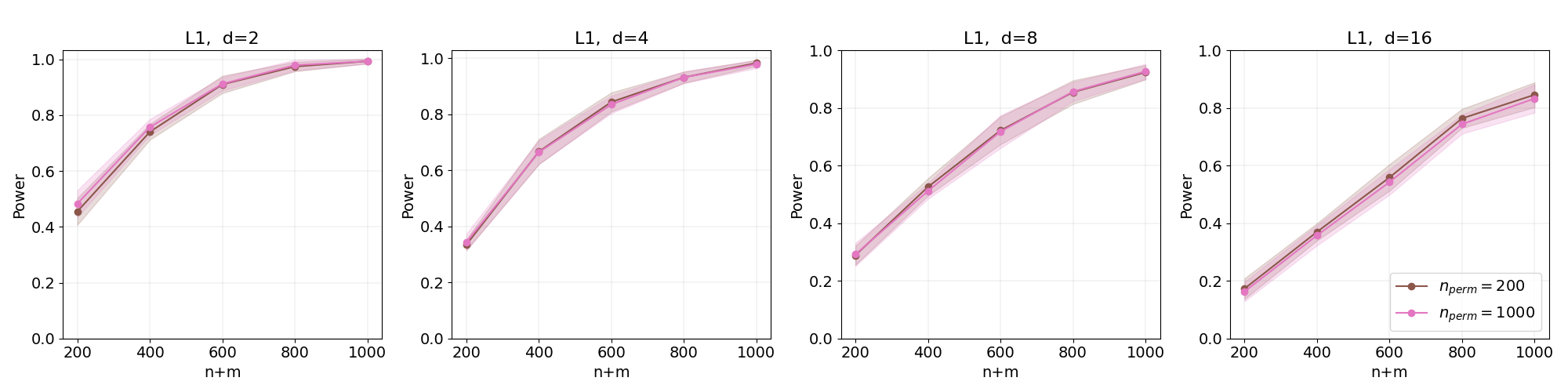}
        \label{fig:ball_perm200}
    \end{subfigure}
    \begin{subfigure}[t]{\textwidth}
        \centering
        \includegraphics[width=\textwidth]{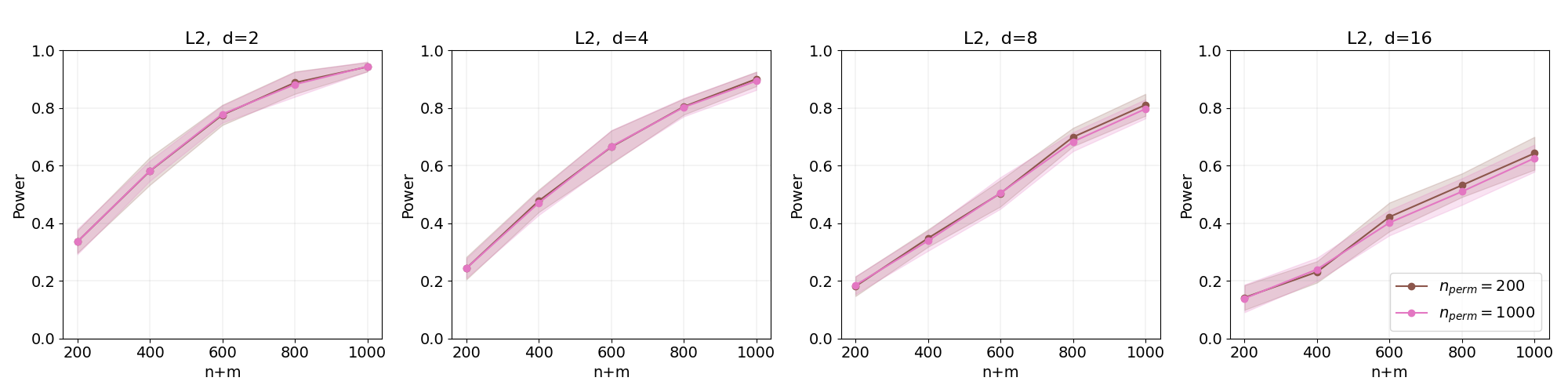}
        \label{fig:mixture_perm200}
    \end{subfigure}
    \begin{subfigure}[t]{\textwidth}
        \centering
        \includegraphics[width=\textwidth]{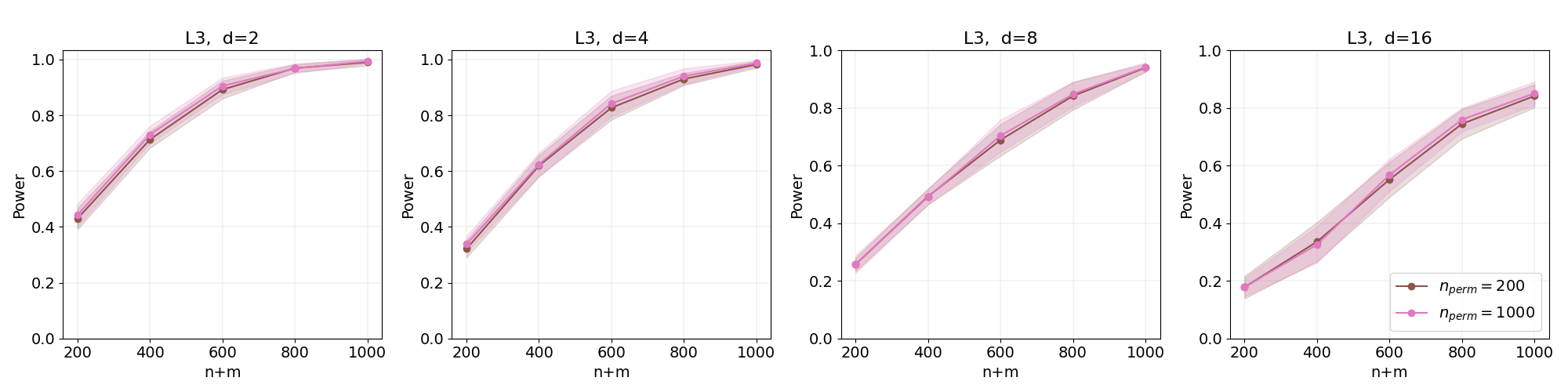}
        \label{fig:skewed_perm200}
    \end{subfigure}
    \caption{
    Effect of the number of permutations: \texttt{L1} (top), \texttt{L2} (middle), \texttt{L3} (bottom). We compare $n_{\mathrm{perm}}=200$ and $n_{\mathrm{perm}}=1{,}000$ across different dimensions $d$.
    }
    \label{fig_result1_perm200}
\end{figure}

\begin{figure}[H]   
    \centering    
    \begin{subfigure}[b]{\textwidth}
        \centering
        \includegraphics[width=\textwidth]{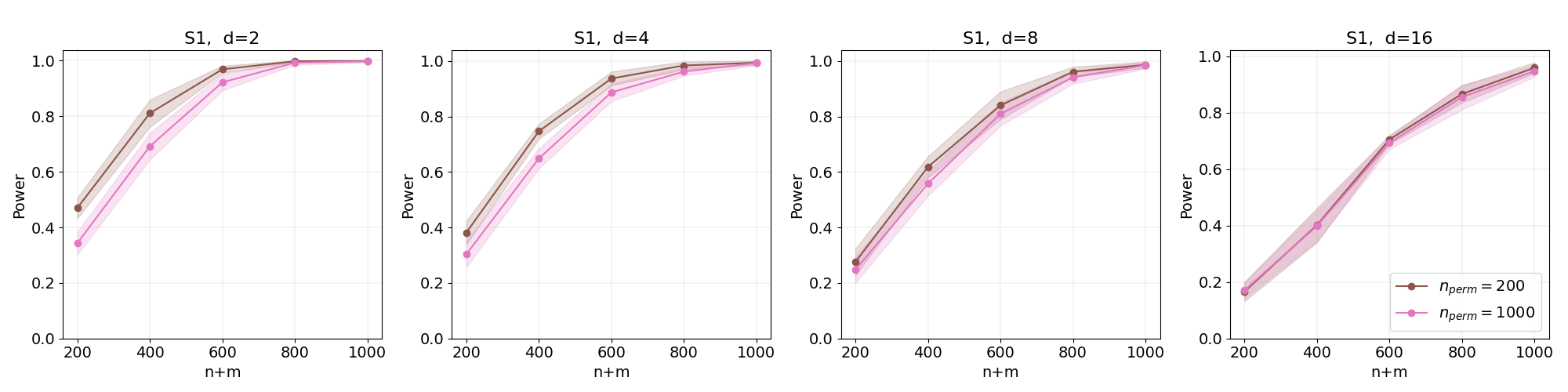}
        \label{fig:varone_perm200}
    \end{subfigure}
    \vspace{0.5em}
    \begin{subfigure}[b]{\textwidth}
        \centering
        \includegraphics[width=\textwidth]{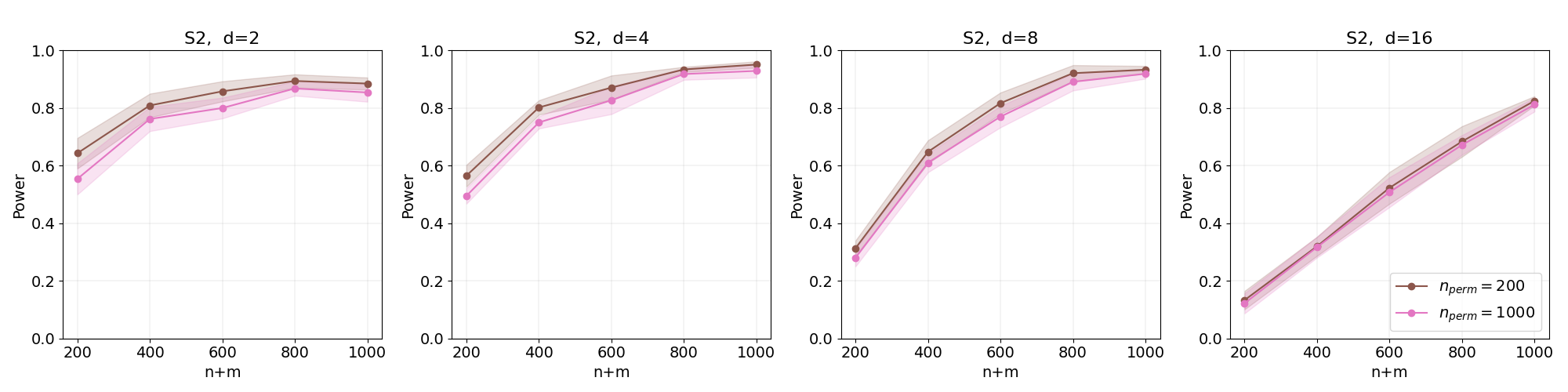}
        \label{fig:gamma_perm200}
    \end{subfigure}
    \vspace{0.5em}
    \begin{subfigure}[b]{\textwidth}
        \centering
        \includegraphics[width=\textwidth]{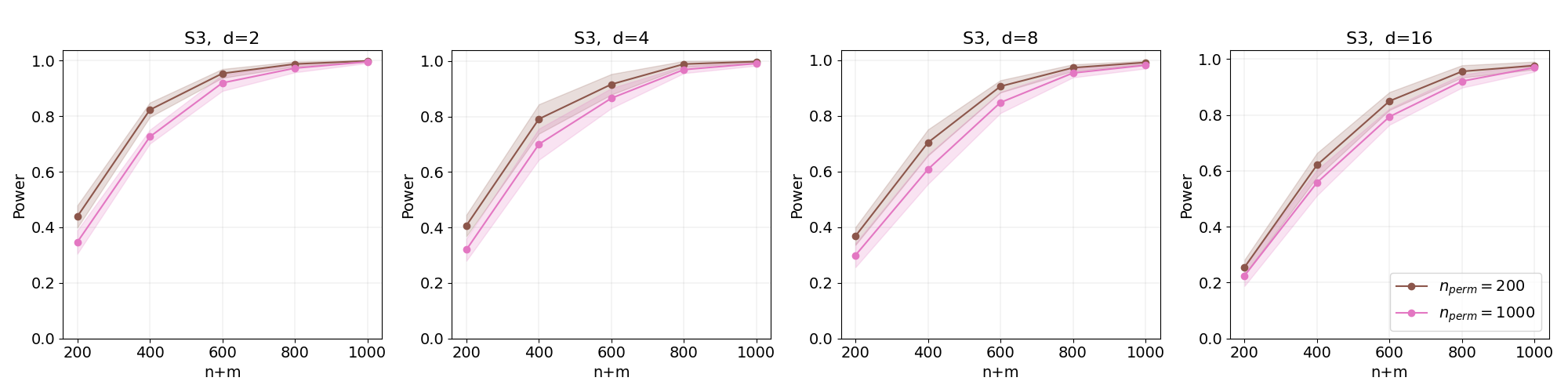}
        \label{fig:tcoord_perm200}
    \end{subfigure}
    \begin{subfigure}[b]{\textwidth}
        \centering
        \includegraphics[width=\textwidth]{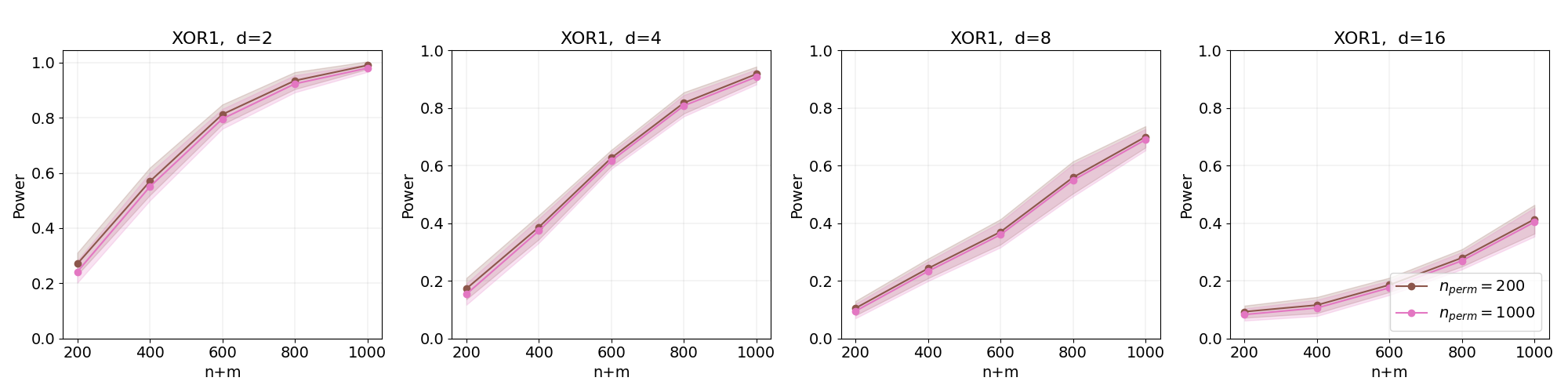}
        \label{fig:xor_perm200}
    \end{subfigure}
    \begin{subfigure}[b]{\textwidth}
        \centering
        \includegraphics[width=\textwidth]{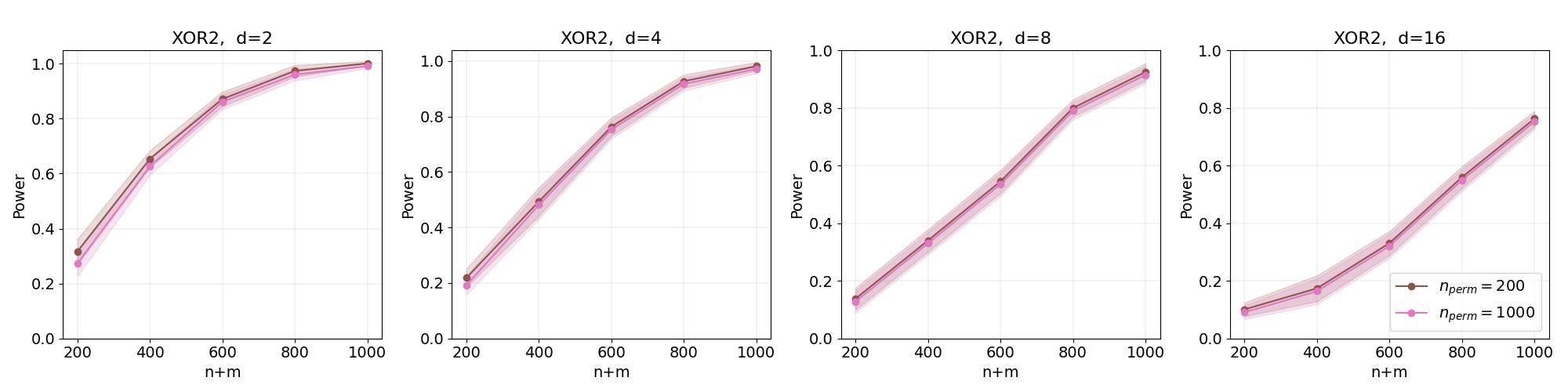}
        \label{fig:xor2_perm200}
    \end{subfigure}
    \caption{
    Effect of the number of permutations. From top to bottom: \texttt{S1}, \texttt{S2}, \texttt{S3}, \texttt{XOR1}, and \texttt{XOR2}. We compare $n_{\mathrm{perm}}=200$ and $n_{\mathrm{perm}}=1{,}000$ across different dimensions $d$.
    }
    \label{fig_result2_perm200}
\end{figure}


\bibliographystyle{apalike}
\bibliography{ref}

\end{document}